\documentclass[twoside,11pt]{article}

\usepackage{makecell}
 \usepackage{amsthm}
 \usepackage{enumitem}
 \usepackage{hyperref}
 \usepackage{amssymb}
 \usepackage{amsmath}
 \usepackage{dsfont}
 \usepackage{mathtools}
 \usepackage{xcolor}
 \usepackage{mathrsfs}
 \usepackage{natbib}

 \usepackage[margin = 1in]{geometry}

 \usepackage{setspace}
 \usepackage{comment}

\usepackage{algorithmic} 
\usepackage{algorithm} 
\usepackage{adjustbox}

 \newtheorem{thm}{Theorem}
 \newtheorem{prop}{Proposition}
 \newtheorem{lem}{Lemma}
 
 \newtheorem{rem}{Remark}
 \newtheorem{cor}{Corollary}

\usepackage{comment}

\newcommand{\intervalleff}[2]{\mathopen{[}#1\,,#2\mathclose{]}}

\newcommand{\intervalleof}[2]{\mathopen{(}#1\,,#2\mathclose{]}}

\newcommand{\RR}{\mathbb{R}}
\newcommand{\NN}{\mathbb{N}}

\newcommand{\EE}{\mathbb{E}}

\newcommand{\learningrate}{\eta}

\newcommand{\francois}[1]{{\color{violet}Francois: #1}}

\newcommand{\iprod}[2]{\left\langle #1,#2\right\rangle}

\newcommand{\DKL}[2]{D_{\mathrm{KL}}(#1||#2)}

\newcommand{\nparam}{P}

\newcommand{\argmax}[1]{\underset{#1}{\mathrm{argmax}\ }}
\newcommand{\argmin}[1]{\underset{#1}{\mathrm{argmin}\ }}

\author{Fran\c{c}ois G. Ged$^{1,2}$ and Maria Han Veiga$^{3}$}

\begin{document}

\title{Matryoshka Policy Gradient for Entropy-Regularized RL: Convergence and Global Optimality}


\maketitle
\noindent\small{
$^1$Chair of Statistical Field Theory, \'{E}cole Polytechnique F\'{e}d\'{e}rale de Lausanne, Switzerland\\
$^2$Dynamical Systems in Biomathematics, University of Vienna, Austria\\
$^3$Department of Mathematics, The Ohio State University, Columbus, USA}

\begin{abstract}
A novel Policy Gradient (PG) algorithm, called \textit{Matryoshka Policy Gradient} (MPG), is introduced and studied, in the context of fixed-horizon max-entropy reinforcement learning, where an agent aims at maximizing entropy bonuses additional to its cumulative rewards.
In the linear function approximation setting with softmax policies, we prove uniqueness and characterize the optimal policy of the entropy regularized objective, together with global convergence of MPG.
These results are proved in the case of continuous state and action space.
MPG is intuitive, theoretically sound and we furthermore show that the optimal policy of the infinite horizon max-entropy objective can be approximated arbitrarily well by the optimal policy of the MPG framework.
Finally, we provide a criterion for global optimality when the policy is parametrized by a neural network in terms of the neural tangent kernel at convergence.
As a proof of concept, we evaluate numerically MPG on standard test benchmarks.
\end{abstract}

  reinforcement learning, policy gradient, entropy regularization, global convergence, neural networks

\section{Introduction}

The family of Policy Gradient algorithms (PG) in Reinforcement Learning (RL) originated several decades ago with the algorithm REINFORCE \citep{Williams92}, the name \textit{Policy Gradient} appearing only in 2000 in \citet{Sutton99}, they recently regained interest thanks to many remarkable achievements, to name a few: in continuous control \citep{Lillicrap15, Schulman15, Schulman17PPO} and
natural language processing such as GPT-3 \citep{Brown20}\footnote{instructGPT and chatGPT are trained with Proximal Policy Optimization, see \url{https://openai.com/blog/chatgpt/}.}, and more generally in the fine-tuning from human feedback stage of large language models \citep{ziegler2019fine}.
See the blog post of \citet{weng18} that lists important PG methods and provides a concise introduction to each of them.

PG methods are considered more suitable for large (possibly continuous) state and action spaces than other nonetheless important methods such as Q-learning and its variations.
However, for large spaces, the exploitation-exploration dilemma becomes more challenging.
In order to enhance exploration, it has become standard to use a regularization to the objective, as in \textit{max-entropy RL} \citep{Nachum17,ODonoghue16,Schulman17,mnih2016asynchronous,Haarnoja18SoftAO}, where the agent maximizes the sum of its rewards plus a bonus for the entropy of its policy\footnote{Other regularization techniques are used and studied in the literature, we focus on entropy regularized RL in this paper.}.
In particular \citet{ahmed2019understanding} study specifically the impact of entropy on policy optimization.
Not only does max-entropy RL boost exploration, it also yields an optimal policy that is stochastic, in the form of a Boltzmann measure, such that the agent keeps taking actions at random while maximizing the regularized objective.
This is sometimes preferable than deterministic policies.
In particular, \citet{Eysenbach21} show that the max-entropy RL optimal policy is robust to adversarial change of the reward function (their Theorem 4.1) and transition probabilities (their Theorem 4.2); see also references therein for more details on that topic.
Finally, max-entropy RL is appealing from a theoretical perspective.
For example, soft Q-learning, introduced by \citet{Haarnoja17} (see also \citet{Haarnoja18SoftAO,Haarnoja18SoftAA} for implementations of soft Q-learning with an actor-critic scheme), strongly resembles PG in max-entropy RL \citep{Schulman17}; max-entropy RL has also been linked to variational inference \citep{Levine18}.
Other appealing features of max-entropy RL are discussed by \citet{Eysenbach19} and references therein.

A vast number of works on RL have focused on either infinite horizon tasks, or episodic tasks where the length of an episode is random. In both these cases, policies only depend on the current state of the agent.
In \citet{Ernst03}, the fixed, finite horizon optimal policy is used as an approximation, as the horizon grows to infinity, to approximate the infinite-horizon optimal policy.
Nonetheless and even though the fixed (deterministic) horizon setting has received less attention, the benefits of fixing the horizon are multiple and have been investigated in recent relevant works, such as \citet{Asis19,Guin22,vp2021finite}.
\citet{Asis19} study a \textit{Temporal Difference} algorithm (which is not PG), typically involving bootstrapping.
When it is used offline (off-policy) together with function approximation, it encounters the well-known stability issue called the \textit{deadly triad}, see \citet{Sutton18} Section 11.3.
By using horizon-dependent value functions, they do not rely on bootstrapping, getting rid of one element of the triad, thus ensuring more stability.
It is worth noting that thanks to the fixed-horizon setting, they empirically overcome the specific Baird's example of divergence.
\citet{Guin22} defines an actor-critic algorithm for constrained RL, where the agent aims at maximizing the cumulative rewards while satisfying some given constraints.
They prove the convergence for finite state and action spaces but do not study global convergence.

In the same spirit, \citet{Seijen19} investigate the impact of the discount factor when optimizing a discounted infinite-horizon objective evaluated on a finite-horizon undiscounted objective.
They empirically found that for some tasks, lower discount factors (thus closer to a fixed-horizon objective) lead to better performance.
With a fixed horizon, policies are time-dependent, and are usually called \textit{non-stationary} policies, as in dynamic programming \citep{Bertsekas1995DynamicPA}.

Regarding PG methods specifically in the fixed-horizon setting with tabular softmax parametrization, the preprint by \citet{klein2023beyond} proves global convergence using the \textit{gradient domination property}, which generally does not holds in the infinite state-action space, see our discussion on previous approaches below.
PG with fixed horizon has also recently been studied outside of the MDP setting.
Global convergence of some algorithms is established for some class of continuous time problems, see e.g. \citet{hambly2021policy, giegrich2022convergence}.

\paragraph{Contributions.}
We consider the function approximation setting with log-linear parametric policies, that are constructed as the softmax of linear models.
For convergence results, we assume perfect gradient updates, i.e. we have access to the exact gradient of the objective.
The main contributions of this work are:
\begin{enumerate}[label = (\roman*)]
    \item We define the fixed-horizon max-entropy RL objective and introduce a new algorithm (Equation \eqref{eq: update cascade learning}), named \textit{Matryoshka Policy Gradient} (MPG).
    \item We establish global convergence for continuous state and action space: under the realizability assumption, MPG converges to the unique optimal policy (Theorem \ref{thm: general case global optimality}). When the realizability assumption does not hold, we prove uniqueness of the optimal policy and prove global convergence of MPG (Theorem \ref{Thm: global cvg outside of the RKHS}).
    \item We approximate arbitrarily well the optimal policy for the infinite horizon objective by the optimal policy of the MPG objective (Proposition \ref{prop: extending horizon converges to standard optimal policy}).
    \item In the case where the policy is parametrized as the softmax of a (deep) neural network's output, we describe the limit of MPG training in terms of the \textit{neural tangent kernel} and the \textit{conjugate kernel} of the neural network at the end of training (Corollary \ref{thm: gen case deep RL global optimality}). In particular, MPG globally converges in the \textit{lazy regime}.
    \item Numerically, we successfully train agents on standard simple tasks without relying on RL tricks, and confirm our theoretical findings (see Section \ref{Section: numerical experiments}). 
\end{enumerate}

In our numerical experiments described in Section \ref{Section: numerical experiments}, we first consider an analytical task and verify the global convergence property of the MPG: MPG consistently finds the unique global optimum, which satisfies the projectional consistency property. Then, we study two benchmarks from OpenAI: the Frozen Lake game and the Cart Pole. We obtain successful policies for both benchmarks with the MPG algorithm, comparing also to vanilla PG method \citep{Sutton99}. Rather than competing with the state-of-the-art algorithms, our aim is to provide a proof of concept by showing that successful training can be obtained with a straightforward implementation of MPG.
We hope that more general and bigger scale experiments implementing variations of MPG will follow the present work.

\paragraph{Comparison with previous results and approaches.}
Besides the well-known \textit{Policy Gradient Theorem} (see Chapter 13 in \citet{Sutton18}) that can imply convergence of PG (provided good learning rate and other assumptions), for many years, not much more was known about the global convergence of PG (i.e. convergence to an optimal policy) until recently.
Despite the numerous remaining gaps, some important progress have already been made.
In particular, the global convergence of PG methods has been studied and proved in specific settings, see for instance \citet{Fazel18,Agarwal22,Bhandari19,Mei20,Zhang20SampleER,Zhang19GlobalCO,Cen20,Ding21,Wang19, Agazzi20, Bhandari20OnTL, Leahy22, Guin22}.
Convergence guarantees often come with convergence rates (with or without perfect gradient estimates).
Though strengthening the trust in PG methods for practical tasks, most of the theoretical guarantees obtained in the literature so far require rather restrictive assumptions, and often assume that the action-state space of the MDP is finite (but not always, e.g. \citet{Agazzi20} address continuous action-state space for neural policies in the mean-field regime and \citet{Leahy22} prove global convergence when adding a strong enough regularization on the parameters.)
In particular, in the context of tabular softmax policies, \citet{Li21} study the dependency of the number of iterations of the (perfect) gradient PG update on the size of the state space, and construct environments with only three actions per state requiring the algorithm to make $\frac{1}{\eta}|\mathcal{S}|^{2^{\Omega(1/(1-\gamma))}}$ iterations to converge, where $\gamma$ is the discount factor and $\eta$ the learning rate.

We now highlight the key differences in proof techniques between our work and previous works.
\citet{Agarwal22} give many convergence guarantees for different policy gradient algorithms.
In particular, in the tabular case with finite state and action spaces, global convergence is obtained thanks to the \textit{gradient domination property}, stated in their Lemma 4.1 as follows: for every probability distributions $\mu,\rho$ on $\mathcal{S}$, for every policy $\pi$, the difference between the value function of the optimal policy $\pi_*$ and the value funtion of $\pi$ satisfies
\begin{align*}
    \sum_{s\in\mathcal{S}}(V_{\pi_*}(s)-V_\pi(s))\rho(s)
    \leq \frac{1}{1-\gamma}\left|\left|\frac{d_\rho^{\pi_*}}{\mu}\right|\right|_\infty \max_{\pi'}(\pi'-\pi)\nabla_\pi \sum_{s\in\mathcal{S}}V_\pi(s)\mu(s),
\end{align*}
where $d_\rho^{\pi_*}$ is the state-visitation distribution induced by the initial state distribution $\rho$ and $\pi_*$, and the max is taken over the set of all policies.
The factor $||d_\rho^{\pi_*}/\mu||_\infty$ is called the \textit{distribution mismatch} between $d_\rho^{\pi_*}$ and $\mu$.
Since $d_\rho^{\pi_*}(s)$ is proportional to the discounted time spent in state $s$ under the optimal policy $\pi_*$, the gradient domination property ensures that if the policy is trained with respect to a state distribution $\mu$ whose support contains that of $d_\rho^{\pi_*}$, then the gradient vanishes only at the optimum.

One advantage is that in the tabular setting, the rate of convergence can be deduced, involving the distribution mismatch, see e.g. Section 4 in \citet{Agarwal22}.
In the function approximation setting with infinite state space and finite action space, they obtain convergence results for the \textit{natural policy gradient} algorithm, which uses the Fisher information matrix induced by the policy in the update, but do not obtain the optimality of the limit.

Similarly, \citet{Mei20} study convergence rates towards global optimum but rely on tabular parametrization with finite state and action spaces to guarantee finite concentration coefficient and to obtain a {\L}ojasiewicz type of inequality (also lower bound on the gradient), which is vacuous in infinite state space. See also \citet{Ding21} for global convergence of an entropy-regularized PG method with softmax policies with sample-based updates, assuming finite state and action spaces and tabular parametrization.
Convergence rates for general modified policy iteration approaches are also obtained in \citet{geist2019theory} for finite state and action spaces with finite concentration coefficient.
In the parametric case, with log-linear policies, \citet{yuan2022linear} study convergence rates for natural gradient descent. However, they do not show global convergence and also require finite state and action space.
In \citet{mei2021leveraging}, a convergence rate to the global optimum is obtained for Geometry-aware Normalized PG, but it requires finite state and action spaces to guarantee finite concentrability coefficient, while their non-uniform {\L}ojasiewicz inequality is vacuous as the number of states grows to infinity.

The main takeaway is that standard methods, including the works mentioned above, use bounds on the gradient of the objective (gradient domination property, {\L}ojasiewicz inequality, ...) to deduce convergence rates, which often need finite state and action space, while global optimality usually requires tabular parametrization.
Besides, none of them concerns non-stationary policies.
On the other hand, our proof for the convergence of MPG does not rely on a lower bound of the gradient, which is why we do not obtain convergence rates. However, we obtain global convergence as follows:
\begin{enumerate}[label = \roman*)]
    \item The gradient of the objective remains Lipschitz continuous along the training trajectory, which ensures that the objective converges (Theorem \ref{thm: policy gradient theorem}).
    \item To ensure that the policy converges too, we show that the sequence generated during training is \textit{relatively compact} (Lemma \ref{lem:converging_subsequence_of_policies}), which essentially guarantees that the sequence of policies generated during training converge (more precisely, any of its subsequence has a converging subsequence).
    \item The parameters remain bounded during training (Lemma \ref{lem:parameters_remain_bounded}), which entails that any limiting policy of training is a critical point that belongs to the parametric space.
    \item Using tools from information geometry (Appendix \ref{appendix: information geometry}), we then show that the only critical point of the objective inside the parametric space is the unique projection of the global optimal policy onto the parametric space (see proofs of Theorem \ref{thm: general case global optimality} and Theorem \ref{Thm: global cvg outside of the RKHS}).
    This projection is globally optimal in the parametric space.
\end{enumerate}
The main strength of MPG
\begin{enumerate}[label = \textbullet]
    \item State space and action space can be infinite, continuous and even unbounded.
    \item Global convergence is guaranteed in the function approximation setting with log-linear policies.
    \item Even when the realizable assumption does not hold, MPG converges to the unique optimal policy in the parametric policy space.
    This global optimum can be characterized as the unique policy satisfying the \textit{projectional consistency property} in the parametric space.
\end{enumerate}

\section{Fixed-horizon max-entropy RL}

In this section, we introduce the fixed-horizon max-entropy RL, describe its optimal policy and establish some of its properties.

\subsection{Definitions}

\paragraph{Markov Decision Process}

The \textit{Markov Decision Process} (MDP) setup is a very standard and important setup in RL \citep{puterman2014markov,Sutton18}.
It is suited for sequential decision making, where the dynamics are Markovian, i.e. depend on the past decisions only through the current state of the agent, making it mathematically tractable.

The agent evolves according to a MDP characterized by the tuple $(\mathcal{S},\mathcal{A},p,p_{\mathrm{rew}})$ modelling the \textit{environment}, a map called a \textit{policy} $\pi:\mathcal{A}\times\mathcal{S}\to [0,1]$, and an \textit{initial state distribution} $\nu$ on $\mathcal{S}$.\footnote{Implicitly assumed in the MDP definition is the fact that all variables such that actions, visited states and rewards are measurable, so that they are well-defined random variables.}
The action space can be state dependent $\mathcal{A}_s$, nonetheless we assume for simplicity that it is the same regardless of the current state of the agent.
We assume that the action and state spaces $\mathcal{A},\mathcal{S}\subset\RR^d$ are closed sets.
Let $s'\mapsto p(s,a,s')$ be the probability (the density if $\mathcal{S}$ is continuous) that the agent moves from $s\in\mathcal{S}$ to $s'\in\mathcal{S}$ after taking action $a\in\mathcal{A}$.
When $p(s,a,s')=\delta_{s',f(s,a)}$ for some $f:\mathcal{S}\times\mathcal{A}\to\mathcal{S}$, then we say that the transitions are deterministic.
The reward depends on the action and on the current state, its law is denoted by $p_{\mathrm{rew}}(\cdot|s,a)$.
Throughout, we assume that the rewards are uniformly bounded and for all $(s,a)\in\mathcal{S}\times\mathcal{A}$, we denote by $r(a,s)$ the mean reward after taking action $a$ at state $s$.

A stationary policy $\pi:\mathcal{A}\times\mathcal{S}\to\intervalleff{0}{1}$ is a map such that for all $s\in\mathcal{S}$, $\pi(\cdot|s)$ is a probability distribution on $\mathcal{A}$ that describes the law of the action taken by the agent at state $s$.
Let $\mathcal{P}$ denote the set of stationary policies.
Let $n\in\NN$ be some fixed horizon, we denote by $\mathcal{P}_n$ the set of non-stationary policies $\pi=(\pi^{(1)},\ldots,\pi^{(n)})$ where for all $i=1,\ldots,n$, $\pi^{(i)}\in\mathcal{P}$.
Henceforth, we use the term ``policy'' for non-stationary policies.
We say that the agent follows a policy $\pi\in\mathcal{P}_n$ if and only if it chooses its actions sequentially according to $\pi^{(n)}$, then $\pi^{(n-1)}$, and so on until $\pi^{(1)}$ and the end of the episode.
That is for each episode of fixed length $n$, starting from a given state $S_0$, the agent generates a path $S_0,A_0,S_1,A_1,\ldots, A_{n-1},S_n$, where $A_i\sim\pi^{(n-i)}(\cdot|S_i)$ and $S_{i+1}\sim p(S_i,A_i,\cdot)$.
Note that in the standard infinite horizon setting, a policy corresponds to an infinite sequence of the same stationary policy $\{(\pi,\pi,\pi,\cdots); \pi\in\mathcal{P}\}\subset\mathcal{P}_\infty$.
All random variables are such that the process is Markovian.

Henceforth, we assume that $\mathcal{A}$ and $\mathcal{S}$ are continuous, the results identically holding true when they are countable.
We also assume that
\begin{itemize}
    \item The sets $\mathcal{A},\mathcal{S}\subset\RR^d$ are closed Borel sets.
    \item The transition probability function $p:\mathcal{S}\times\mathcal{A}\times\mathcal{S}\to [0,1]$ and the reward function $r:\mathcal{A}\times\mathcal{S}$ are measurable.
    \item For all Borel set $B\subset\mathcal{S}$, the map $(s,a)\mapsto p(s,a,B)$ is continuous.
    \item The so-called \textit{measurable selection assumption} holds.
\end{itemize}
The measurable selection assumption is a technical statement that ensures that the MDP has a well-defined optimal solution.
Several conditions ensuring that it holds can be found in \citet{Hernandez99} Section 3.3.
To avoid technical discussions that are not relevant to the present work, the reader can replace the measurable selection assumption by assuming the following simple sufficient and not too restrictive condition:
\begin{itemize}
    \item The action space $\mathcal{A}$ is compact and the reward function and the transition probability function are continuous with respect to the Euclidean metric.
\end{itemize}

\paragraph{Value and Q functions.}

For every $s\in\mathcal{S}$ and $\pi,\pi'\in\mathcal{P}$, we denote by $\mathrm{D_{KL}}(\pi||\pi')(s)=\mathrm{D_{KL}}(\pi(\cdot|s)||\pi'(\cdot|s))$ the Kullback-Leibler divergence between $\pi(\cdot|s)$ and $\pi(\cdot|s')$, defined as
\begin{align*}
    \DKL{\pi}{\pi'}(s):=\int_\mathcal{A}\log\frac{\pi}{\pi'}(a|s)\pi(\mathrm{d}a|s),
\end{align*}
and is set to $\infty$ if $\pi'(\cdot|s)$ is not absolutely continuous with respect to $\pi(\cdot|s)$.

To regularize the rewards, we add a penalty term that corresponds to the Kullback-Leibler (KL) divergence of the agent's policy and a baseline policy.
In practice, the baseline policy can be used to encode a priori knowledge of the environment; a uniform baseline policy corresponds to adding entropy bonuses to the rewards.
Regularizing with the KL divergence is thus more general than with entropy bonuses and this is the regularization that we consider in this paper, akin to \citet{Schulman17}.

We denote by $\overline{\pi}$ the arbitrary baseline policy and let us assume for conciseness that $\overline{\pi}\in\mathcal{P}$.
Let $\tau>0$ be the so-called \textit{temperature parameter} governing the strength the of the regularization.
Similar to \citet{Ding21,Mei20,geist2019theory} in the stationary setting, we define the $n$-step value function $V_{\pi}^{(n)}:\mathcal{S}\to\mathbb{R}$ induced by a policy $\pi\in\mathcal{P}_n$ as
\begin{align*}
	V_{\pi}^{(n)}:s\mapsto
	\mathbb{E}_{\pi}\left[\sum_{k=0}^{n-1}(R_k-\tau\DKL{\pi^{(n-k)}}{\overline{\pi}}(S_k))\Big|S_0=s\right],
\end{align*}
where the expectation is along the trajectory of length $n$ sampled under policy $\pi=(\pi^{(1)},\ldots,\pi^{(n)})$.
Note that we have
\begin{align}\label{eq: gen case recursive definition value function}
    V^{(n)}_{\pi}(s)
    = \EE_{\pi^{(n)}}[R_0] - &\tau\DKL{\pi^{(n)}}{\overline{\pi}}(s)+\EE_{\pi^{(n)}}[V^{(n-1)}_{\pi'}(S_1)],
\end{align}
where $\pi'=(\pi^{(1)},\ldots,\pi^{(n-1)})\in\mathcal{P}_{n-1}$, and $S_1\sim\int_{\mathcal{A}}p(s,a,\cdot)\pi^{(n)}(\mathrm{d}a|s)$.
It is common to add a discount factor $\gamma\in\intervalleof{0}{1}$ to the rewards to favor more the quickly obtainable rewards.
In the infinite horizon case ($n=\infty$), $\gamma<1$ ensures that the cumulative reward is finite a.s. (provided finite first moment).
Our study trivially applies to the case where the rewards are discounted.

The $n$-step entropy regularized $Q$-function induced by $\pi$ is defined as
\begin{align}\label{eq: gen case definition Q-function}
    Q_{\pi}^{(n)}:(a,s)\mapsto r(a,s)+\int_{\mathcal{S}}p(s,a,\mathrm{d}s')V_{\pi'}^{(n-1)}(s').
\end{align}

\textbf{Notation:}
Henceforth, for a policy $\pi\in\mathcal{P}_n$, we use the abuse of notation $V_\pi^{(i)}$ for $i<n$ for the $i$-step value function associated with $(\pi^{(1)},\ldots,\pi^{(i)})$, and similarly for the $Q$ functions and other quantities of interest, when the context makes it clear which policy is used.

\subsection{Objective and optimal policy}

The standard discounted max-entropy RL objective is defined for stationary policies $\pi\in\mathcal{P}$ by
\begin{align}\label{eq: standard max-RL objective}
    J(\pi):= &\int_{\mathcal{S}}\EE_{\pi_t}\bigg[\sum_{k=0}^T\gamma^k\left(R_k-\DKL{\pi_t}{\overline{\pi}}(S_k)\right)\Big|S_0=s\bigg]\nu(\mathrm{d}s),
\end{align}
where $T\in\NN\cup\{\infty\}$ is the horizon and $\nu$ is the initial state distribution, see e.g. \citet{Eysenbach21} and references therein.
It is often assumed that $T$ is random and therefore $\pi$ is stationary.

Instead of the above objective, we consider the following objective function for non-stationary policy $\pi$:
\begin{align}\label{def: general case objective function}
    J_n(\pi)
    :=\int_{\mathcal{S}}V_{\pi}^{(n)}(s)\nu(\mathrm{d}s).
\end{align}
Since we assume that the rewards are bounded and since the Kullback-Leibler divergence is non-negative, the objective function above is bounded from above by $n||r||_\infty$.


We say that a policy $\pi\in\mathcal{P}_n$ is optimal if and only if $J_n(\pi)\geq J_n(\pi')$ for all $\pi'\in\mathcal{P}_n$.
Note that in general, uniqueness is not guaranteed, since for example $\pi^{(n)}$ only sees states in the support of $\nu$, which can be strictly smaller than $\mathcal{S}$.
If a policy $\pi\in\mathcal{P}_n$ is such that $V_\pi^{(i)}(s)\geq V_{\pi'}^{(i)}(s)$ for all $s\in\mathcal{S}$, all $i=1,\ldots, n$ and all $\pi'\in\mathcal{P}_n$, we say that $\pi$ is \textit{uniformly optimal}.
It is clear that a uniformly optimal policy is in particular optimal.
The existence and unicity of the uniformly optimal policy is established by the next proposition, providing in passing its explicit expression.

\begin{prop}\label{prop: optimal policy}
    There exists a unique uniformly optimal policy, denoted by $\pi_*=(\pi_*^{(1)},\ldots,\pi_*^{(n)})\in\mathcal{P}_n$.
    The $i$-step optimal policies, $i=1,\ldots,n$, can be obtained as follows: for all $a\in\mathcal{A}$, $s\in\mathcal{S}$,
    \begin{align*}
        \pi_*^{(1)}(a|s)
        =\frac{\overline{\pi}(a|s)\exp(r(a,s)/\tau)}{\EE_{\overline{\pi}}[\exp(r(A,s)/\tau)]}, \quad
        \pi_*^{(i+1)}(a|s)
        =\frac{\overline{\pi}(a|s)\exp\left(Q_*^{(i+1)}(a,s)/\tau\right)}{\EE_{\overline{\pi}}\left[\exp\left(Q_*^{(i+1)}(A,s)/\tau\right)\right]},
    \end{align*}
    where $Q_*^{(i+1)}$ is a short-hand notation for $Q_{\pi_*}^{(i+1)}$ recursively defined as in \eqref{eq: gen case definition Q-function}.
\end{prop}

For $i=1,\ldots,n$, let $\mathbf{m}_\pi^{(i)}$ be the law of $S_{n-i}$ under $\pi$ and with given initial state distribution $\nu$.
Note that $\mathbf{m}_\pi^{(n)}=\nu$.
It is readily seen that if $\pi$ is optimal for $J_n$, then necessarily, $\pi^{(n)}(\cdot|s)=\pi_*^{(n)}(\cdot|s)$ for $\nu$-almost every $s$.
In particular, $m_\pi^{(n-1)}=m_{\pi_*}^{(n-1)}$ and then $\pi^{(n-1)}(\cdot|s)=\pi_*^{(n-1)}(\cdot|s)$ for $m_{\pi_*}^{(n-1)}$-almost every $s\in\mathcal{S}$.
Reasoning by induction shows that $\pi^{(i)}(\cdot|s)=\pi_*^{(i)}(\cdot|s)$ for $m_{\pi_*}^{(i)}$-almost every $s\in\mathcal{S}$, for all $i=1,\ldots,n$.
Hence, the optimal policy is unique over the support of the state distributions induced by the uniformly optimal policy.
Since $\pi_*(a|s)>0$ for all $(a,s)\in\mathcal{A}\times\mathcal{S}$, the supports of these state distributions consist of all reachable states from the support of $\nu$.

\begin{lem}\label{lem: general case V_* = log expectation}
    For all $s\in\mathcal{S}$ and $n\geq 1$, it holds that
    \begin{align*}
        V_*^{(n)}(s)
        =\tau\log\EE_{\overline{\pi}}\left[\exp\left(Q_*^{(n)}(A,s)/\tau\right)\right],
    \end{align*}
    where $V_*^{(0)}(s')=0$.
\end{lem}

Thanks to Lemma \ref{lem: general case V_* = log expectation}, we can write more concisely
\begin{align}\label{eq: general case optimal policy}
    \pi_*^{(i)}(a|s)
    &=\overline{\pi}(a|s)\exp\left(\left(Q_*^{(i)}(a,s)-V_*^{(i)}(s)\right)/\tau\right).
\end{align}

For all $n,m\in\NN$ such that $n>m$, we define the operator $T_{n,m}:\mathcal{P}_n\to\mathcal{P}_m$ by
\begin{align}\label{eq: definition translation operator}
    T_{n,m}:
    (\pi^{(1)},\ldots,\pi^{(n)}) & \mapsto (\pi^{(1)},\ldots,\pi^{(m)}).
\end{align}
In Proposition \ref{prop: extending horizon converges to standard optimal policy} below, for all $n\in\NN$, we denote by $\pi_{*,n}\in\mathcal{P}_n$ the uniformly optimal policy with maximum horizon $n$ and with discounted rewards.
The infinite horizon entropy-regularized RL objective is defined in \eqref{eq: standard max-RL objective} with $T=\infty$,
and we denote by $\pi_{*,\infty}$ the corresponding uniformly optimal policy.

\begin{prop}\label{prop: extending horizon converges to standard optimal policy}
    Suppose that the initial state distribution $\nu$ has full support and that the MDP is ergodic. We have:
    \begin{enumerate}[label = (\roman*)]
        \item As $n\to\infty$, the policy $\pi_{*,n}^{(n)}$ converges to $\pi_{*,\infty}$, in the sense that for $\nu$-almost all $s\in\mathcal{S}$,
        \begin{align*}
            \lim_{n\to\infty}\int_{\mathcal{A}}\left|\pi_{*,n}^{(n)}(a|s)-\pi_{*,\infty}(a|s)\right|\mathrm{d}a=0.
        \end{align*}
        \item for all $n,m\in\NN$ such that $n>m$, it holds that $T_{n,m}(\pi_{*,n})=\pi_{*,m}$.
    \end{enumerate}
\end{prop}

The above Proposition \ref{prop: extending horizon converges to standard optimal policy} is intuitive when $\nu$ has full support so that the unique optimal policy is the uniformly optimal policy:
item (i) shows that one can learn the standard discounted entropy-regularized RL objective by incrementally extending the agent's horizon;
item (ii) goes the other way and shows that the uniformly optimal policy for large horizon is built of shorter horizons uniformly optimal policies in a consistent manner.

\section{Matryoshka Policy Gradient}

\subsection{Policy parametrization}
For $i\in\{1,\ldots,n\}$, let $\theta^{(i)}\in\RR^{\nparam_i}$ be the parameters of a linear model $h_{\theta^{(i)}}^{(i)}:\mathcal{A}\times\mathcal{S}\to\mathbb{R}$, that outputs for all $(a,s)\in\mathcal{A}\times\mathcal{S}$ the $i$-step preference $h_{\theta^{(i)}}^{(i)}(a,s)$ for action $a$ at state $s$,
that is,
\begin{align*}
    h_{\theta^{(i)}}^{(i)}(a,s):=\theta^{(i)}\cdot\psi^{(i)}(a,s),
\end{align*}
where $\psi^{(i)}:\mathcal{A}\times\mathcal{S}\to\RR^{\nparam_i}$ is a feature map.
We assume throughout the paper that $\psi^{(i)}$ is a continuous and bounded map, such that for all $s\in\mathcal{S}$ and all $\theta^{(i)}$ with $||\theta^{(i)}||\neq 0$, the map $a\mapsto h_\theta^{(i)}(a,s)$ is not constant.
The $i$-step policy $\pi_{\theta^{(i)}}^{(i)}$ is defined as the Boltzmann policy according to $h^{(i)}$, that is, for all $(a,s)\in\mathcal{A}\times\mathcal{S}$,
\begin{align*}
	\pi_{\theta^{(i)}}^{(i)}(a|s):=\overline{\pi}(a|s)\frac{\exp(h_{\theta^{(i)}}^{(i)}(a,s)/\tau)}{\int_{\mathcal{A}}\exp(h_{\theta^{(i)}}^{(i)}(a',s)/\tau)\overline{\pi}(\mathrm{d}a'|s)}.
\end{align*}
The gradient of the policy thus reads as
\begin{align}\label{eq: gradient policy pi}
    &\nabla\pi_{\theta^{(i)}}^{(i)}(a|s)
    =\pi_{\theta^{(i)}}^{(i)}(a|s)\int_{\mathcal{A}}\left(\delta_{a,\mathrm{d}a'}-\pi_{\theta^{(i)}}^{(i)}(\mathrm{d}a'|s)\right)\nabla h^{(i)}_{\theta^{(i)}}(a',s)/\tau.
\end{align}

\subsection{Definition of the MPG update}

Policy gradient (PG) for max-entropy RL consists in following $\nabla_\theta J(\pi_\theta)$ for the standard objective \eqref{eq: standard max-RL objective}.
In our setting, the ideal PG update would be such that $\theta_{t+1}-\theta_t=\learningrate\nabla_\theta J_n(\pi_t)$.
We introduce Matryoshka Policy Gradient (MPG) as a practical algorithm that produces unbiased estimates of the gradient (see Lemma \ref{lem:Matryoshka_policy_gradient_thm} in Appendix).

Suppose that at time $t\in\NN$ of training, the agent starts at a state $S_0\sim\nu_0$.
To lighten the notation, we write $\pi_t^{(i)}:=\pi_{\theta_t^{(i)}}^{(i)}$.
We assume that the agent samples a trajectory according to the policy $\pi_t$, defined as follows:
\vspace{-0.1cm}
\begin{enumerate}[label = \textbullet]
    \item sample action $A_0$ according to $\pi_t^{(n)}(\cdot|S_0)$,
    \vspace{-0.1cm}
    \item collect reward $R_0\sim p_{\mathrm{rew}}(\cdot|S_0,A_0)$ and move to next state $S_1\sim p(S_0,A_0,\cdot)$,
    \vspace{-0.1cm}
    \item sample action $A_1$ according to $\pi_t^{(n-1)}(\cdot|S_1)$,
    \vspace{-0.1cm}
    \item $\cdots$
    \vspace{-0.1cm}
    \item stop at state $S_n$.
\end{enumerate}
\vspace{-0.1cm}
The MPG update is as follows: for $i=1,\ldots,n$,
\begin{align}\label{eq: update cascade learning}
    \theta_{t+1}^{(i)}
    &=\theta_t^{(i)} + \learningrate\sum_{\ell=n-i}^{n-1}\left(R_{\ell}-\tau\log\frac{\pi_t^{(n-\ell)}}{\overline{\pi}}(A_{\ell}|S_{\ell})\right)
    \nabla\log\pi_t^{(i)}(A_{n-i}|S_{n-i})\nonumber\\
    &=\theta_t^{(i)} + \learningrate C_i\nabla\log\pi_t^{(i)}(A_{n-i}|S_{n-i}),
\end{align}
where we just introduced the shorthand notation $C_i$.
We see that the $i$-step policy $\pi^{(i)}$ is updated using the $(i-\ell)$-step policies.

\subsection{Global convergence: the realizable case}

Recall that $\mathbf{m}_\pi^{(i)}$ denotes the law of $S_{n-i}$ when following policy $\pi$ from a starting state with distribution $\nu$.
We say that a sequence of policies $(\pi_t)_{t\in\NN}\subset\mathcal{P}_n$ converges to $\pi\in\mathcal{P}_n$ if and only if for every $i=1,\ldots,n$, for $\pi$-almost every $a\in\mathcal{A}$ and for $\mathbf{m}_{\pi}^{(i)}$-almost every $s\in\mathcal{S}$, it holds that
\begin{align}\label{eq: def convergence of policies}
  \pi^{(i)}_t(a|s)\underset{t\to\infty}{\longrightarrow} \pi^{(i)}(a|s).
\end{align}
To be concise, the states on which the convergence holds are called \textit{reachable state}, where we keep the dependence in $i$ implicit.
In particular, we will write ``$\pi_1=\pi_2$ on reachable states'' to mean that for all $i=1,\ldots,n$, it holds that $\pi_1^{(i)}(\cdot|s)=\pi_2^{(2)}(\cdot|s)$ for $\mathbf{m}_{\pi_2}^{(i)}$-almost every $s$.

With PG, the so-called \textit{Policy Gradient Theorem} (see Section 13.2 in \citet{Sutton18}) provides a direct way to guarantee convergence of the algorithm.
The analog holds in our setup, that is, if $\theta_{t+1}$ is obtained as in \eqref{eq: update cascade learning}, then $\EE[\theta_{t+1}-\theta_t]=\eta\nabla_\theta J_n(\pi_t)$; it is proven in Appendix \ref{app:on_the_cvg_of_training}.

Importantly, training with true gradient update converges, as stated in the following theorem:

\begin{thm}\label{thm: policy gradient theorem}
    Suppose that $||\psi||:=\sup_{a\in\mathcal{A},s\in\mathcal{S},i=1,\ldots,n}||\psi^{(i)}(a,s)||_2<\infty$, and that $\theta_{t+1}=\theta_t+\eta\nabla_\theta J_n(\pi_t)$ for all $t\geq 0$.
    For all initial parameters $\theta_0$, if $\eta<\frac{2}{L(\theta_0)}$, then it holds that $J_n(\pi_t)$ converges as $t\to\infty$, where 
    \begin{align*}
        L(\theta_0):=4(n^2+n^3)\Big(2+\frac{P}{\tau}\Big)\frac{\Vert\psi\Vert^2}{\tau^2}(J_n(\pi_*)-J_n(\pi_{\theta_0}) + 3^n\Vert r\Vert_\infty) + 4n^2\frac{\Vert\psi\Vert^2}{\tau^2}.
    \end{align*}
    Moreover, $||\nabla_\theta J_n(\pi_t)||_2\to 0$ as $t\to\infty$.    
\end{thm}

Note that Theorem \ref{thm: policy gradient theorem} does not show that the policy $\pi_t$ converges as $t\to\infty$, only that the objective does.
Furthermore, even if $\pi_t$ converges, its limit could be outside of the parametric space, if the parameters during training are such that $||\theta_t||_2\to\infty$ as $t\to\infty$.

For the Theorem below, we assume the following:
\begin{enumerate}[label = \textbf{A\arabic*.}]
    \item \textbf{Realizability assumption}
    There exists $\theta_*\in\RR^P$ such that $\pi_{\theta_*}=\pi_*$.\label{assumption: optimal policy in RKHS}
\end{enumerate}

\begin{thm}\label{thm: general case global optimality}

    Under \ref{assumption: optimal policy in RKHS} and the same assumptions as in Theorem \ref{thm: policy gradient theorem}, $\lim_{t\to\infty}\pi_t=\pi_*$ in the sense of \eqref{eq: def convergence of policies}.
\end{thm}

\paragraph{Intuition of the proof: the bandit case.}

To prove Theorem \ref{thm: policy gradient theorem}, we bound the 2-norm of the Hessian of the objective to show that $\nabla_\theta J_n(\pi_\theta)$ is Lipschitz, with a constant that only decreases along training trajectories, as long as the learning rate is chosen small enough.
However, the objective is non-concave, and it is not obvious that its critical points all correspond to policy $\pi_*$.
We present the heuristics on the illustrative bandit case $|\mathcal{S}|=1$ with horizon $n=1$.
We thus keep the state and horizon implicit below.

We know from Theorem \ref{thm: policy gradient theorem} that $J(\pi_t)$ converges as $t\to\infty$, but the limit could be reached as $\Vert\theta_t\Vert\to\infty$.
However, since there is an optimal $\theta_*\in\RR^P$ by \ref{assumption: optimal policy in RKHS}, for very large $\Vert\theta_t\Vert_2$, the vectors $-\theta_t$ and $\theta_*-\theta_t$ must be almost colinear.
Based on this observation, we show that if the norm of $\theta_t$ is very large, then moving the parameters slightly in the direction $-\theta_t$ improves the performance, showing that $\Vert\theta_t\Vert_2$ remains bounded.

The second step is to show that for all $\theta\in\RR^P$ if $\pi_\theta\neq\pi_*$, then $\theta$ is not a critical point.
The objective is given by $J(\pi_\theta)=J(\pi_*)-\tau\DKL{\pi_\theta}{\pi_*}$.
Without loss of generality, assume that $\tau=1$.
In particular,
\begin{align*}
    \nabla_{\theta} J(\pi_\theta)
    &=-\int_\mathcal{A}\pi_\theta(\mathrm{d}a)\left(\log\frac{\pi_\theta}{\pi_*}(a) + 1\right)\nabla_\theta\log\pi_\theta(a)\\
    &=-\int_\mathcal{A}\pi_\theta(\mathrm{d}a)\log\frac{\pi_\theta}{\pi_*}(a)\nabla_\theta\log\pi_\theta(a),
\end{align*}
where we used that $\EE_{\pi_\theta}[\nabla_\theta\log\pi_\theta(A)]=0$.
For the case $|\mathcal{A}|<\infty$ and tabular parametrisation $\pi_\theta(a)= e^{\theta_a}/\EE_{\overline{\pi}}[e^{\theta_A}]$ and $\nabla_\theta \log\pi_\theta(a)=(\delta_{a}(a')-\pi_\theta(a'))_{a'\in\mathcal{A}}$.
Recall that $\pi_*(a)=e^{r(a)}/\EE_{\overline{\pi}}[e^{r(A)}]$. 
We have
\begin{align*}
    \nabla_\theta J(\pi_\theta)
    &=-\sum_{a\in\mathcal{A}}\pi_\theta(a)\left(\theta_a-r(a) + \mathrm{Const}\right)\nabla_\theta\log\pi_\theta(a)\\
    &=- \sum_{a\in\mathcal{A}}\pi_\theta(a)\left(\theta_a-r(a)\right)(\delta_{a}(a')-\pi_\theta(a'))_{a'\in\mathcal{A}}\\
    &=-\left(\pi_\theta(a')(\theta_{a'}-r(a')-\EE_{\pi_\theta}[\theta_A-r(A)])\right)_{a'\in\mathcal{A}}.
\end{align*}
Hence, the gradient is null if and only if $\theta_{a'}-r(a')=\EE_{\pi_\theta}[\theta_A-r(A)]$ for all $a'\in\mathcal{A}$, that is, if and only if $\theta_{a'}-r(a')$ is constant in $a'$.
This is equivalent to having $\pi_\theta=\pi_*$, which proves that all critical points of $\theta\mapsto J(\pi_\theta)$ encode the optimal policy $\pi_*$ in the bandit case with tabular softmax parametrisation.

For the more general log-linear parametrisation $h_\theta=\theta\cdot\psi$, the gradient times itself yields
\begin{align*}
    \nabla_\theta J(\pi_\theta)\cdot\nabla_\theta J(\pi_\theta)
    &=\int_\mathcal{A}\pi_\theta(\mathrm{d}a)\log\frac{\pi_\theta}{\pi_*}(a)\int_\mathcal{A}\pi_\theta(\mathrm{d}a')\log\frac{\pi_\theta}{\pi_*}(a')\nabla_\theta\log\pi_\theta(a)\cdot\nabla_\theta\log\pi_\theta(a')\\
    &=\int_\mathcal{A}\pi_\theta(\mathrm{d}a)\left(h_\theta(a)-r(a)\right)\int_\mathcal{A}\pi_\theta(\mathrm{d}a')\left(h_\theta(a')-r(a')\right)\widetilde{\Theta}(a,a'),
\end{align*}
where we introduced the matrix $\widetilde{\Theta}(a,a')=\nabla_\theta\log\pi_\theta(a)\cdot\nabla_\theta\log\pi_\theta(a')$, depending implicitly on the parameters.
It is linked to the other matrix $\Theta(a,a'):=\psi(a)\cdot\psi(a')$ since one can check that
\begin{align*}
    \widetilde{\Theta}(a,a')
    =\Theta(a,a')-\EE_{\pi_\theta}[\Theta(A,a')]-\EE_{\pi_\theta}[\Theta(a,A')]+\EE_{\pi_\theta}[\Theta(A,A')].
\end{align*}
The gradient $\nabla_\theta J(\pi_\theta)$ is null if and only if $(h_\theta(a)-r(a))_{a\in\mathcal{A}}$ belongs to the null space of $\widetilde{\Theta}$.
Note, however, that $(h_\theta(a)-r(a))_{a\in\mathcal{A}}$ belongs to the image of $\Theta$, since we assume \ref{assumption: optimal policy in RKHS}.
One can show through standard facts on matrices that the relation between $\Theta$ and $\widetilde{\Theta}$ implies that $\nabla_\theta J(\pi_\theta)$ is null if and only if $h_\theta-r$ is constant, or equivalently, if and only if $\pi_\theta=\pi_*$.

This idea works for infinite or continuous $\mathcal{A}$ using kernels and their reproducible kernel Hilbert space (details are provided in Appendix \ref{appendix: RKHS}).
We also stress that using non-stationary policies is crucial in extending the proof to larger horizons $n>1$, by using that fixing the parameters of the policies with horizon less or equal to $n-1$, the horizon $n$ objective can be seen as a horizon $1$ objective where the rewards are determined by $r$ and the fixed subsequent policies.

\subsection{Global convergence: beyond the realizability assumption}
\label{section main: global convergence beyond the realizability assumption}

Let $\mathscr{P}_n=\{\pi_\theta;\theta\in\RR^P\}\subset\mathcal{P}_n$ be the set of parametric policies.
We now address the case $\pi_*\notin\mathscr{P}_n$, that is, Assumption \ref{assumption: optimal policy in RKHS} does not hold.
We give a sketch of the main ideas to extend Theorem \ref{thm: general case global optimality} to the non-realizable case, showing global convergence and providing a characterisation of the limit. All details are provided in Appendix \ref{appendix: general state space}

We focus on the $1$-step policy.
Suppose that $\vartheta\in\RR^P$ is a critical point of $\theta\mapsto J_n(\pi_\theta)$.
Since $Q_*^{(1)}(a,s)=r(a,s)$, one can show that
\begin{align}\label{eq: gradient is 0 sketch of proof}
    0
    &=\nabla_{\theta^{(1)}}J_n(\pi_\vartheta)
    =-\EE_{\pi_{\vartheta}}\left[\nabla_{\theta^{(1)}}\DKL{\pi_{\vartheta}^{(1)}}{\pi_*^{(1)}}(S_{n-1})\right],
\end{align}
where the law of $S_{n-1}$ only depends on $\pi_{\vartheta}^{(n)},\ldots,\pi_{\vartheta}^{(2)}$.
Since this law is fixed ($\vartheta$ is a critical point), the right-hand side above corresponds to the gradient of a \textit{Bregman divergence} on the set of $1$-step policies, which we denote by $D(\pi_{\vartheta}^{(1)},\pi^{(1)}_*)$.
Let $\pi_{\theta_*}^{(1)}\in\mathrm{argmin}_{\pi^{(1)}_{\theta}\in\mathscr{P}_1}D(\pi_{\theta}^{(1)},\pi_*^{(1)})$.
Bregman divergences satisfy a Pythagorean identity, which in particular implies that
\begin{align*}
    D(\pi_{\vartheta}^{(1)},\pi_*^{(1)})
    &= D(\pi_{\vartheta}^{(1)},\pi_{\theta_*}^{(1)}) + D(\pi_{\theta_*}^{(1)},\pi_*^{(1)}).
\end{align*}
Hence, we have by \eqref{eq: gradient is 0 sketch of proof} that
\begin{align*}
    0=-\nabla_{\theta^{(1)}}D(\pi_{\vartheta}^{(1)},\pi_{\theta_*}^{(1)}).
\end{align*}
We deduce that $\pi_{\vartheta^{(1)}}$ is a critical point of the $1$-step MPG objective, where the initial state distribution is prescribed by $\pi_\vartheta^{(i)}$ for $i=2,\ldots,n$, and where the rewards are given by $r_{\theta_*}:=h_{\theta_*}^{(1)}$.
In particular, Theorem \ref{thm: general case global optimality} applies and shows that $\pi_\vartheta^{(1)}(\cdot|s)=\pi_{\theta_*}^{(1)}(\cdot|s)$ for all reachable states $s$.
This also proves the uniqueness of $\pi_{\theta_*}^{(1)}$ on reachable states.

The argument propagates to larger horizons, by using that maxima can be taken in any order, which proves that the unique critical point $\pi_{\vartheta}$ of $J_n$ is globally optimal.
Formally, the following theorem completes the picture of the global convergence guarantees of MPG:

\begin{thm}\label{Thm: global cvg outside of the RKHS}
    Under the same assumptions as in Theorem \ref{thm: policy gradient theorem}, it holds that $\lim_{t\to\infty}\pi_t=\pi_{\theta_*}$ in the sense of \eqref{eq: def convergence of policies}, where $\pi_{\theta_*}=\mathrm{argmax}_{\pi_\theta\in\mathscr{P}_n}J_n(\pi_\theta)$ is unique on reachable states.
\end{thm}

Clearly, when $\pi_*\in\mathscr{P}_n$, then $\pi_\infty=\pi_*$ on reachable states and we retrieve Theorem \ref{thm: general case global optimality}.

\subsection{Projectional consistency property}

Let $\Theta^{(i)}:(\mathcal{A}\times\mathcal{S})^2\to\RR$ be the positive-semidefinite kernel given by the dot product of the feature map, that is
\begin{align*}
    \Theta^{(i)}((a,s),(a',s'))
    :=\psi^{(i)}(a,s)\cdot\psi^{(i)}(a',s').
\end{align*}
The function space $\{f:(a,s)\mapsto\theta^{(i)}\cdot\psi^{(i)}(a,s);\theta^{(i)}\in\RR^{\nparam_i}\}$ from which we chose the preferences of our parametric Boltzmann policies corresponds to the \textit{reproducible kernel Hilbert space} (RKHS) associated with $\Theta^{(i)}$, that we denote by $\mathcal{H}_{\Theta^{(i)}}$.
Note that when $\mathcal{A},\mathcal{S}$ are finite, with Kronecker delta kernels $\Theta^{(i)}((a,s),(a',s'))=\delta_{a,a'}\delta_{s,s'}$, we retrieve the so-called \textit{tabular case} with one parameter $\theta_{s,a}$ per state-action pair $(s,a)$.
Since we assume that the $\psi^{(i)}$'s are continuous and bounded, it is also the case for the kernels.

The realizability assumption \ref{assumption: optimal policy in RKHS} can be equivalently written as: for all $i=1,\ldots,n$, there exists a map $C_i:\mathcal{S}\to\RR$ such that $(a,s)\mapsto Q_*^{(i)}(a,s)+C_i(s)\in\mathcal{H}_{\Theta^{(i)}}$,
where the maps $C_i$ are constant in $a$.
The $C_i$'s play no role in the policies encoded by functions in the RKHS, since for a fixed $s$, shifting the preferences by a constant keeps the policy unchanged.

It turns out that the global optimum $\pi_{\theta_*}$ from Theorem \ref{thm: general case global optimality} can be characterised by a property of independent interest.
Let $\theta\in\RR^P$ and for all $i=1,\ldots,n$, let $\mathbf{m}_\theta^{(i)}$ be the law of state $S_{n-i}$ under policy $\pi_\theta$ with $\mathbf{m}_\theta^{(n)}=\nu_0$ by assumption.
Define $P_{i}:L^2(\mathbf{m}_\theta^{(i)}(\mathrm{d}s)\pi_\theta^{(i)}(\mathrm{d}a|s))\to\mathcal{H}_{\Theta^{(i)}}$ to be the orthogonal projection onto $\mathcal{H}_{\Theta^{(i)}}$ in the $L^2(\mathbf{m}_\theta^{(i)}(\mathrm{d}s)\pi^{(i)}_\theta(\mathrm{d}a|s))$ sense.    
We say that $\pi_\theta$ satisfies the \textit{projectional consistency property} if and only if for all $i=1,\ldots,n$, it holds that
\begin{align}\label{eq: projectional consistency property}
    \pi^{(i)}_{\theta}(a|s)=\overline{\pi}(a|s)\frac{\exp\left(P_i Q_{\pi_{\theta}}^{(i)}(a,s)/\tau\right)}{\int_\mathcal{A}\overline{\pi}(\mathrm{d}a'|s)\exp\left(P_i Q_{\pi_{\theta}}^{(i)}(a',s)/\tau\right)}.
\end{align}

\begin{prop}\label{prop: projectional consistency property}
    The global optimum $\pi_{\theta_*}$ from Theorem \ref{Thm: global cvg outside of the RKHS} is the only policy in $\mathscr{P}_n$, up to non-reachable states, that satisfies the projectional consistency property \eqref{eq: projectional consistency property}
\end{prop}

\subsection{Neural MPG}

Suppose that instead of a linear model, the policy's preferences $h_\theta^{(i)}$, $i=1,\ldots,n$, are parametrized by deep neural networks.
It is immediate from the proofs that the policy gradient theorem holds true, that is, $\theta_{t+1}-\theta_t=\eta\nabla_\theta J_n(\pi_t)$ for the ideal MPG update.
We describe the limit of training in terms of the \textit{Neural Tangent Kernels} (NTKs) of the neural networks and the \textit{conjugate kernels} (CKs).
The NTK of the $i$-step policy (or rather, of the $i$-step preference) at time $t$ of training is defined for all $(a,s),(a',s')\in\mathcal{A}\times\mathcal{S}$ as
\begin{align*}
    \Theta_t^{(i)}((a,s),(a',s'))
    :=\nabla_{\theta^{(i)}} h_t^{(i)}(a,s)\cdot\nabla_{\theta^{(i)}}h_t^{(i)}(a',s').
\end{align*}
The CK of the $i$-step policy is defined as the inner product of the last hidden layer, that we denote by $\alpha$, that is
\begin{align*}
    \Sigma_t((a,s),(a',s')):=\alpha_t(a,s)\cdot\alpha_t(a',s').
\end{align*}
Moreover, letting $\mathcal{H}_K$ be the induced RKHS of a kernel $K$, it holds that $\mathcal{H}_{\Sigma_t}\subset\mathcal{H}_{\Theta_t}$, see Appendix \ref{appendix: neural networks} for more details.

For the trained policy $\pi_\infty$, let $\mathscr{P}_n^{\Theta}$ be the space of log-linear policies whose $i$-step preference belongs to $\mathscr{H}_{\Theta_\infty^{(i)}}$, $i=1,\ldots,n$, and similarly for $\mathscr{P}_n^{\Sigma}$ and $\Sigma_\infty^{(i)}$.

\begin{cor}\label{thm: gen case deep RL global optimality}
    Let $\pi_t\in\mathcal{P}_n$ be parametrized by neural networks.
    Suppose that $\theta_{t+1}-\theta_t=\eta\nabla_\theta J_n(\pi_t)$ with $\eta>0$ small enough and that $\pi_\infty=\lim_{t\to\infty}\pi_t$ with parameters $||\theta_\infty||_2<\infty$.
    Then, it holds that
    \begin{align*}
        \pi_\infty
        =\underset{\pi\in\mathscr{P}_n^{\Theta}}{\mathrm{argmax}}\ J_n(\pi)
        =\underset{\pi\in\mathscr{P}_n^{\Sigma}}{\mathrm{argmax}}\ J_n(\pi).
    \end{align*}
    In particular, if $\pi_*\in\mathcal{P}_n^{\Theta}$ (equivalently $\mathcal{P}_n^{\Sigma}$), then $\pi_\infty=\pi_*$ on reachable states.
\end{cor}
A direct consequence of Corollary \ref{thm: gen case deep RL global optimality} is global convergence of MPG in the NTK regime, see the forthcoming Remark \ref{rem: consequence corollary deep RL} in Appendix.

\section{Numerical experiments}
\label{Section: numerical experiments}

This section describes the performance of the MPG framework. In the first experiment, we evaluate the MPG framework on an analytical task and show that we converge to the the optimal policy when the optimal policy is realisable, and to the optimal policy in the parameter space when the true optimal policy is not representable by the policy parameterisation. Then, we compare the performance of MPG against REINFORCE \citep{Sutton99} (denoted as PG), REINFORCE with entropy regularisation (softPG) and a non-stationary policy gradient method (nsPG), which is the MPG method without entropy regularisation in two simple control tasks. More details on the implementation, experimental setups and additional results can be found in Appendix \ref{appendix: numerical experiments}.

\subsection{Analytical task}
To numerically evaluate the consequences of Theorem \ref{thm: general case global optimality}, we devise the following analytical problem: consider a state-space consisting of $\mathcal{S}=\{0,1,2,3,4\}$, an action space $\mathcal{A}=\{1,2\}$ with horizon $n=2$. At each state $s$, the agent performs action $a$, taking the agent to the next state $(s+a)\mod 5$ (see appendix \ref{appendix:analytical_task} for fully specified $Q_*^{(1)}$ function, the linear basis $\{e_i;i=1,\ldots,5\}$ considered and experimental setup for the presented experiments). 

We consider the preference function to be represented by a linear model and we consider a true gradient update. Then, we investigate the first two step policies obtained using MPG are when assumption \ref{assumption: optimal policy in RKHS} holds and when it does not. The results are shown in Figure \ref{fig:analytical_task}. Namely, on the left, we  use the full basis $\{e_i;i=1,\ldots,5\}$ for the parametric model, and we are able to find the 1-step and 2-step policies which maximize the objective $J$, and converge towards the optimal 1-step and 2-step policies. On the right of Figure \ref{fig:analytical_task}, we performed the same experiment using an incomplete basis, that cannot express $Q_*^{(1)}$ nor $Q_*^{(2)}$. Namely, we used $\{e_i;i=1,\ldots,4\}$ for both the $1$-step and the $2$-step policies. In this case, we check that the limit is the only policy satisfying the projectional consistency property within the parametric policy space. In both cases, the $L_\infty$-error between the obtained policies and optimal policies go to zero as more episodes are used.


\begin{figure}[h]
\centering
\includegraphics[width=0.45\textwidth]{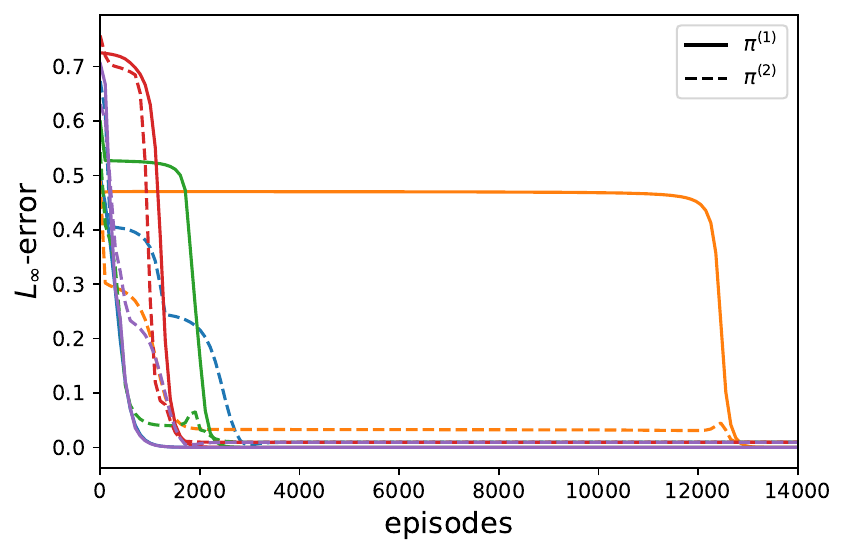}\includegraphics[width=0.45\textwidth]{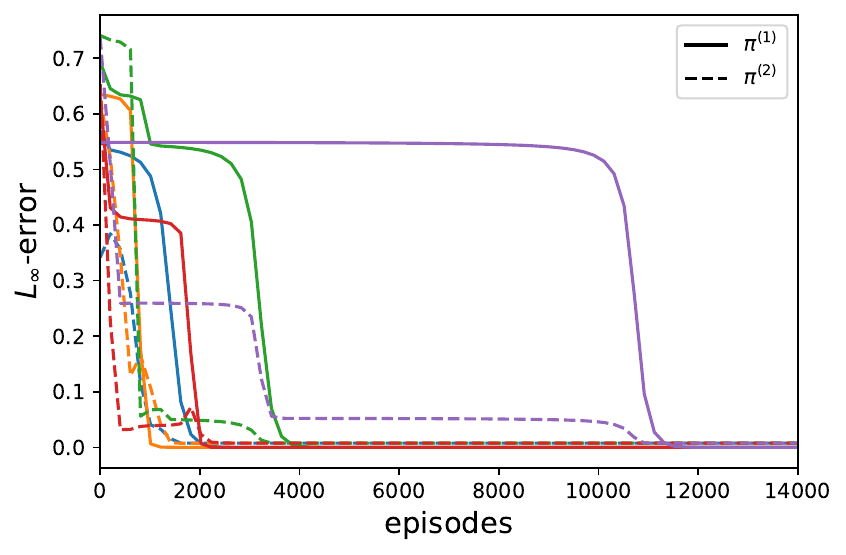}
\caption{Analytical task. Convergence of 5 agents with random initialisation during training; the errors are measured through the $L_\infty$-norm and $\pi^{(1)}$ denotes the one-step policy and $\pi^{(2)}$ the two-step policy. On the left, the convergence of the found 1-step and 2-step policies towards the optimal policies when the parametric space can represent the policy (i.e. when assumption \ref{assumption: optimal policy in RKHS} holds) is shown. On the right, the convergence of the 1-step and 2-step policies towards the optimal projected policies (i.e. when assumption \ref{assumption: optimal policy in RKHS} does not hold) is shown.}
\label{fig:analytical_task}
\end{figure}

\subsection{Control problems}
In this section we present a summary of the performance of the MPG algorithm on two standard RL problems, comparing it to REINFORCE \citep{Sutton99} (denoted as PG), REINFORCE with entropy regularisation (softPG) and a non-stationary policy gradient method (nsPG). In the following experiments, we use a deep neural network to represent the policy (see appendix \ref{appendix: numerical experiments} for architecture) and we estimate the gradient update based on one trajectory as in \eqref{eq: update cascade learning}. 

For both tasks, we follow the experimental protocol as in \citet{patterson2023empirical}, where we first make a sweep over the hyper-parameter spaces considered, evaluating the performance over $3$ agents per set of hyper-parameters, which are initial temperature $\tau_0$ and initial learning rate $\eta_0$ because we decay both the learning rate and temperature -- namely, starting from a higher temperature encourages the exploration of the environment in early stages of training. After this initial stage, we select the best performing sets of hyper-parameters and run more throughout experiments (over $50$ agents) to compare the performance of the different algorithms. For those experiments, confidence intervals around the mean are computed using bootstrap and considering $m=1000$ resampled samples of the mean. In this section we present the results for the the extended experiments, while the preliminary experiments with the hyper-parameter exploration can be found in the appendix \ref{appendix:control_problems}.

\paragraph{Frozen Lake:}
The Frozen Lake benchmark (see \citet{opengym}) features a $4\times 4$ grid composed of cells, holes and one treasure, and a discrete action space, namely, the agent can move in four directions (up, down, left, right). The episode terminates when the agent reaches the treasure or falls down holes. 

For all three algorithms, we considered the hyper-parameter space over initial learning rates and initial temperature, as specified in table \ref{table:frozenlake-hypers}. Then, for each algorithm, we augmented the search space if we found best performing agents at the boundary of the considered initial hyper-parameter space, also denoted in table \ref{table:frozenlake-hypers}. In addition, we considered a horizon length of $N=20$, a terminal $\tau_T = 0.01$, terminal learning rate $\eta_T = 1\times 10^{-6}$ and $1000$ episodes. From these initial runs (results can be found in the appendix), we found the best sets of hyper-parameters for each of the algorithms, denoted on table \ref{table:frozen-lake-best}\footnote{In the cases where there were several sets of good hyper-parameters, we ran an intermediate step with $15$ agents and selected the best set of hyper-parameters.}. Using those hyper-parameters, we ran more extended experiments, now considering $50$ independent agents. In figure \ref{fig:frozen-lake-training}, on the left, we present the training curves, showing the accumulated reward per episode, the shaded regions bound the mean using the $2.5$th percentile and $97.5$th percentile means using bootstrap to compute the confidence interval; on the right, we present the histogram of the cumulative rewards at test time, after training: each agent attempts to solve the task $100$ times. While the performance between MPG and PG was relatively similar: either the agent finds the treasure consistently or fails to find the treasure, training with nsPG or softPG did not yield a good performance in this task; namely, using softPG the agent often gets stuck at moving around the map (which yielded a +0.01 reward at each step, and a cumulative reward of 0.2) instead of finding the treasure. We conducted more experiments using different terminal temperature $\tau_T$ and terminal learning rate $\eta_T$ with little success. Possible ways to address this could be to reshape the reward function further but this was beyond the scope of this current paper.

\begin{table}
\begin{center}
 \caption{Hyper-parameters for Frozen lake \label{table:frozenlake-hypers}}
\begin{tabular}{| c |c | c | c | c | c | } 
 \hline
  & Initial & PG & softPG & nsPG & MPG\\ [0.5ex] 
 \hline\hline
 $\eta_0$ &  \makecell{$\{ 0.01, 5\times 10^{-3}, 1\times 10^{-3},$\\$5\times 10^{-4}, 1\times 10^{-4}\}$} & $\{\}$ & $\{ 0.1,0.05 \}$ & $\{0.5,0.1,0.05\}$ & $\{\}$\\ 
 \hline
 $\tau_0$ & \makecell{$\{0.15,0.2,0.25,0.3,0.35,$\\$0.4,0.45,0.5,0.55,0.6\}$} & NA & $\{0.6,0.65,0.7\}$ & NA &  $\{\}$ \\
 [1ex] 
 \hline
\end{tabular}
\end{center}
\end{table}

\begin{table}
\begin{center}
 \caption{Best set of hyper-parameters for Frozen lake \label{table:frozen-lake-best}}
\begin{tabular}{| c | c | c | c | c | } 
 \hline
  & PG & softPG & nsPG & MPG \\ [0.5ex] 
 \hline\hline
 $\eta_0$ & 0.01 & 0.01 & 0.05 & $5\times 10^{-3}$     \\ 
 \hline
 $\tau_0$ & NA & 0.6 & NA & 0.4 \\
 [1ex] 
 \hline
\end{tabular}
\end{center}
\end{table}

\begin{figure}[h]
\centering
\includegraphics[width=0.49\textwidth]{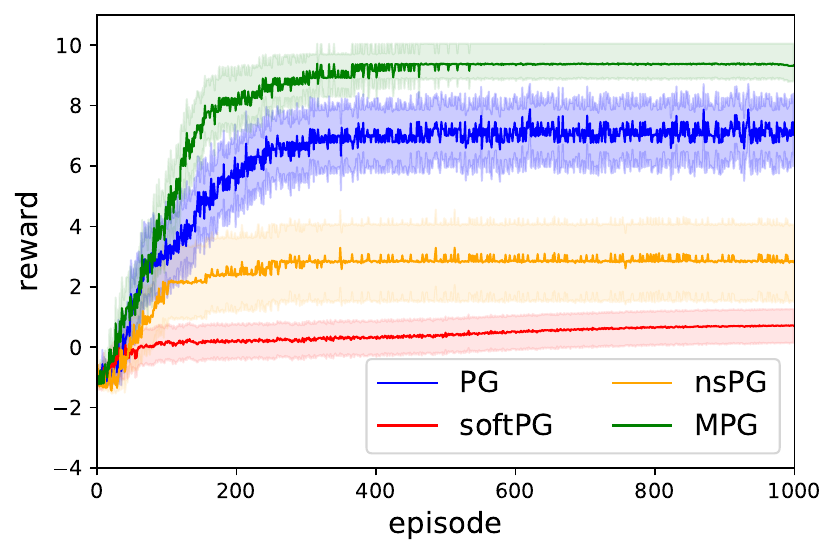} \includegraphics[width=0.49\textwidth]{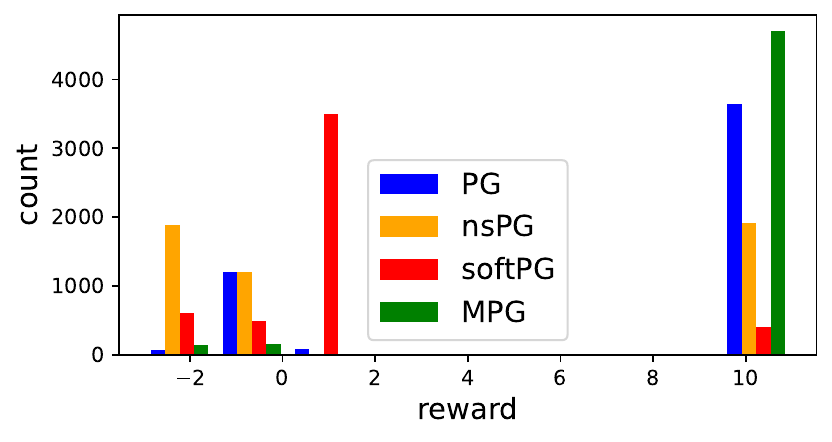}
\caption{Frozen Lake. Left: Cumulative rewards per episode during training time when training using different RL algorithms with the best found set of hyper-parameters. Right: Cumulative rewards per episode after training, each trained agent attempts to solve the task $100$ times.}
\label{fig:frozen-lake-training}
\end{figure}

Furthermore, we noted that when the horizon was decreased (for $N=10$ and $N=15$), we were not able to find the treasure with PG, softPG nor nsPG, whereas training with MPG, the agents would still consistently find the treasure and train successfully.

\paragraph{Cart Pole:}
The Cart Pole benchmark is a classical control problem where a pole is attached by an un-actuated joint to a cart, which moves along a frictionless track. The pole is placed upright on the cart, and the goal is to balance the pole by moving the cart to the left or right for some finite horizon time. It features a continuous environment and a discrete action space. 

We considered a horizon length of $N=100$, a terminal $\tau_T = 0.01$, terminal learning rate $\eta_T = 5\times 10^{-8}$ and $1000$ episodes. For all three algorithms, we considered the hyper-parameter space over initial learning rates and initial temperature, as specified in table \ref{table:cartpole-hypers}, if the best parameter was found at the boundary of the hyper-parameter space, we added another value to the hyper-parameter search. From these initial runs, we found the best sets of hyper-parameters for each of the algorithms, denoted on table \ref{table:cartpole-best}. Using these hyper-parameters, we ran more extended experiments, now considering $50$ independent agents. In figure \ref{fig:cart-pole-rewards}, on the left, we present the training curves, showing the accumulated reward per episode, the shaded regions bound the mean using the $2.5$th percentile and $97.5$th percentile mean using bootstrap to compute the confidence interval; on the right, we present the histogram of the cumulative rewards after training. We note that all algorithms attain quite similar performance, with the entropy regularised ones (MPG and softPG) requiring more episodes to reach the same cumulative reward. This is not surprising, as the agent spends more time exploring the environment at the early stages of training because $\tau$ is larger. We observe that once trained, the testing performance of PG, MPG and nsPG is quite similar, whereas softPG has a larger spread in the cumulative reward, which appears consistent with the larger confidence intervals observed during training.

\begin{table}
\begin{center}
 \caption{Hyper-parameters for balancing cart pole task \label{table:cartpole-hypers}}
\begin{tabular}{| c |c | c | c | c | c | } 
 \hline
  & Initial & PG &softPG & nsPG & MPG \\ [0.5ex] 
 \hline\hline
 $\eta_0$ &  \makecell{$\{ 5\times 10^{-5},$\\$ 1\times 10^{-5},$\\ $5\times 10^{-6},$\\$1\times 10^{-6}\}$} & \makecell{$\{0.005,0.001$,\\$5\times 10^{-4}$,\\$1\times 10^{-4}\}$} & $\{1\times 10^{-4}\}$ & \makecell{$\{0.005,0.001$,\\$5\times 10^{-4},$\\$1\times 10^{-4}\}$} & \makecell{$\{1\times 10^{-4}\}$} \\ 
 \hline
 $\tau_0$ & \makecell{$\{0.1,0.15,$\\$0.20,0.25,$\\$0.3\}$} & NA & $\{0.35,0.4,0.45\}$ & NA & $\{\}$ \\
 [1ex] 
 \hline
\end{tabular}
\end{center}
\end{table}

\begin{table}
\begin{center}
 \caption{Best set of hyper-parameters for the balancing cart pole task \label{table:cartpole-best}}
\begin{tabular}{| c | c | c | c | c |} 
 \hline
  & PG & softPG & nsPG & MPG \\ [0.5ex] 
 \hline\hline
 $\eta_0$ &  0.001 &  $5\times 10^{-5}$ & 0.001 & $5\times 10^{-5}$    \\ 
 \hline
 $\tau_0$ & NA & 0.3 & NA & 0.3 \\
 [1ex] 
 \hline
\end{tabular}
\end{center}
\end{table}

\begin{figure}[h]
\centering
\includegraphics[width=0.49\textwidth]{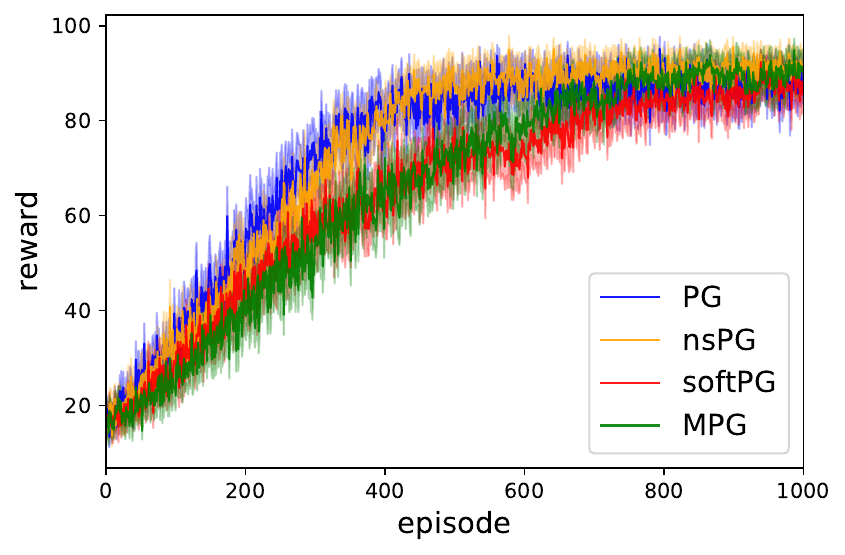}\includegraphics[width=0.49\textwidth]{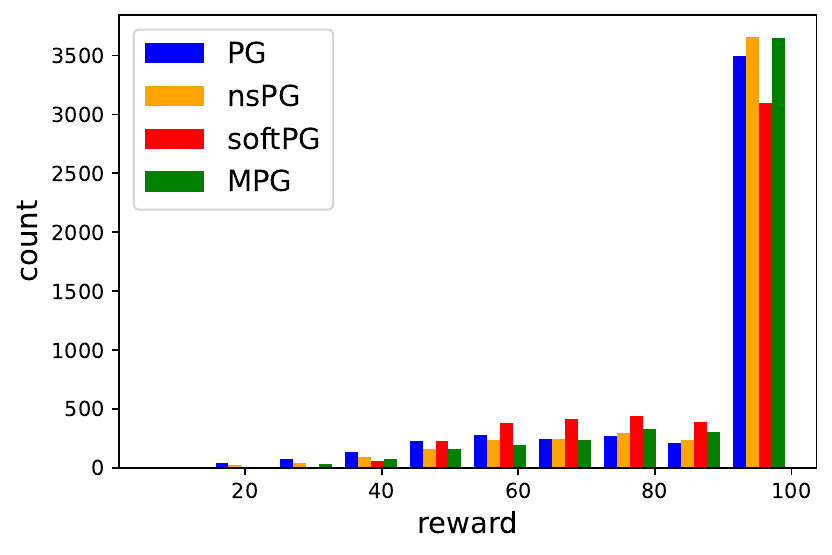}
\caption{Cart Pole. Left: Cumulative rewards per episode during training time when training using different RL algorithms with the best found set of hyper-parameters. Right: Cumulative rewards per episode after training, each trained agent attempts to solve the task $100$ times.  \label{fig:cart-pole-rewards}}
\end{figure} 

\section{Conclusion}

In this paper, we have studied a framework combining fixed-horizon RL and max-entropy RL.
We have introduced the Matryoshka Policy Gradient algorithm in the function approximation setting, with log-linear parametric policies.
We proved that the global optimum of the MPG objective is unique, and that MPG converges to this global optimum, including for continuous state and action space.
Furthermore, we proved that these results hold true even when the true optimal policy does not belong to the parametric space (that is when Assumption \ref{assumption: optimal policy in RKHS} does not hold).
The limit – globally optimal within the parametric space – corresponds to a projection of the optimal policy onto the parametric space.
It is written as the softmax of orthogonal projections of the optimal preferences onto the RKHSs of the parametrization, with respect to the state-visitation measures induced by the policy, see \eqref{eq: projectional consistency property}.
Finally, letting the horizon tend to infinity, the optimal policy of MPG retrieves the optimal policy of the standard infinite-horizon max-entropy objective, when the initial state distribution has full support.
For neural policies, we prove that the limit is optimal within the RKHSs of the NTK (equivalently of the CK) at the end of training, and can be written in terms of orthogonal projections of optimal preferences onto these RKHSs, yielding criterion for global optimality in terms of the NTK (equivalently the CK).
In particular it establishes the global convergence of neural MPG in the NTK regime.
The MPG framework is intuitive, theoretically sound and it is easy to implement. Furthermore, as verified in the numerical experiments, there appears to be an slight advantage to using entropy regularisation and non-stationary policies over the compared PG algorithms. More challenging experiments will be considered in future work.

\paragraph{Limitations.}
The main limitations of our work are the following:
(a) we have not studied the rate of convergence of MPG (typically more assumptions on the environment, the horizon, are needed),
(b) we assumed that perfect gradient updates whereas in practice, one uses the estimate \eqref{eq: update cascade learning}, (c) as a theoretical paper, our numerical experiments are rather simple.
We hope to address these limitations in future work, as well a other perspectives such as:
\begin{itemize}
    \item Additionally to MPG as defined in this paper, we expect to have nice theoretical properties of variations of MPG that are used for other PG algorithms.
    E.g. one can think of natural MPG, actor-critic MPG, path consistency MPG (see \citet{Nachum17} for path consistency learning).
    \item We motivated the use of MPG with neural softmax policies by some theoretical, practical, and heuristic arguments; we believe that more can be said on the use of neural policies with MPG, in particular by studying the spectra of the NTK and the CK of neural networks along specific geodesics in the parametric space.
    \item How does the fixed-horizon max-entropy framework compares to the standard max-entropy RL framework in terms of exploration, adversarial robustness, sample efficiency, and so on?
\end{itemize}

\newpage

\appendix
\addcontentsline{toc}{section}{Appendix}

\section*{Appendix}

The appendix is organized as follows:
\begin{itemize}
    \item \ref{appendix: parametrization}: we recall basic properties of softmax policies, then discuss the potential benefits to using a single neural network for the preferences of all $i$-step policies.
    This section ends with an explanation on how to approximate a kernel with finitely many features.
    \item \ref{appendix: RKHS}: we state and derive some basic facts on RKHS.
    \item \ref{appendix: information geometry}: We use concepts from Information Geometry to show that critical points of the MPG objective correspond to critical points of a \textit{Bregman divergence}; this fact is useful when the realizable assumption does not hold to ensure that MPG converges to the unique global optimum.
    \item \ref{appendix: general state space}: we prove the Matryoshka Policy Gradient Theorem (Theorem \ref{thm: policy gradient theorem}), Proposition \ref{prop: extending horizon converges to standard optimal policy} that shows that the infinite horizon optimal policy can approximated arbitrarily well by finite horizon optimal policies, Theorem \ref{thm: general case global optimality} and Theorem \ref{Thm: global cvg outside of the RKHS} that shows global convergence of MPG.
    \item \ref{appendix: assumptions}: we list and discuss our main assumptions.
    \item \ref{appendix: numerical experiments}: we provide more detailed numerical experiments implementing MPG.
\end{itemize}

\section{More on the parametrization}
\label{appendix: parametrization}

\subsection{Softmax policy}

As long as the map $\psi$ is uniformly bounded, softmax policies enjoy the two following properties:
\begin{itemize}
    \item For all $s\in\mathcal{S}$, it holds that
    \begin{align}\label{eq: softmax gradient cancels expected constant}
        \EE_{\pi_\theta}\left[\nabla_\theta\log\pi_\theta(A|s)\right]
        =\int_{\mathcal{A}}\nabla_\theta\pi_\theta(\mathrm{d}a|s)
        =0.
    \end{align}
    \item It holds that $\pi_\theta(a|s)>0$ for all $(a,s)\in\mathcal{A}\times\mathcal{S}$ such that $\overline{\pi}(a|s)>0$.
\end{itemize}

\subsection{Neural networks}
\label{appendix: neural networks}

\paragraph{Neural Tangent Kernel.}

For a measurable nonlinearity $\sigma:\RR\to\RR$, we recursively define a neural network of depth $L\geq 1$, with trainable parameters $W^{\ell}\in\RR^{d_{\ell}}\times\RR^{d_{\ell+1}}$ as $f:x\in\RR^{d_{0}}\mapsto \widetilde{\alpha}^L(x)\in\RR^{d_{L}}$, with $\alpha^0(x):=x$ and
\begin{align*}
    \widetilde{\alpha}^{\ell+1}(x)&\ := W^\ell\alpha^{\ell}(x),\\
    \alpha^{\ell+1}(x)&\ := \sigma\left(\widetilde{\alpha}^{\ell+1}(x)\right),
\end{align*}
where $\sigma$ is applied element-wise.

Note that the connection between the last hidden layer and the output layer is linear, since $f=W^L\alpha^{L-1}$.
In particular, $f$ belongs to the RKHS of the \textit{conjugate kernel} (CK) associated with the neural network, defined as
\begin{align*}
    \Sigma(x,x'):=\alpha^{L-1}(x)\cdot\alpha^{L-1}(x').
\end{align*}
On the other hand, the training of the neural network is governed by the \textit{neural tangent kernel} (NTK), which is defined as
\begin{align*}
    \Theta(x,x'):=\nabla f(x)\cdot\nabla f(x')
    =\sum_{p=1}^P\partial_{\theta_p}f(x)\partial_{\theta_p}f(x'),
\end{align*}
where $\theta\in\RR^{P}$ is the vector of all the trainable parameters of the neural network.
It is important to note that both the CK and the NTK depend on the parameters and as such, move during training.
Moreover, isolating the derivatives with respect to parameters $W^L$ of the last linear layer from the others $\widetilde{\theta}$, we have that
\begin{align*}
    \Theta(x,x') &= \alpha^{L-1}(x)\alpha^{L-1}(x')+\nabla_{\widetilde{\theta}}f(x)\nabla_{\widetilde{\theta}}f(x')\\
    &=\Sigma(x,x')+K(x,x'),
\end{align*}
and $K$ is another positive semidefinite kernel.
We therefore have that
\begin{align}\label{eq: CK smaller than NTK}
    \mathcal{H}_\Sigma\subset\mathcal{H}_\Theta,\qquad\forall\theta\in\RR^P.
\end{align}

\begin{rem}\label{rem: consequence corollary deep RL}
        For infinitely wide neural networks in the NTK regime \citep{Jacot18} (a.k.a. lazy regime \citep{Chizat18}, kernel regime), the NTK is fixed during training and is strictly positive definite, therefore convergence to the optimal policy is guaranteed.
\end{rem}

\paragraph{Non-stationary policy parametrized by a single neural network.}

One of the assumptions of MPG is that for any $i\neq j$, the policies $\pi^{(i)}_{\theta^{(i)}}$ and $\pi^{(j)}_{\theta^{(j)}}$ do not share parameters.
Using one neural network per horizon becomes quickly costly as the maximal horizon increases.
In order to avoid this issue, one can use a single neural network $h_{\theta}$ to parametrize all $i$-step policies by using $i$ as an input such that $\pi^{(i)}_{\theta}(a|s)\propto \overline{\pi}(a|s)\exp(h_{\theta}(a,s,i)/\tau)$.
By deviating from the theory, we nonetheless expect the performance of the model to be enhanced: as $i$ grows large, the $i$-step optimal policy gets closer to the $i+1$-step policy.
One could also use $1-\frac{1}{i}$ as an input to the network (or any increasing map $g:\NN\mapsto \intervalleff{0}{1}$ such that $i\mapsto g(i+1)-g(i)$ is decreasing).

\subsection{Kernel methods}
\label{appendix: kernel methods}

Suppose that $\Theta$ is a strictly pd kernel with $\nparam$ positive eigenvalues.
Recall the linear model $a\mapsto h_\theta(a) = \theta\cdot\psi(a)$, with parameters $\theta\in\RR^P$, such that $\psi$ is a feature map associated with $\Theta$.
Then if $P=\infty$, one can use random features, i.e. sample $g_1,\ldots,g_{\nparam'}$ i.i.d. Gaussian processes with covariance kernel $\Theta$, then $h_\theta:=\frac{1}{\sqrt{\nparam'}}\sum_{i=1}^{\nparam'} \theta_i g_i$.
One can thus approximate the true kernel predictor using a finite number of features, see \citet{Jacot20ImplicitRO}.

Another way to approximate the kernel predictor with finitely many features is to use the spectral truncated kernel $\widehat{\Theta}$ of rank $P'\in\NN$, by cutting off the smallest eigenvalues.
If $(e_i,\lambda_i)_{i\geq 1}$ are the eigenfunction/eigenvalue pairs of $\Theta$ ranked in the non-increasing order of $\lambda_i$, one can use
\begin{align*}
    \widehat{\Theta}(x,x'):=\sum_{i=1}^{\nparam'}\lambda_ie_i(x)e_i(x'),
\end{align*}
and the predictor $h_\theta:=\sum_{i=1}^{P'}\theta_ie_i$.

\section{Reproducible kernel Hilbert spaces}
\label{appendix: RKHS}

In this section, we recall and provide some basic facts on RKHSs that we use throughout the proofs.
Given some RKHS $\mathcal{H}$, we write $\mathcal{H}^\perp$ for its orthogonal complement; it is also an RKHS.

\begin{lem}\label{lem: facts on RKHS}
    Let $\mathcal{H}_1,\mathcal{H}_2$ be two RKHSs on $\mathcal{A}\times\mathcal{S}$,
    \begin{enumerate}[label = (\roman*)]
        \item The intersection $\mathcal{H}_1\cap\mathcal{H}_2$ is an RKHS.
        \item For any element $f\in\mathcal{H}_1$, there exists a unique decomposition $f=g_\bullet+g_\perp$ such that $g_\bullet\in\mathcal{H}_1\cap\mathcal{H}_2$ and $g_\perp\in\mathcal{H}_1\cap(\mathcal{H}_2)^\perp$.
    \end{enumerate}
\end{lem}

For a probability measure of the form $\mu(\mathrm{d}s)\pi(\mathrm{d}a|s)$ on $\mathcal{A}\times\mathcal{S}$, where $\pi$ is a policy, and for a positive-semidefinite kernel $K$ on $\mathcal{A}\times\mathcal{S}$, we define the integral operator $I_K(\mu,\pi):L^2(\mu(\mathrm{d}s)\pi(\mathrm{d}a|s))$ by
\begin{align*}
    I_K(f;\mu,\pi):(a,s)\mapsto\int_{\mathcal{A}\times\mathcal{S}}\mu(\mathrm{d}s')\pi(\mathrm{d}a'|s')f(a',s')K((a,s),(a',s')).
\end{align*}

Mercer's Theorem states that if $\mathcal{A}\times\mathcal{S}$ is closed (in a real space), 
and if $K$ is continuous and satisfies $\int_{(\mathcal{A}\times\mathcal{S})^2}K((a,s),(a',s'))^2\pi(\mathrm{d}a|s)\mu(\mathrm{d}s)\pi(\mathrm{d}a'|s')\mu(\mathrm{d}s')<\infty$, then there exists eigenfunction/eigenvalue pairs $(e_i,\lambda_i)_{i\geq 1}$ associated with $I_K(\mu,\pi)$, ranked in the non-increasing order of $\lambda_i\geq 0$ such that
\begin{align*}
    K((a,s),(a',s'))=\sum_{i\geq 1}\lambda_ie_i(a,s)e_i(a',s').
\end{align*}
Moreover, $\{e_i;i\geq 1\}$ is an orthonormal basis of $L^2(\mu(\mathrm{d}s)\pi(\mathrm{d}a|s))$ and the RKHS $\mathcal{H}_{K}$ has orthonormal basis $\{\sqrt{\lambda_i}e_i;\lambda_i>0\}$ with respect to the RKHS inner product.
We refer the reader to \citet{minh06} for more details.

We stress that the notion of orthogonality \textbf{depends} on the measure $\mu(\mathrm{d}s)\pi(\mathrm{d}a|s)$.
Henceforth, we write $\mathcal{H}^\perp$ for the orthogonal space of the RKHS, where this measure is implicit but given by the context.

In the rest of the current section, we use the notation introduced above and assume that Mercer's Theorem applies.

\begin{lem}\label{lem: if I_Kf=0 then in ortho}
    Let $f\in L^2(\mu(\mathrm{d}s)\pi(\mathrm{d}a|s))$.
    It holds that $I_K(f;\mu,\pi)(a,s)=0$ for all $a\in\mathcal{A},s\in\mathcal{S}$ if and only if $f\in(\mathcal{H}_K)^\perp$.
\end{lem}

\begin{proof}[Proof]
    We write
    \begin{align*}
        &\int_{\mathcal{A}\times\mathcal{S}}\mu(\mathrm{d}s)\pi(\mathrm{d}a|s)f(a,s)I_K(f;\mu,\pi)(a,s)\\
        &\hspace{1cm}=\int_{\mathcal{A}\times\mathcal{S}}\mu(\mathrm{d}s)\pi(\mathrm{d}a|s)(a,s)\int_{\mathcal{A}\times\mathcal{S}}\mu(\mathrm{d}s')\pi(\mathrm{d}a'|s')f(a,s)f(a',s')K((a,s),(a',s'))\\
        &\hspace{1cm}=\sum_{i\geq 1}\lambda_i\left(\int_{\mathcal{A}\times\mathcal{S}}\mu(\mathrm{d}s)\pi(\mathrm{d}a|s)f(a,s)e_i(a,s)\right)^2,
    \end{align*}
    where we used Mercer's Theorem to write $K((a,s),(a',s'))=\sum_{i\geq 1}\lambda_ie_i(a,s)e_i(a',s')$.
    The claim follows.
\end{proof}

\begin{lem}\label{lem: softmax RKHS is larger}
    It holds that
    \begin{align*}
        \widetilde{K}((a,s),(a',s'))
        :=\ &K((a,s),(a',s')) - \int_\mathcal{A}K((b,s),(a',s'))\pi(\mathrm{d}b|s)\\
        &\hspace{0.7cm}-\int_\mathcal{A}K((a,s),(b',s'))\pi(\mathrm{d}b'|s')
        +\int_{\mathcal{A}^2}K((b,s),(b',s'))\pi(\mathrm{d}b|s)\pi(\mathrm{d}b'|s')
    \end{align*}
    is a positive-semidefinite kernel.
    Furthermore, any map $g\in\mathcal{H}_{K}\cap(\mathcal{H}_{\widetilde{K}})^{\perp}$ is such that for $\mu$-almost every $s\in\mathcal{S}$, the map $a\mapsto g(a,s)$ is constant.
\end{lem}

\begin{proof}[Proof]
    Let $d:=\sup\{i\geq 1:\lambda_i>0\}$ where the $\lambda_i$'s are the eigenvalues of $I_K$.
    To prove the first part of the claim, it suffices to show that for all $g\in L^2(\mu(\mathrm{d}s)\pi(\mathrm{d}a|s))$, we have
    \begin{align}\label{eq: to prove for larger softmax RKHS}
        \int_{(\mathcal{S}\times\mathcal{A})^2}\mu(\mathrm{d}s)\pi(\mathrm{d}a|s)\mu(\mathrm{d}s')\pi(\mathrm{d}a'|s')g(a,s)g(a',s')\widetilde{K}((a,s),(a',s'))\geq 0.
    \end{align}
    To ease the notation, for any maps $f,g\in L^2(\mu(\mathrm{d}s)\pi(\mathrm{d}a|s))$ we write
    \begin{align*}
        \langle f,g\rangle
        :=\ &\int_{\mathcal{S}\times\mathcal{A}}\mu(\mathrm{d}s)\pi(\mathrm{d}a|s)f(a,s)g(a,s),\\
        \overline{f}(s)
        :=\ &\int_\mathcal{A}\pi(\mathrm{d}a|s)f(a,s).
    \end{align*}
    
    We now establish \eqref{eq: to prove for larger softmax RKHS}.
    Using that $K((a,s),(a',s'))=\sum_{j\leq d}\lambda_j e_j(a,s)e_j(a',s')$, we get
    \begin{align*}
        \widetilde{K}((a,s),(a',s'))
        =\sum_{j\leq d}\lambda_j (e_j(a,s)-\overline{e}_j(s))(e_j(a',s')-\overline{e}_j(s')).
    \end{align*}
    The left-hand side of \eqref{eq: to prove for larger softmax RKHS} thus reads as
    \begin{align*}
        \sum_{j\leq d}\lambda_j\left(\langle g,e_j\rangle^2
        - 2\langle g,e_j\rangle \langle \overline{g},\overline{e}_j\rangle
        + \langle \overline{g},\overline{e}_j\rangle^2 \right)
        &=\sum_{j\leq d}\lambda_j\left(\alpha_j^2
        - 2\alpha_j \langle \overline{g},\overline{e}_j\rangle
        + \langle \overline{g},\overline{e}_j\rangle^2 \right)\\
        &=\sum_{j\leq d}\lambda_j\left(\alpha_j
        - \langle \overline{g},\overline{e}_j\rangle \right)^2.
    \end{align*}
    The right-hand side above being clearly non-negative, this shows that $\widetilde{K}$ is positive-semidefinite.

    We now turn our attention to the last part of the claim.
    Suppose that $g\in\mathcal{H}_K\cap(\mathcal{H}_{\widetilde{K}})^{\perp}$, so that we can write $g=\sum_{j\leq d}\alpha_je_j$, with $\alpha_j=\langle g, e_j\rangle$.
    Moreover, by Lemma \ref{lem: if I_Kf=0 then in ortho}, we have an equality in \eqref{eq: to prove for larger softmax RKHS}, and we get
    \begin{align*}
        \sum_{j\leq d}\lambda_j(\langle g,e_j\rangle-\langle \overline{g},\overline{e}_j\rangle)^2=0.
    \end{align*}
    We thus necessarily have $\langle g,e_j\rangle=\langle \overline{g},\overline{e}_j\rangle$ for all $j\leq d$.
    In particular,
    \begin{align*}
        \langle g,g\rangle
        &=\sum_{j\leq d} \alpha_j^2
        =\sum_{j\leq d} \alpha_j\langle \overline{g}, \overline{e}_j \rangle
        =\langle \overline{g},\overline{g}\rangle.
    \end{align*}
    On the other hand, Cauchy-Schwarz Inequality shows that if $s\mapsto g(a,s)$ is not constant in $a$ for all $s$, then
    \begin{align*}
        \langle g,g\rangle
        &=\int_{\mathcal{S}}\mu(\mathrm{d}s)\int_\mathcal{A}\pi(\mathrm{d}a|s)g(a,s)^2\\
        &>\int_{\mathcal{S}}\mu(\mathrm{d}s)\left(\int_\mathcal{A}\pi(\mathrm{d}a|s)|g(a,s)|\right)^2\\
        &\geq \langle \overline{g},\overline{g}\rangle.
    \end{align*}
    This is a contradiction and thus implies that $g$ must be constant in $a$.
\end{proof}
Since the feature maps we consider in this work are continuous, necessarily, the map $g$ in the above is constant for $\mu$-almost every $s\in\mathcal{S}$ if and only if it is constant for all $s$ in the support of $\mu$.

\section{Information Geometry}
\label{appendix: information geometry}

The goal of this section is to show that a Pythagorean identity that is used in the forthcoming proof of Theorem \ref{Thm: global cvg outside of the RKHS}.
We use it in the case of a $1$-step policy, and for a fixed state distribution, that we denote by $\nu$ in this section.
Without loss of generality, we also assume to ease the notation that $\tau=1$.

Consider the parametric space of preferences $\mathcal{H}_\Theta:=\{h_\theta=\sum_{k=1}^d\theta_k\psi_k;\theta\in\RR^d\}$, where $\psi$ is the feature map of a positive definite kernel $\Theta$ that we assume to be continuous and bounded.
The space $\mathcal{H}_\Theta$ is the RKHS associated with $\Theta$.
Fix $\vartheta\in\RR^d$ and let $\pi_\vartheta$ be the $1$-step policy induced by the preference $h_\vartheta$ (with baseline policy $\overline{\pi}$ as usual).

Up to reparametrization (potentially decreasing the value of $d$), we can assume without loss of generality that $\mathcal{H}_\Theta:=\{h_\theta=\sum_{k=1}^d\theta_k\varphi_k;\theta\in\RR^d\}$
where $\{\varphi_k;k=1,\ldots,d\}$ is an orthonormal basis of $\mathcal{H}_\Theta$ in $L^2(\nu(\mathrm{d}s)\pi_\vartheta(\mathrm{d}a|s))$.
For $k\geq 1$, define $\varphi_k$ such that $\{\varphi_k;k\geq d+1\}$ is an orthonormal basis of $(\mathcal{H}_\Theta)^\perp$, the orthogonal complement of $\mathcal{H}_\Theta$ in $L^2(\nu(\mathrm{d}s)\pi_\vartheta(\mathrm{d}a|s))$.


The map
\begin{align*}
    F:\theta&\mapsto\int_\mathcal{S}\nu(\mathrm{d}s)\log\int_\mathcal{A}\overline{\pi}(\mathrm{d}a|s)e^{h_\theta(a,s)}\\
    \RR^{d'}&\to\RR
\end{align*}
is strictly convex.
Indeed, it is straightforward to compute
\begin{align*}
    \partial_{\theta_i}F(\theta)
    &=\int_\mathcal{S}\nu(\mathrm{d}s)\int_\mathcal{A}\pi_\theta(\mathrm{d}a|s)\varphi_i(a,s),
\end{align*}
and then
\begin{align*}
    \nabla_\theta\nabla_\theta F(\theta)
    &=\left(\int_\mathcal{S}\nu(\mathrm{d}s)\int_\mathcal{A}\pi_\theta(\mathrm{d}a|s)\varphi_i(a,s)(\varphi_j(a,s)-\EE_{\pi_\theta}[\varphi_j(A,s)])\right)_{i,j\leq d'}\\
    &=\int_{\mathcal{S}}\nu(\mathrm{d}s)\mathrm{Var}_{\pi_\theta}\left[\varphi(A,s)\right],
\end{align*}
where $\mathrm{Var}_{\pi_\theta}\left[\varphi(A,s)\right]$ is the covariance matrix of $\varphi(A,s)$ for $A\sim\pi_\theta(\cdot|s)$.
We thus have that
\begin{align*}
    \theta^T\nabla_\theta\nabla_\theta F(\theta)\theta
    &=\int_\mathcal{S}\nu(\mathrm{d}s)\int_\mathcal{A}\pi_\theta(\mathrm{d}a|s)\left(h_\theta(a,s)-\EE_{\pi_\theta}[h_\theta(A,s)]\right)^2,
\end{align*}
which is non-negative, and null if and only if $h_\theta(a,s)=\EE_{\pi_\theta}[h(A,s)]$ for all $a\in\mathcal{A}$, that is, if and only if $\theta=0$.
This shows that $\nabla_\theta\nabla_\theta F(\theta)$ is strictly convex on $\RR^{d'}$.

As a strictly convex map, $F$ induces a \textit{Bregman divergence} on the quotient space $\mathscr{P}(d')$, where $\mathscr{P}(d')$ is the set of softmax policies with preferences parametrized by $\theta\in\RR^{d'}$.
The Bregman divergence is defined as
\begin{align*}
    D_F(\theta,\theta'):&=F(\theta)-F(\theta')-\nabla F(\theta')\cdot(\theta-\theta')\\
    &=\int_{\mathcal{S}}\nu(\mathrm{d}s)\DKL{\pi_{\theta'}}{\pi_\theta}(s).
\end{align*}
More generally, $D_F(\pi,\pi'):=\int_{\mathcal{S}}\nu(\mathrm{d}s)\DKL{\pi'}{\pi}(s)$ is well defined for any policies $\pi,\pi'\in\mathcal{P}$.
One can define a dual coordinate system $\xi(\theta):=\nabla F(\theta)$, and the manifold $\mathscr{P}(d')$ is said to be \textit{dually flat}, as each coordinate system induces a notion of flatness.

Recall that $\vartheta\in\RR^d$ is fixed and let $\widehat{Q}\in L^2(\nu(\mathrm{d}s)\pi_\vartheta(\mathrm{d}a|s))$, with $\widehat{\pi}$ the induced policy.
In particular, we can write $\widehat{Q}=\sum_{k=1}^\infty \widehat{\theta}_k\varphi_k$.
For all $d'\geq 1$, let $\widehat{\pi}_{d'}\in\mathscr{P}(d')$ be the policy induced by $\widehat{h}_{d'}:=\sum_{k=1}^{d'}\widehat{\theta}_k\varphi_k\in\mathcal{F}_{d'}$.
In particular, by Theorem 1.5 in \citet{amari16} p.27, we have that
\begin{align*}
    \widehat{\pi}_{d}^{d'}:=\mathrm{argmin}_{\pi_\theta\in\mathscr{P}(d)}D_F(\widehat{\pi}_{d'},\pi_\theta)
\end{align*}
is unique in $\mathscr{P}(d)$,
and moreover,
\begin{align*}
    D_F(\widehat{\pi}_{d'},\pi_\vartheta)
    =D_F(\widehat{\pi}_{d'},\widehat{\pi}_{d}^{d'})+D_F(\widehat{\pi}_{d}^{d'},\pi_\vartheta).
\end{align*}
We now extend this identity to the infinite dimensional case, that is, with $\widehat{\pi}$ in place of $\widehat{\pi}_{d'}$.

Firstly, it is clear that $\widehat{\pi}_{d'}\to\widehat{\pi}$ as $d'\to\infty$.
Since $D_F$ is continuous, the Maximum Theorem (see p.116 of \citet{berge63}) entails that
\begin{align*}
    \widehat{\pi}^{\infty}_d:=\lim_{d'\to\infty}\widehat{\pi}^{d'}_d=\mathrm{argmin}_{\pi_\theta\in\mathscr{P}(d)}D_F(\widehat{\pi},\pi_\theta),
\end{align*}
and then
\begin{align*}
    D_F(\widehat{\pi},\pi_\vartheta)
    =D_F(\widehat{\pi},\widehat{\pi}_{d}^{\infty})+D_F(\widehat{\pi}_{d}^{\infty},\pi_\vartheta).
\end{align*}
(Alternatively, one can show the above as Equation (4) in \citet{fukumizu05})

From the above equation we easily deduce the following Lemma:

\begin{lem}\label{lem: same critical points for divergence}
    With the notation introduced above, the vector $\vartheta\in\RR^d$ is a critical point of $\theta\mapsto D_F(\widehat{\pi},\pi_\theta)$ if and only if it is a critical point of $\theta\mapsto D_F(\widehat{\pi}_d^{\infty},\pi_\theta)$.
\end{lem}

\section{Proofs}
\label{appendix: general state space}

\subsection{Matryoshka Policy Gradient Theorem and convergence of the objective}

The following property is simple yet useful for later computations.

\begin{lem}\label{lem: general case value function is optimal value function minus DKL}
    For all $n\geq 1$, all $\pi\in\mathcal{P}_n$ and all $s\in\mathcal{S}$, it holds that
    \begin{align*}
        &V_{\pi}^{(n)}(s)
        -V_*^{(n)}(s)
        =-\tau\EE_{\pi}\left[\left.\sum_{i=0}^{n-1}\DKL{\pi^{(n-i)}}{\pi_*^{(n-i)}}(S_i)\right|S_0=s\right].
    \end{align*}
\end{lem}

\begin{proof}[Proof]
    Recall \eqref{eq: gen case recursive definition value function} and write
    \begin{align*}
        V_{\pi}^{(n)}(s)
        &=\int_{\mathcal{A}}\pi^{(n)}(\mathrm{d}a|s)\bigg(r(a,s) - \tau\log\frac{\pi^{(n)}}{\overline{\pi}}(a|s) + \int_{\mathcal{S}}p(s,a,\mathrm{d}s')V^{(n-1)}_{\pi}(s')\bigg)\\
        &=\int_{\mathcal{A}}\pi^{(n)}(\mathrm{d}a|s)\bigg(V_*^{(n)}(s) - \tau\log\frac{\pi^{(n)}}{\pi_*}(a|s) + \int_{\mathcal{S}}p(s,a,\mathrm{d}s')(V^{(n-1)}_{\pi}(s')-V_*^{(n-1)}(s'))\bigg),
    \end{align*}
    where we plugged in the expression of the optimal policy \eqref{eq: general case optimal policy}.
    We can rewrite the above as
    \begin{align*}
        &V_{\pi}^{(n)}(s)-V_*^{(n)}(s)
        =-\tau\DKL{\pi^{(n)}}{\pi_*^{(n)}} + \EE_\pi\left[V^{(n-1)}_{\pi}(S_1)-V_*^{(n-1)}(S_1)\Big|S_0=s\right].
    \end{align*}
    The claim follows by induction.
\end{proof}

The next lemma provides bounds that are needed to take derivatives of the objective for the proof of Theorem \ref{thm: policy gradient theorem} later on.

\begin{lem}\label{lem:bounds_derivatives_policy}
    For all $a\in\mathcal{A}$, $s\in\mathcal{S}$, $\theta,\theta'\in\RR^P$ and all $i\in\{1,\ldots,n\}$, it holds that
    \begin{enumerate}[label = (\roman*)]
        \item $\Vert\nabla_\theta\pi_\theta^{(i)}(a|s)\Vert_2\leq \frac{2}{\tau}\Vert\psi\Vert\pi_\theta^{(i)}(a|s)$;
        \item $\Vert\nabla_\theta^2\log\pi_\theta^{(i)}(a|s)\Vert_2\leq \frac{2P_i}{\tau^2}\Vert\psi\Vert^2$;
        \item $\int_\mathcal{A}\pi^{(i)}_\theta(\mathrm{d}a|s)\Big\vert\log\frac{\pi_\theta^{(i)}}{\pi_*^{(i)}}(a|s)\Big\vert\leq \DKL{\pi_\theta^{(i)}}{\pi_*^{(i)}}(s) + \frac{3^n}{\tau}\Vert r\Vert_\infty$;
        \item Let $f_\theta:\mathcal{A}\times\mathcal{S}\mapsto\RR$ be differentiable with respect to $\theta\in\RR^P$ such that $\Vert f\Vert_\infty\leq C$ and $\Vert\nabla_\theta f_\theta(a,s)\Vert_2\leq C(1+\Vert \theta\Vert_2)$ for all $(a,s)\in\mathcal{A}\times\mathcal{S}$ for some constant $C>0$, then it holds that
        \begin{align*}
            \nabla_\theta\EE_{\pi_\theta}\left[f_\theta(A_{i},S_{i})\right]
            &= \EE_{\pi_\theta}\bigg[\bigg(\sum_{j=0}^{i}\nabla_\theta\log\pi_\theta^{(n-j)}(A_j|S_j)\bigg)f_\theta(A_i,S_i)\bigg] + \EE_{\pi_\theta}\left[\nabla_\theta f_\theta(A_i,S_i)\right].
        \end{align*}
    \end{enumerate}
\end{lem}

\begin{proof}[Proof]
    $(i)$
    Recall \eqref{eq: gradient policy pi}.
    We have $\nabla_\theta\pi_\theta^{(i)}(a|s)=\pi_\theta^{(i)}(a|s)(\psi^{(i)}(a,s)-\EE_{\pi_\theta}[\psi^{(i)}(A,s)])/\tau$.
    In particular, $\Vert\nabla_\theta \pi_\theta^{(i)}(a|s)\Vert_2 \leq  2\pi_\theta^{(i)}(a|s)\Vert\psi\Vert/\tau$, where we recall that $\Vert\psi\Vert=\sup_{a,s,i}\Vert\psi^{(i)}(a,s)\Vert_2$.
    This proves $(i)$.
    \vspace{0.5cm}
    
    $(ii)$
    We compute
    \begin{align*}
        \nabla_\theta^2\log\pi_\theta^{(i)}(a|s)
        &=\nabla_\theta\big(\psi^{(i)}(a,s-\EE_{\pi_\theta}[\psi^{(i)}(A,s)])\big)/\tau\\
        &=-\frac{1}{\tau^2}\EE_{\pi_\theta^{(i)}}\Big[\psi^{(i)}(A,s)\big(\psi^{(i)}(A,s)-\EE_{\pi_\theta}[\psi^{(i)}(A',s)]\big)^{T}\Big].
    \end{align*}
    Since the $2$-norm of a matrix is upper-bounded by its Frobenius norm, the above entails that $\Vert \nabla_\theta^2\log\pi_\theta(a|s)\Vert_2^2\leq \sum_{k,\ell\leq P}(\nabla_\theta^2\log\pi_\theta(a|s))_{k,\ell}^2\leq \frac{4P^2}{\tau^4}\Vert \psi\Vert^4$, which proves $(ii)$.
    
    \vspace{0.5cm}
    
    $(iii)$
    Let $i\in\{1,\ldots,n\}$.
    We note that
    \begin{align*}
        \int_{\mathcal{A}}\pi_\theta^{(i)}(\mathrm{d}a|s)\bigg|\log\frac{\pi_\theta^{(i)}}{\pi_*^{(i)}}(a|s)\bigg|
        &=\DKL{\pi_\theta^{(i)}}{\pi_*^{(i)}}(s)+2\int_\mathcal{A}\pi_\theta^{(i)}(\mathrm{d}a|s)\log\frac{\pi_*^{(i)}}{\pi_\theta^{(i)}}(a|s)\mathds{1}_{\{\pi_\theta^{(i)}(a|s)<\pi_*^{(i)}(a|s)\}}.
    \end{align*}
    We claim that for all $\theta\in\RR^P$ and all $s\in\mathcal{S}$, it holds that
    \begin{align}\label{eq:bound_to_prove_lemma_lipschitz}
        \int_\mathcal{A}\pi_\theta^{(i)}(\mathrm{d}a|s)\log\frac{\pi_*^{(i)}}{\pi_\theta^{(i)}}(a|s)\mathds{1}_{\{\pi_*^{(i)}(a|s)>\pi_\theta^{(i)}(a|s)\}} 
        &\leq e^{-1}+\frac{3^n}{2\tau}||r||_\infty.
    \end{align}
    To lighten the notation, let us keep the variables $a$ and $s$ implicit in the calculations.
    To establish \eqref{eq:bound_to_prove_lemma_lipschitz}, we write by definition that
    \begin{align*}
        0&\leq\int_{\mathcal{A}}\mathrm{d}\pi_\theta^{(i)}\log\frac{\pi_*^{(i)}}{\pi_\theta^{(i)}}\mathds{1}_{\{\pi_*^{(i)}>\pi_\theta^{(i)}\}}\\
        &=\frac{1}{\tau}\int_\mathcal{A}\mathrm{d}\pi_\theta^{(i)}\left(Q_*^{(i)}-V_*^{(i)}\right)\mathds{1}_{\{\pi_*^{(i)}>\pi_\theta^{(i)}\}}\\
        &\hspace{3cm}-\int_\mathcal{A}\mathrm{d}\overline{\pi}e^{h_\theta^{(i)}/\tau-\log\EE_{\overline{\pi}}[e^{h_\theta^{(i)}/\tau}]}\left(\frac{1}{\tau}h_\theta^{(i)}-\log\EE_{\overline{\pi}}[e^{h_\theta^{(i)}/\tau}]\right)\mathds{1}_{\{\pi_*^{(i)}>\pi_\theta^{(i)}\}}\\
        &\leq \frac{||Q_*^{(i)}-V_*^{(i)}||_\infty}{\tau} -\int_\mathcal{A}\mathrm{d}\overline{\pi}e^{h_\theta^{(i)}/\tau-\log\EE_{\overline{\pi}}[e^{h_\theta^{(i)}/\tau}]}\left(\frac{1}{\tau}h_\theta^{(i)}-\log\EE_{\overline{\pi}}[e^{h_\theta^{(i)}/\tau}]\right)\mathds{1}_{\{\pi_*^{(i)}>\pi_\theta^{(i)}\}}.
    \end{align*}
    To lower bound the second term, note that the integral is of the form $\int_\mathcal{A}\mathrm{d}\overline{\pi}e^{f_\theta}f_\theta$, and one can easily check that $xe^{x}\geq e^{-1}$ for all $x\in\RR$, so that $\int_\mathcal{A}\mathrm{d}\overline{\pi}e^{f_\theta}f_\theta\geq e^{-1}$.
    We thus have proved that
    \begin{align*}
        \int_\mathcal{A}\mathrm{d}\pi_\theta^{(i)}\log\frac{\pi_*^{(i)}}{\pi_\theta^{(i)}}\mathds{1}_{\{\pi_*^{(i)}>\pi_\theta^{(i)}\}} 
        \leq e^{-1}+\frac{||Q_*^{(i)}-V_*^{(i)}||_\infty}{\tau}.
    \end{align*}
    We now bound $||Q_*^{(i)}-V_*^{(i)}||_\infty$.
    Note that $V_*^{(1)}(s)=\tau\log\EE_{\overline{\pi}}[e^{r(A,s)/\tau}]$, so that $||V_*^{(1)}||_\infty\leq ||r||_\infty$, and then,
    \begin{align*}
        ||Q_*^{(1)}||_\infty,||V_*^{(1)}||_\infty&\leq ||r||_\infty.
    \end{align*}
    We reason by induction. We have
    \begin{align*}
        \left|Q_*^{(i+1)}(a,s)-V_*^{(i+1)}(s)\right|
        &=\left|r(a,s)+\int_\mathcal{S}p(s,a,\mathrm{d}s')V_*^{(i)}(s')-\tau\DKL{\pi_*^{(i+1)}}{\overline{\pi}}(s)-V_*^{(i+1)}(s)\right|\\
        &\leq \left|r(a,s)+\int_\mathcal{S}p(s,a,\mathrm{d}s')V_*^{(i)}(s')+\EE_{\pi_*}[Q_*^{(i+1)}(A,s)]\right|\\
        &\leq ||r||_\infty + ||V_*^{(i)}||_\infty + ||Q_*^{(i+1)}||_\infty\\
        &\leq 2||r||_\infty + 2||V_*^{(i)}||_\infty.
    \end{align*}
    In particular, $||V_*^{(i+1)}||_\infty\leq ||Q_*^{(i+1)}||_\infty + 2||r||_\infty + 2||V_*^{(i)}||_\infty\leq 3||r||_\infty + 3||V_*^{(i)}||_\infty$. By induction, we get for all $i=1,\ldots,n$ that
    \begin{align}\label{eq:bound_Q*_and_V*}
        ||V_*^{(i)}||_\infty&\leq \sum_{j=1}^{i-1}3^{j}r=\frac{3^{i-1}}{2}||r||_\infty\leq \frac{3^{n-1}-1}{2}||r||_\infty,\nonumber\\
        ||Q_*^{(i)}||_\infty&\leq ||r||_\infty + ||V_*^{(i-1)}||_\infty\leq \frac{3^{n-2}+1}{2}||r||_\infty.
    \end{align}
    This proves \eqref{eq:bound_to_prove_lemma_lipschitz}, which in turns proves $(iii)$.

    \vspace{0.5cm}

    $(iv)$
    Let $f_\theta$ satisfy the conditions of the statement.
    For all $i\in\{0,\ldots,n-1\}$, the state distribution satisfies
    \begin{align*}
        \mathbf{m}^{(n-i)}_{\pi_\theta}(\mathrm{d}s)
        &=\int_\mathcal{S}\mathbf{m}^{(n-i+1)}_{\pi_\theta}(\mathrm{d}s')\int_\mathcal{A}\pi_\theta^{(n-i+1)}(\mathrm{d}a|s')p(s',a,\mathrm{d}s)\\
        &=\int_\mathcal{S}\nu(\mathrm{d}s_0)\int_\mathcal{A}\pi_\theta^{(n)}(\mathrm{d}a_0|s_0)\int_\mathcal{S}p(s_0,a_0,\mathrm{d}s_1)\ldots\\
        &\hspace{2cm}\ldots\times\int_\mathcal{A}\pi_\theta^{(n-i+1)}(\mathrm{d}a_{i-1}|s_{i-1})\int_\mathcal{S}p(s_{i-1},a_{i-1},\mathrm{d}s).
    \end{align*}
    Note that $\pi_\theta^{(i)}=\overline{\pi}\frac{e^{h_\theta^{(i)}/\tau}}{\EE_{\overline{\pi}}[e^{h_\theta^{(i)}/\tau}]}$, with $\Vert h_\theta^{(i)}\Vert_\infty\leq \sup_{a,s}\Vert\theta^{(i)}\Vert_2\Vert \psi^{(i)}(a,s)\Vert_2<\Vert\psi\Vert$, which entails that for all $\theta\in\RR^P$, being integrable with respect to $\mathbf{m}_{\pi_\theta}^{(i)}(\mathrm{d}s)\pi_\theta^{(i)}(\mathrm{d}a|s)$ is equivalent to being integrable with respect to $\mathbf{m}_{\overline{\pi}}^{(i)}(\mathrm{d}s)\overline{\pi}^{(i)}(\mathrm{d}a|s)$.
    On the other hand, for any $s_0,a_0,\ldots,s_{i},a_{i}$, we have that
    \begin{align*}
        \nabla_\theta\prod_{\ell=0}^i\pi_\theta^{(n-\ell)}(a_\ell|s_\ell)
        &=\left(\prod_{j=0}^i\pi_\theta^{(n-j)}(a_j|s_j)\right)\sum_{j=0}^{i}\nabla_\theta\log\pi_\theta^{(n-j)}(a_j|s_j)
    \end{align*}
    Since $\Vert\log\pi_\theta^{(n-j)}(a_j|s_j)\Vert_2\leq \frac{2}{\tau}\Vert\psi\Vert$ by $(i)$, from Measure Theory, we know that we can differentiate inside the integral and write
    \begin{align*}
        &\nabla_\theta \int_{\mathcal{S}}\mathbf{m}_{\pi_\theta}^{(\ell)}(\mathrm{d}s)\int_\mathcal{A}\pi_\theta^{(\ell)}(\mathrm{d}a|s)f_\theta(a,s)\\
        &\hspace{0.3cm}=\int\!\cdots\!\int \nu(\mathrm{d}s_0)\left(\prod_{i=0}^\ell \pi_\theta^{(n-i)}(\mathrm{d}a_i|s_i) p(s_i,a_i\mathrm{d}s_{i+1})\right)\left(\sum_{j=0}^{\ell}\nabla_\theta\log\pi_\theta^{(n-j)}(a_j|s_j)\right) f_\theta(a_\ell,s_\ell)\\
        &\hspace{2cm}+\int_{\mathcal{S}}\mathbf{m}_{\pi_\theta}^{(\ell)}(\mathrm{d}s)\int_\mathcal{A}\pi_\theta^{(\ell)}(\mathrm{d}a|s)\nabla_\theta f_\theta(a,s),
    \end{align*}
    which is equivalent to the claim and thus concludes the proof.
\end{proof}

We show below that in expectation, the MPG update \eqref{eq: update cascade learning} is proportional to the gradient of the objective.
It is the only statement where we do not assume perfect gradient update; everywhere else, we assume that $\theta_{t+1}=\theta_t+\eta\nabla_\theta J_n(\pi_t)$.

\begin{lem}\label{lem:Matryoshka_policy_gradient_thm}
    For $\theta_t$ constructed as in \eqref{eq: update cascade learning}, it holds that $\EE[\theta_{t+1}-\theta_t]\propto\nabla_{\theta}J_n(\pi_t)$.
\end{lem}

\begin{proof}[Proof]
    Recall that $\Vert r\Vert_\infty<\infty$ and note that 
    \begin{align*}
        \Big\vert\log\frac{\pi_t^{(i)}}{\overline{\pi}}(a|s)\Big\vert
        &\leq \big\vert h_\theta^{(i)}(a,s)/\tau-\log \EE_{\overline{\pi}}[e^{h_\theta^{(i)}(A,s)/\tau}] \big\vert\\
        &\leq 2\Vert\theta\Vert_2\Vert\psi\Vert,
    \end{align*}
    where we recall that $\Vert\psi\Vert=\sup_{a,s,i}\Vert\psi^{(i)}(a,s)\Vert_2$.
    Hence, Lemma \ref{lem:bounds_derivatives_policy}$(iv)$ applies in the computations below.
        
    Let $\mathbf{m}_\pi^{(i)}$ denote the law of $S_{n-i}$, that is, the $(n-i)$-th visited state under $\pi$. 
    The distribution of the sequence $S_0,A_0,\ldots,A_{n-i-1},S_{n-i}$ is not influenced by the parameters $\theta^{(i)}$, thus we can write
    \begin{align*}
        \nabla_{\theta^{(i)}}J_n(\pi_t)
        &=\int_{\mathcal{S}}\nabla_{\theta^{(i)}}V_{\pi_t}^{(n)}(s)\nu_0(\mathrm{d}s)\\
        &=\nabla_{\theta^{(i)}}\bigg(\EE_{\pi_t}\bigg[\sum_{\ell=0}^{n-i-1}R_\ell-\tau\log\frac{\pi_t^{(n-\ell)}}{\overline{\pi}}(A_{\ell}|S_\ell)\bigg]\\
        &\hspace{2cm}+\EE_{S_{n-i}\sim\mathbf{m}_{\pi_t}^{(i)}}\bigg[\EE_{\pi_t}\bigg[\sum_{\ell=n-i}^{n}R_\ell-\tau\log\frac{\pi_t^{(n-\ell)}}{\overline{\pi}}(A_{\ell}|S_\ell)
        \Big|S_{n-i}\bigg]\bigg]\bigg)\\
        &=0+\EE_{S_{n-i}\sim\mathbf{m}_{\pi_t}^{(i)}}\bigg[\nabla_{\theta^{(i)}}V_{\pi_t}^{(i)}(S_{n-i})\bigg],
    \end{align*}
    where we have used the Markov property.
    We then have that
    \begin{align*}
        \nabla_{\theta^{(i)}}V_{\pi_t}^{(i)}(S_{n-i})
        &=\nabla_{\theta^{(i)}}\EE_{T_{n,i}(\pi_t)}\bigg[\sum_{\ell=n-i}^{n}R_\ell-\tau\log\frac{\pi^{(n-\ell)}_t}{\overline{\pi}}(A_\ell|S_\ell)\bigg|S_{n-i}\bigg]\\
        &=\EE_{\pi_t}\bigg[\bigg(\sum_{\ell=n-i}^{n}\bigg(R_\ell-\tau\log\frac{\pi^{(n-\ell)}_t}{\overline{\pi}}(A_\ell|S_\ell)\bigg)-\tau\bigg)
        \nabla\log\pi_t^{(i)}(A_{n-i}|S_{n-i})\Big|S_{n-i}\bigg]\\
        &=\EE_{\pi_t}\bigg[\sum_{\ell=n-i}^{n}\bigg(R_\ell-\tau\log\frac{\pi^{(n-\ell)}_t}{\overline{\pi}}(A_\ell|S_\ell)\bigg)
        \nabla\log\pi_t^{(i)}(A_{n-i}|S_{n-i})\Big|S_{n-i}\bigg],
    \end{align*}
    where we have used \eqref{eq: softmax gradient cancels expected constant} to get rid of $\tau$.
    
    Recalling the MPG update \eqref{eq: update cascade learning}, we thus have proved that $\EE[\theta_{t+1}-\theta_t]=\eta\nabla_{\theta}J_n(\pi_t)$.
\end{proof}

\begin{proof}[\textbf{Proof of Theorem \ref{thm: policy gradient theorem}}]
    The strategy is to show that $\theta\mapsto \nabla_\theta J_n(\pi_\theta)$ is Lipschitz by bounding the $2$-norm of the Hessian of $J_n$ along the training trajectory.
    It is standard in Optimisation that this implies that for $\eta$ smaller than $2$ over the Lipschitz constant, the objective is monotonically increasing during gradient ascent, to finally deduce the convergence of $J_n(\pi_t)$ as $t\to\infty$.

    Recall that by Lemma \ref{lem:bounds_derivatives_policy}$(i)$, for all $i\in\{0,\ldots,n-1\}$, for all $a\in\mathcal{A}$ and $s\in\mathcal{S}$, we have
    \begin{align*}
        \Vert\nabla_\theta \log\frac{\pi_\theta^{(n-i)}}{\pi_*^{(n-i)}}(a|s)\Vert_2
        &\leq \frac{2}{\tau}\Vert\psi\Vert.
    \end{align*}
    Moreover, for all $(a,s)\in\mathcal{A}\times\mathcal{S}$,
    \begin{align}\label{eq:logpi/pi*_is_bounded}
        \Big|\log\frac{\pi_\theta^{(n-i)}}{\pi_*^{(n-i)}}(a|s)\Big|
        &\leq \big| h_\theta^{({(n-i)})}(a,s)/\tau - \log\EE_{\overline{\pi}}[e^{h_\theta^{(n-i)}(A,s)/\tau}]\big| + \frac{1}{\tau}\Vert Q_*^{(n-i)}-V_*^{(n-i)}\Vert_\infty\nonumber\\
        &\leq \frac{1}{\tau}\left(\Vert\theta\Vert_2 \Vert\psi\Vert + 3^n\Vert r\Vert_\infty \right),
    \end{align}
    where we used \eqref{eq:bound_Q*_and_V*}.
    Hence, we can apply Lemma \ref{lem:bounds_derivatives_policy}$(iv)$ to differentiate $J_n(\pi_\theta)$, thus obtaining
    \begin{align}\label{eq:gradient_objective}
        \nabla_\theta J_n(\pi_\theta)
        &=-\tau\sum_{i=0}^{n-1}\nabla_\theta\EE_{\pi_\theta}\left[\log\frac{\pi_\theta^{(n-i)}}{\pi_*^{(n-i)}}(A_i|S_i)\right]\nonumber\\
        &=-\tau\sum_{i=0}^{n-1}\EE_{\pi_\theta}\bigg[\bigg(\sum_{j=0}^{i}\nabla_\theta\log\pi_\theta^{(n-j)}(A_j|S_j)\bigg)\log\frac{\pi_\theta^{(n-i)}}{\pi_*^{(n-i)}}(A_i|S_i)\bigg]\nonumber\\
        &\hspace{6cm} + \EE_{\pi_\theta}\left[\nabla_\theta \log\frac{\pi_\theta^{(n-i)}}{\pi_*^{(n-i)}}(A_i|S_i)\right]\nonumber\\
        &=-\tau\sum_{i=0}^{n-1}\sum_{j=0}^{i}\EE_{\pi_\theta}\bigg[\nabla_\theta\log\pi_\theta^{(n-j)}(A_j|S_j)\log\frac{\pi_\theta^{(n-i)}}{\pi_*^{(n-i)}}(A_i|S_i)\bigg],
    \end{align}
    where we used that $\EE_{\pi_\theta}[\nabla_\theta \log\pi_\theta^{(n-i)}(A_i|S_i)]=0$.
    Exchanging the order of summation and focusing on the components of the gradient $\nabla_{\theta^{(n-j)}}$ for $j\in\{0,\ldots,n-1\}$ fixed, we get
    \begin{align}\label{eq:grad_J_wrt_n-j}
        \nabla_{\theta^{(n-j)}} J_n(\pi_\theta)
        &=-\tau\sum_{i=j}^{n-1}\EE_{\pi_\theta}\bigg[\nabla_{\theta^{(n-j)}}\log\pi_\theta^{(n-j)}(A_j|S_j)\log\frac{\pi_\theta^{(n-i)}}{\pi_*^{(n-i)}}(A_i|S_i)\bigg].
    \end{align}

    To compute the Hessian of $J_n(\pi_\theta)$, we need to differentiate once more the expectation.
    Fix $j\geq j'\in\{0,\ldots,n-1\}$, we compute the components of the Hessian of the form $\nabla_{\theta^{(n-j')}}\nabla_{\theta^{(n-j)}}J_n(\pi_\theta)$.
    For $j'<j$, the terms inside the expectation do not depend on $\theta^{(n-j')}$, so that Lemma \ref{lem:bounds_derivatives_policy}$(iv)$ trivially applies and yields
    \begin{align*}
        &\nabla_{\theta^{(n-j')},\theta^{(n-j)}}^{2}J_n(\pi_\theta)\\
        &\hspace{0.7cm}= -\tau\sum_{i=j}^{n-1}\EE_{\pi_\theta}\bigg[\nabla_{\theta^{(n-j')}}\log\pi_\theta^{(n-j')}(A_{j'}|S_{j'})(\nabla_{\theta^{(n-j)}}\log\pi_\theta^{(n-j)}(A_{j}|S_{j}))^T\log\frac{\pi_\theta^{(n-i)}}{\pi_*^{(n-i)}}(A_i|S_i)\bigg].
    \end{align*}
    Lemma \ref{lem:bounds_derivatives_policy} shows that
    \begin{align*}
        \Vert \nabla_{\theta^{(n-j')},\theta^{(n-j)}}^{2}J_n(\pi_\theta)\Vert_2
        &\leq \sum_{i=j}^{n-1}\frac{4\Vert \psi\Vert^2}{\tau}\EE_{\pi_{\theta}}\bigg[\Big|\log\frac{\pi_\theta^{(n-i)}}{\pi_*^{(n-i)}}(A_i|S_i)\Big|\bigg]\\
        &\leq n\frac{4\Vert \psi\Vert^2}{\tau^2}(J_n(\pi_*)-J_n(\pi_\theta) + 3^n\Vert r\Vert_\infty),
    \end{align*}
    where we used Lemma \ref{lem: general case value function is optimal value function minus DKL} to bound the expectation of the Kullback-Leibler divergences by the performance gap.
    For $j'=j$, to apply Lemma \ref{lem:bounds_derivatives_policy}$(iv)$, we check that the gradient of the terms in the expectation of \eqref{eq:grad_J_wrt_n-j} is bounded, that is,
    \begin{align*}
        (a,s)\mapsto\ &\nabla_{\theta^{(n-j)}}^2\log\pi_\theta^{(n-j)}(a|s)\EE_{\pi_\theta}\bigg[\log\frac{\pi_\theta^{(n-i)}}{\pi_*^{(n-i)}}(A_i|S_i)\Big|A_j=a,S_j=s\bigg]\\
        &\hspace{1cm}+\nabla_{\theta^{(n-j)}}\log\pi_\theta^{(n-j)}(a|s)(\nabla_{\theta^{(n-j)}}\log\pi_\theta^{(n-j)}(a|s))^T.
    \end{align*}
    By \eqref{eq:logpi/pi*_is_bounded} and Lemma \ref{lem:bounds_derivatives_policy}$(i)$ and $(ii)$, each coordinate of the above matrix is bounded by $C\Vert\theta\Vert_2$ for some constant $C>0$.
    Hence Lemma \ref{lem:bounds_derivatives_policy}$(iv)$ applies and the Hessian for $j=j'$ has the additional terms
    \begin{align*}
        &\sum_{i=j}^{n-1}\EE_{\pi_\theta}\bigg[\nabla_{\theta^{(n-j)}}^2\log\pi_\theta^{(n-j)}(A_{j}|S_{j})\log\frac{\pi_\theta^{(n-i)}}{\pi_*^{(n-i)}}(A_i|S_i)\bigg]\\
        &\hspace{3cm}+\EE_{\pi_\theta}\bigg[\nabla_{\theta^{(n-j)}}\log\pi_\theta^{(n-j)}(A_{j}|S_{j})(\nabla_{\theta^{(n-j)}}\log\pi_\theta^{(n-j)}(A_j|S_j))^T\bigg]
    \end{align*}
    We get by Lemma \ref{lem:bounds_derivatives_policy}
    \begin{align*}
        \Vert \nabla_{\theta^{(n-j)},\theta^{(n-j)}}^{2}J_n(\pi_\theta)\Vert_2
        &\leq 4n\frac{\Vert \psi\Vert^2}{\tau^2}(J_n(\pi_*)-J_n(\pi_\theta) + 3^n\Vert r\Vert_\infty)\\
        &\hspace{0.3cm}+ \sum_{i=j}^{n-1} \bigg(\EE\left[\Vert\nabla_{\theta^{(n-j)}}^2\log\pi_\theta^{(n-j)}(A_{j}|S_{j})\Vert_2\Big\vert \log\frac{\pi_\theta^{(n-i)}}{\pi_*^{(n-i)}}(A_i|S_i)\Big\vert\right]\\
        &\hspace{0.3cm}+ \EE_{\pi_\theta}\bigg[\Big\Vert\nabla_{\theta^{(n-j)}}\log\pi_\theta^{(n-j)}(A_{j}|S_{j})(\nabla_{\theta^{(n-j)}}\log\pi_\theta^{(n-j)}(A_j|S_j))^T\Big\Vert_2\bigg]\bigg)\\
        &\leq 4n\frac{\Vert \psi\Vert^2}{\tau^2}(J_n(\pi_*)-J_n(\pi_\theta) + 3^n\Vert r\Vert_\infty)\\
        &\hspace{0.3cm}+ 2n \frac{P}{\tau^3}\Vert\psi\Vert^2 (J_n(\pi_*)-J_n(\pi_\theta) + 3^n\Vert r\Vert_\infty)
        + 4n\frac{\Vert\psi\Vert^2}{\tau^2}\\
        &=2n\frac{\Vert\psi\Vert^2}{\tau^2} \Big(2+\frac{P}{\tau}\Big)(J_n(\pi_*)-J_n(\pi_\theta) + 3^n\Vert r\Vert_\infty) + 4n\frac{\Vert\psi\Vert^2}{\tau^2}.
    \end{align*}

    Finally, to obtain a bound on $\Vert \nabla_{\theta}^{2}J_n(\pi_\theta)\Vert_2$, we only need to sum over $j,j'\in\{0,\ldots,n-1\}$. This yields
    \begin{align*}
        \nabla_{\theta}^{2}J_n(\pi_\theta)\Vert_2
        &\leq 2n^2\frac{\Vert\psi\Vert^2}{\tau^2} \Big(2+\frac{P}{\tau}\Big)(J_n(\pi_*)-J_n(\pi_\theta) + 3^n\Vert r\Vert_\infty) + 4n^2\frac{\Vert\psi\Vert^2}{\tau^2}\\
        &\hspace{2cm} + (n-1)4n^2 \frac{\Vert \psi\Vert^2}{\tau^2}(J_n(\pi_*)-J_n(\pi_\theta) + 3^n\Vert r\Vert_\infty)\\
        &\leq 4(n^2+n^3)\Big(2+\frac{P}{\tau}\Big)\frac{\Vert\psi\Vert^2}{\tau^2}(J_n(\pi_*)-J_n(\pi_\theta) + 3^n\Vert r\Vert_\infty) + 4n^2\frac{\Vert\psi\Vert^2}{\tau^2}\\
        &=L(\theta).
    \end{align*}

    We thus have shown that $\nabla_\theta J_n(\pi_\theta)$ is locally Lipschitz with constant $L(\theta)$.
    Since $\theta\mapsto L(\theta)$ is monotonically decreasing as $J_n(\pi_*)-J_n(\pi_\theta)$ decreases, if $\eta<2/L(\theta_0)$ at the start of training, as explained at the beginning of the proof, it implies by induction that $\eta<2/L(\theta_t)$ for all $t\geq 0$, which entails the claim and concludes the proof.
\end{proof}

\subsection{On the optimal policy}

\begin{proof}[\textbf{Proof of Lemma \ref{lem: general case V_* = log expectation}}]
    By definition, we write
    \begin{align*}
        V_*^{(n)}(s)
        &=\tau\int_{\mathcal{A}}\pi_*^{(n)}(\mathrm{d}a|s)\left(Q_*^{(n)}(a,s)-\tau\log\frac{\pi_*^{(n)}}{\overline{\pi}}(a|s)\right)\\
        &=\tau\log\EE_{\overline{\pi}}\left[\exp(Q_*^{(n)}(A,s)/\tau)\right]
        \int_{\mathcal{A}}\overline{\pi}(\mathrm{d}a|s)\frac{\exp\left(Q_*^{(n)}(a,s)/\tau\right)}{\EE_{\overline{\pi}}\left[\exp\left(Q_*^{(n)}(A,s)/\tau\right)\right]}\\
        &=\tau\log\EE_{\overline{\pi}}\left[\exp\left(Q_*^{(n)}(A,s)/\tau\right)\right],
    \end{align*}
    as claimed, which concludes the proof.
\end{proof}

\begin{proof}[\textbf{Proof of Proposition \ref{prop: extending horizon converges to standard optimal policy}}]
    Assume that $\nu$ has full support. 
  
    (i) Let $\pi\in\mathcal{P}$ be any standard policy, and let $\pi_n=(\pi,\ldots,\pi)\in\mathcal{P}_n$.
    By definition of the standard infinite-horizon discounted objective $J_\infty$, using the dominated convergence theorem (rewards are bounded), we have that $J_n(\pi_n)\to J_\infty(\pi)$.
    In particular, we get that $\pi_{*,n}^{(n)}$ achieves a performance arbitrarily close to that of $\pi_{*,\infty}$ in the infinite horizon discounted setting, and since the optimal policy of $J_\infty$ is unique ($\nu$-almost everywhere), we deduce that $\pi_{*,n}^{(n)}\to\pi_{*,\infty}$ as $n\to\infty$.

    (ii)
    Suppose that $J_1(\pi_{*,1})>J_1(T_{n,1}(\pi_{*,n}))$, that is
    \begin{align*}
        \int_{\mathcal{S}}V_{\pi_{*,1}}^{(1)}(s)\nu(\mathrm{d}s)
        >\int_{\mathcal{S}}V_{\pi_{*,n}}^{(1)}(s)\nu(\mathrm{d}s).
    \end{align*}
    In particular, the set $\widetilde{\mathcal{S}}:=\{s\in\mathcal{S}:\ V_{\pi_{*,1}}^{(1)}(s)>V_{\pi_{*,n}}^{(1)}(s)\}$ is non-empty and $\nu(\widetilde{\mathcal{S}})$.
    Furthermore, by optimality, $s\in\mathcal{S}\setminus\widetilde{\mathcal{S}}$ if and only if $V_{\pi_{*,1}}^{(1)}(s)=V_{\pi_{*,n}}^{(1)}(s)$.
    Let $\widetilde{\pi}_{*,n}\in\mathcal{P}_n$ be identical to $\pi_{*,n}$ except for the 1-step policy where $\pi_{*,n}^{(1)}$ is replaced by $\pi_{*,1}$.
    Then, the recursive structure of the value function \eqref{eq: gen case recursive definition value function} entails that $J_n(\widetilde{\pi}_{*,n})>J_n(\pi_{*,n})$, which is a contradiction.
    Therefore, $T_{n,1}(\pi_{*,n})=\pi_{*,1}$.
    
    Then, by induction and using the recursive structure of the value function, the same argument shows that $T_{n,m}(\pi_{*,n})=\pi_{*,m}$ for all $m=2,\ldots,n-1$, which concludes the proof.
\end{proof}

\subsection{On the convergence of training}
\label{app:on_the_cvg_of_training}

By Theorem \ref{thm: policy gradient theorem}, we know that $J_n(\pi_t)$ converges monotonically as $t\to\infty$.
However, this does not ensure that $\pi_t$ converges, and a fortiori that $\theta_t$ converges.
In this section, we prove two results of importance in establishing the global convergences of Theorem \ref{thm: general case global optimality} and Theorem \ref{Thm: global cvg outside of the RKHS}

Below, we show in Lemma \ref{lem:converging_subsequence_of_policies} that the sequence of policies visited during training is \textit{relatively compact}, that is, any of its subsequences admits a \textit{weakly converging} subsequence.
A sequence of measure $(\mu_k)_{k\geq 0}$ is said to converge weakly if and only if $\lim_{k\to\infty}\int f\mathrm{d}\mu_k=\int f\mathrm{d}\mu$ for every continuous bounded map $f$.

Then, we show in Lemma \ref{lem:parameters_remain_bounded} that the parameters of a converging subsequence $(\pi_{t_k})_{k\geq 0}$ remain uniformly bounded, which implies that any limit of the parameters $\theta_{t_k}$ belongs to $\RR^P$.
In particular, the subsequences of policies converging weakly during training actually have their parameters converging inside $\RR^P$, so that the convergence is in the stronger sense of \eqref{eq: def convergence of policies}.

\begin{lem}\label{lem:converging_subsequence_of_policies}
    Under the assumptions of Theorem \ref{thm: policy gradient theorem}, for all $i=1,\ldots,n$, the sequence of probability measures $(\mathbf{m}_{\pi_t}^{(i)}(\mathrm{d}s)\pi_t^{(i)}(\mathrm{d}a))_{t\geq 0}$ on $\mathcal{S}\times\mathcal{A}$ is relatively compact.
\end{lem}

\begin{proof}[Proof]
    By Prohorov's theorem (Theorem 5.1 in \citet{billingsley13}), it suffices to show that $(\mathbf{m}_{\pi_t}^{(i)}(\mathrm{d}s)\pi_t^{(i)}(\mathrm{d}a))_{t\geq 0}$ is \textit{tight} for all $i=1,\ldots,n$.
    We say that a sequence of probability measures $\mu_t$ is tight if and only if for all $\epsilon>0$, there exists a compact set $K_\epsilon$ such that $\mu_t(K_\epsilon)>1-\epsilon$. Roughly speaking, this ensures that no mass escapes at infinity.

    Starting with $i=n$, we first show that for every $\epsilon>0$, there exists a compact set $K_{\epsilon}\subset \mathcal{S}\times\mathcal{A}$ such that $\limsup_{t\to\infty}\int_{K_\epsilon}\mathrm{d}\nu\mathrm{d}\pi_t^{(n)}>1-\epsilon$.
    By contradiction, suppose that this is not the case, then there exists $\epsilon>0$ such that $\limsup_{t\to\infty}\int_{K^c}\mathrm{d}\nu\mathrm{d}\pi_t^{(n)}\geq \epsilon$ for all compact $K\subset\mathcal{S}\times\mathcal{A}$, where $K^{c}$ is the complement of $K$.
    Let $\delta>0$ be arbitrarily smaller that $\epsilon$ and consider a compact $K_{\delta}\subset\mathcal{S}\times\mathcal{A}$ such that $\int_{K_\delta^c}\mathrm{d}\nu\mathrm{d}\overline{\pi}<\delta$, then necessarily, we have
    \begin{align*}
        \limsup_{t\to\infty}\int_{K_\epsilon^c}\nu(\mathrm{d}s)\pi_t^{(n)}(\mathrm{d}a|s)\log\frac{\pi_t^{(n)}}{\overline{\pi}}(a|s)
        &=-\limsup_{t\to\infty}\int_{K_\epsilon^c}\nu(\mathrm{d}s)\pi_t^{(n)}(\mathrm{d}a|s)\log\frac{\overline{\pi}}{\pi_t^{(n)}}(a|s)\\
        &\geq -\limsup_{t\to\infty}\log\frac{\int_{K_\epsilon^c}\nu(\mathrm{d}s)\overline{\pi}(\mathrm{d}a|s)}{\int_{K_\epsilon^c}\nu(\mathrm{d}s)\pi_t^{(n)}(\mathrm{d}a|s)}\\
        &\geq -\log \frac{\delta}{\epsilon},
    \end{align*}
    where we used Jensen's inequality by concavity of the logarithm.
    Since $\delta>0$ is arbitrary, this shows that $\limsup_{t\to\infty}\int_{\mathcal{S}}\DKL{\pi_t^{(n)}}{\overline{\pi}}(s)\nu(\mathrm{d}s)=\infty$, which contradicts the fact that $J_n(\pi_t)=\sum_{i=0}^n\EE_{\pi_t}[R_i-\DKL{\pi_t^{(n-i)}}{\overline{\pi}}(S_i)]$ converges to a finite value.
    
    Hence $(\nu(\mathrm{d}s)\pi_{t}^{(n)}(\mathrm{d}a))_{t\geq 0}$ is tight and then relatively compact.
    
    To end the proof, we reason by induction as follows: let $1<i\leq n$ and consider a subsequence $(\pi_{t_k})_{k\geq 0}$ such that
    $\mathbf{m}_{\pi_t}^{(j)}(\mathrm{d}s)\pi_{t_k}^{(j)}(\mathrm{d}a)$ converges weakly toward $\mathbf{m}_{\pi_\infty}^{(j)}(\mathrm{d}s)\pi_\infty^{(j)}(\mathrm{d}a)$ for all $j=i,\ldots,n$, for some policy $\pi_\infty$.
    Note that $\mathbf{m}_{\pi_{t_k}}^{(i-1)}$ only depends on $\pi_{t_k}^{(j)}$ for $j\in\{i,\ldots,n\}$, so that for any Borel subset $B\subset\mathcal{S}$,
    \begin{align*}
        \mathbf{m}_{\pi_{t_k}}^{(i-1)}(B)
        &=\int_{\mathcal{S}}\mathbf{m}_{\pi_{t_k}}^{(i)}(\mathrm{d}s')\int_{\mathcal{A}}\pi_{t_k}^{(i)}(\mathrm{d}a)p(s',a,B)\\
        &\underset{k\to\infty}{\longrightarrow}
        \int_{\mathcal{S}}\mathbf{m}_{\pi_{\infty}}^{(i)}(\mathrm{d}s')\int_{\mathcal{A}}\pi_{\infty}^{(i)}(\mathrm{d}a)p(s',a,B),
    \end{align*}
    where the convergence is in the weak sense and where we used the fact that $(s,a)\mapsto p(s,a,B)$ is continuous (and obviously bounded).
    Then, the same reasoning by contradiction as for the case $i=n$ applies, by letting $K_{\delta}\subset\mathcal{S}\times\mathcal{A}$ be a compact subset such that $\int_{K_\delta^c}\mathrm{d}\mathbf{m}_{\pi_{t_k}}^{(i-1)}\mathrm{d}\overline{\pi}<\delta$.
    This concludes the proof.
\end{proof}

\begin{lem}\label{lem:parameters_remain_bounded}
    Under the assumptions of Theorem \ref{thm: policy gradient theorem}, it holds that $\sup_{t\geq 0}||\theta_t||<\infty$.
\end{lem}

\begin{proof}[Proof]
    Let $\theta_{t_k}$ be a subsequence of $\theta_t$ such that $\pi_{t_k}$ converges weakly to some $\pi_\infty$, which exists thanks to Lemma \ref{lem:converging_subsequence_of_policies}.
    Assume that there exists $i\in\{1,\ldots,n\}$ such that $||\theta_{t_k}^{(i)}||\to\infty$ as $k\to\infty$, and let $i$ be the smallest such integers.
    Let $\underline{\theta}_{t_k}^{(i)}:=\frac{\theta_{t_k}^{(i)}}{||\theta_{t_k}^{(i)}||_2}$ and $\underline{\theta}_\infty^{(i)}:=\lim_{k\to\infty}\underline{\theta}_{t_k}^{(i)}$ (to ensure convergence, one can always take subsequences since $\underline{\theta}_{t_k}$ lives in a compact sphere.)
    Without loss of generality, we choose the subsequence such that
    \begin{align}\label{eq:inproof_to_contradict_bounded_param}
        \underline{\theta}_{\infty}^{(i)}\cdot\nabla_{\theta^{(i)}} J_n(\pi_{t_k})>0.
    \end{align}
    Indeed, since $||\theta_{t_k}^{(i)}||_2\to\infty$ and $\underline{\theta}_{t_k}\to\underline{\theta}_\infty$, necessarily, $\nabla_{\theta^{(i)}} J_n(\pi_{t_k})$ must point inside the half-plane $\{v\in\RR^{P_i}:v\cdot\underline{\theta}_\infty^{(i)}>0\}$ infinitely many times.
    We now show that this leads to a contradiction.
    
    Firstly, there exists a constant $C$ independent of $t_k$ such that for all $j<i$, it holds that $||\theta_{t_k}^{(j)}||_2\leq C$.
    We thus have for all $s\in\mathcal{S}$ that
    \begin{align*}
        \DKL{\pi_{t_k}^{(j)}}{\overline{\pi}}(s)
        &=\int_{\mathcal{A}}\pi_{t_k}^{(j}(\mathrm{d}a|s)\left(h_{t_k}^{(j)}(a,s)/\tau-\log\EE_{\overline{\pi}}[e^{h_{t_k}^{(j)}(A,s)/\tau}]\right)\\
        &\leq \frac{2}{\tau}||\theta_{t_k}^{(j)}||_2||\psi^{(j)}||_\infty\\
        &\leq \frac{2}{\tau}C||\psi^{(j)}||_\infty
    \end{align*}
    In particular, by definition of $Q$-functions, this yields
    \begin{align}\label{eq:inproof_uniform_bound_of_Q}
        ||Q_{\pi_{t_k}}^{(i)}||_\infty\leq i\left(||r||_\infty + 2C ||\psi||_\infty\right).
    \end{align}
    On the other hand, by Lemma \ref{lem: general case value function is optimal value function minus DKL}, it holds that
    \begin{align*}
        Q_{\pi_t}^{(i)}(a,s)
        &=r(a,s) + \int_\mathcal{S}p(s,a,\mathrm{d}s')V_{\pi_t}^{(i-1)}(s')\\
        &=\tau\log\frac{\pi_*^{(i)}}{\overline{\pi}}(a|s) - \int_\mathcal{S}p(s,a,\mathrm{d}s')(V_{\pi_t}^{(i-1)}(s')-V_*^{(i)}(s'))\\
        &=\tau\log\frac{\pi_*^{(i)}}{\overline{\pi}}(a|s) - \tau\sum_{k=i+1}^{n-1}\EE_{\pi_t}\left[\DKL{\pi_t^{(n-k)}}{\pi_*^{(n-k)}}(S_k)|S_i=s,A_i=a\right].
    \end{align*}
    Note that by compactness, we can take a subsequence such that $\underline{\theta}_{t_k}^{(j)}\to\underline{\theta}_{\infty}^{(j)}$ simultaneously for all $j\leq i$.
    We still denote by $\underline{\theta}_{t_k}^{(j)}$ such subsequences.
    A computation similar to Equation \eqref{eq:gradient_objective} gives
    \begin{align*}
        -\underline{\theta}_{t_k}^{(i)}\cdot \nabla_{\theta^{(i)}} J_n(\pi_{t_k})
        &= \int_{\mathcal{S}}\mathbf{m}_{\pi_{t_k}}^{(i)}(\mathrm{d}s)\int_{\mathcal{A}}\pi_{t_k}^{(i)}(\mathrm{d}a|s)\log\frac{\pi_{\underline{\theta}_{t_k}}^{(i)}}{\overline{\pi}}(a|s) \\
        &\hspace{1cm}\times\bigg(\tau\log\frac{\pi_{t_k}^{(i)}}{\overline{\pi}}(a|s)-\tau\DKL{\pi_{t_k}^{(i)}}{\overline{\pi}}(s)-Q_{\pi_{t_k}}^{(i)}(a,s)+\EE_{\pi_{t_k}}\left[Q_{\pi_{t_k}^{(i)}}(A|s)\right]\bigg).
    \end{align*}
    Using that $h_{\theta_{t_k}^{(i)}}^{(i)}=||\theta_t^{(i)}||_2h_{\underline{\theta}_{t_k}^{(i)}}^{(i)}$ by definition, we have that
    \begin{align*}
        \tau\bigg(\log\frac{\pi_{t_k}^{(i)}}{\overline{\pi}}(a|s)-\DKL{\pi_{t_k}^{(i)}}{\overline{\pi}}(s)\bigg)
        &=h_{\theta_{t_k}}^{(i)}(a,s)-\EE_{\pi_{t_k}}\left[h_{\theta_{t_k}}^{(i)}(A,s)\right]\\
        &=||\theta_{t_k}^{(i)}||_2\bigg(h_{\underline{\theta}_{t_k}}^{(i)}(a,s)-\EE_{\pi_{t_k}}\left[h_{\underline{\theta}_{t_k}}^{(i)}(A,s)\right]\bigg)\\
        &=\tau||\theta_{t_k}^{(i)}||_2\bigg(\log\frac{\pi_{\underline{\theta}_{t_k}}^{(i)}}{\overline{\pi}}(a|s)-\EE_{\pi_{t_k}}\big[\log\frac{\pi_{\underline{\theta}_{t_k}}^{(i)}}{\overline{\pi}}(A|s)\big]\bigg)
    \end{align*}
    Hence, we obtain
    \begin{align*}
        -\underline{\theta}_{t_k}^{(i)}\cdot \nabla_{\theta^{(i)}} J_n(\pi_{t_k})
        &= ||\theta_{t_k}^{(i)}||_2\int_{\mathcal{S}}\mathbf{m}_{\pi_{t_k}}^{(i)}(\mathrm{d}s)\int_{\mathcal{A}}\pi_{t_k}^{(i)}(\mathrm{d}a|s)\log\frac{\pi_{\underline{\theta}_{t_k}}^{(i)}}{\overline{\pi}}(a|s) \\
        &\hspace{0.3cm}\times\bigg(\tau\log\frac{\pi_{\underline{\theta}_{t_k}}^{(i)}}{\overline{\pi}}(a|s)-\tau\EE_{\pi_{t_k}}\Big[\log\frac{\pi_{\underline{\theta}_{t_k}}^{(i)}}{\overline{\pi}}(A|s)\Big]-\frac{Q_{\pi_{t_k}}^{(i)}(a,s)-\EE_{\pi_{t_k}}\Big[Q_{\pi_{t_k}}^{(i)}(A|s)\Big]}{||\theta_{t_k}||_2}\bigg).
    \end{align*}
    By Theorem \ref{thm: policy gradient theorem}, the right-hand side above converges to $0$ as $k\to\infty$.
    In particular, since $||\theta_{t_k}^{(i)}||_2\to\infty$, the weak convergence of $\mathbf{m}_{\pi_{t_k}}^{(i)}(\mathrm{d}s)\times\pi_{t_k}^{(i)}(\mathrm{d}a|s)$ and the uniform boundedness of $\log\frac{\pi_{\underline{\theta}_{t_k}}^{(i)}}{\overline{\pi}}$ and $Q_{\pi_{t_k}}^{(i)}$ entail that
    \begin{align*}
        0
        &=\int_{\mathcal{S}}\mathbf{m}_{\pi_\infty}^{(i)}(\mathrm{d}s)\int_{\mathcal{A}}\pi_\infty^{(i)}(\mathrm{d}a|s)\bigg(\log\frac{\pi_{\underline{\theta}_\infty^{(i)}}}{\overline{\pi}}(a|s)-\EE_{\pi_{\infty}}\Big[\log\frac{\pi_{\underline{\theta}_\infty}^{(i)}}{\overline{\pi}}(A|s)\Big]\bigg)^2.
    \end{align*}
    This shows that for $\mathbf{m}_{\pi_\infty}^{(i)}$-almost every $s$, the map $a\mapsto \log\frac{\pi_{\underline{\theta}_\infty}^{(i)}}{\overline{\pi}}(a|s)$ is constant on the support of $\pi_\infty^{(i)}$.
    If $\pi_\infty^{(i)}(\cdot|s)$ has full support for $\mathbf{m}_{\pi_{t_k}}^{(i)}$-almost every $s$, then $a\mapsto h_{\underline{\theta}_\infty}(a,s)$ is constant for such $s$, which contradicts the fact that $||\underline{\theta}_\infty||_2=1\neq 0$.

    Suppose instead that $\pi_{\infty}^{(i)}(\cdot|s)$ does not have full support for a subset of $\mathcal{S}$ with positive $\mathbf{m}_{\pi_{\infty}}^{(i)}$-measure.
    For all $s\in\mathcal{S}$, let $E_s:=\mathrm{Supp}(\pi_\infty(\cdot|s))$ and denote by $E_S^c$ its complementary set.
    One can show by contradiction that since $h_{t_k}^{(i)}=||\theta_{t_k}^{(i)}||_2h_{\underline{\theta}_{t_k}}^{(i)}$ and $h_{\underline{\theta}_{t_k}}^{(i)}\to h_{\underline{\theta}_{\infty}}^{(i)}$ pointwise as $k\to\infty$, it holds that
    $h_{\underline{\theta}_{\infty}}^{(i)}(a,s)=C_s:=\sup_{a'\in\mathcal{A}}h_{\underline{\theta}_{\infty}}^{(i)}(a',s)$, for all $a\in E_s$, and similarly, $h_{\underline{\theta}_{\infty}}^{(i)}(a,s)\leq C_s$ for all $a\in E_s^c$.
    We therefore see that for all $s\in\mathcal{S}$, it holds that
    \begin{align}\label{eq:inproof_integral_log}
        &\int_{\mathcal{A}}\pi_{t_k}^{(i)}(\mathrm{d}a|s)\left(\log\frac{\pi_{t_k}^{(i)}}{\overline{\pi}}(a|s)-\DKL{\pi_{t_k}^{(i)}}{\overline{\pi}}(s)-Q_{\pi_{t_k}}^{(i)}(a,s)+\EE_{\pi_{t_k}}\left[Q_{\pi_{t_k}}^{(i)}(A,s)\right]\right)\log\frac{\pi_{\underline{\theta}_\infty}^{(i)}}{\overline{\pi}}(s)\nonumber\\
        &\hspace{0.5cm}=\int_{E_s^c}\pi_{t_k}^{(i)}(\mathrm{d}a|s)\left(\log\frac{\pi_{t_k}^{(i)}}{\overline{\pi}}(a|s)-\DKL{\pi_{t_k}^{(i)}}{\overline{\pi}}(s)-Q_{\pi_{t_k}}^{(i)}(a,s)+\EE_{\pi_{t_k}}\left[Q_{\pi_{t_k}}^{(i)}(A,s)\right]\right)\nonumber\\
        &\hspace{9cm}\times\left(\log\frac{\pi_{\underline{\theta}_\infty}^{(i)}}{\overline{\pi}}(a|s)-\log\frac{C_s}{\overline{\pi}}\right),
    \end{align}
    where we used that the integral over $\mathcal{A}$ of the terms inside the first parentheses is null.
    As explained above, the second factor in the integral is non-positive, and strictly negative for some $a$ since $\pi_{\underline{\theta}}^{(i)}\neq \overline{\pi}$.
    We claim that the first term is negative, which is seen as follows: we note that for all $s$, for all $a\in E_s$ and $a'\in E_s^c$, we have for all $k$ large enough that $\log\frac{\pi_{t_k}^{(i)}}{\overline{\pi}}(a|s)+Q_{\pi_{t_k}}^{(i)}(a,s)>\log\frac{\pi_{t_k}^{(i)}}{\overline{\pi}}(a'|s)+Q_{\pi_{t_k}}^{(i)}(a',s)$, since $Q_{\pi_{t_k}}^{(i)}$ is uniformly bounded and $\log\frac{\pi_{t_k}^{(i)}}{\overline{\pi}}(a'|s)\to -\infty$ for all $a'\in E_s^c$.
    Moreover, $\pi_{t_k}^{(i)}(E_s^c|s)\to 0$ as $k\to\infty$.
    Hence, for all $a'\in E_s^c$, we have $\log\frac{\pi_{t_k}^{(i)}}{\overline{\pi}}(a'|s)+Q_{\pi_{t_k}}^{(i)}(a',s)>\EE_{\pi_{t_k}}[\log\frac{\pi_{t_k}^{(i)}}{\overline{\pi}}(A|s)+Q_{\pi_{t_k}}^{(i)}(A,s)]$ for all $k$ large enough.
    This shows that for all $s\in\mathcal{S}$, for all $k$ large enough, the left-hand side of \eqref{eq:inproof_integral_log} is positive.

    Hence, we now see that
    \begin{align*}
        -\underline{\theta}_{\infty}^{(i)}\cdot\nabla_{\theta^{(i)}} J_n(\pi_{t_k})
        &= \int_{\mathcal{S}}\mathbf{m}_{\pi_{t_k}}^{(i)}(\mathrm{d}s)\int_{\mathcal{A}}\pi_{t_k}^{(i)}(\mathrm{d}a|s)\log\frac{\pi_{\underline{\theta}_{\infty}}^{(i)}}{\overline{\pi}}(a|s) \\
        &\hspace{1cm}\times\bigg(\tau\log\frac{\pi_{t_k}^{(i)}}{\overline{\pi}}(a|s)-\tau\DKL{\pi_{t_k}^{(i)}}{\overline{\pi}}(s)-Q_{\pi_{t_k}}^{(i)}(a,s)+\EE_{\pi_{t_k}}\left[Q_{\pi_{t_k}^{(i)}}(A|s)\right]\bigg)
    \end{align*}
    is positive for all $k$ large enough.
    This contradicts \eqref{eq:inproof_to_contradict_bounded_param} and concludes the proof.
\end{proof}

\subsection{Global optimality of MPG: realizable case}

\begin{proof}[\textbf{Proof of Proposition \ref{prop: optimal policy}}]
    The Kullback-Leibler divergence being non-negative, it is readily seen that for all $s\in\mathcal{S}$, the maximal value of $\pi\mapsto V_{\pi}^{(n)}(s)$ is obtained for $\pi=\pi_*$.
    It is then immediate that $\pi_*$ is the unique uniformly optimal policy for the objective $J_n$ given in \eqref{def: general case objective function}.
\end{proof}

Recall that $\mathbf{m}_\pi^{(i)}$ denotes the law of $S_{n-i}$ when following policiy $\pi$ from initial state $S_0\sim\nu$.
\begin{lem}\label{lem: general case expected update log pi_t}
    Let $t\in\NN$ and $m\in\{1,\ldots,n\}$.
    Suppose that $\pi_t^{(k)}(\cdot|s)=\pi_*^{(k)}(\cdot|s)$ for $\mathbf{m}_{\pi}^{(k)}$-almost every $s\in\mathcal{S}$, for all $k=1,\ldots,m-1$.
    For all $a\in\mathcal{A}$ and $s\in\mathcal{S}$, it holds that
    \begin{align*}
        &\log\pi^{(m)}_{t+1}(a|s)-\log\pi^{(m)}_t(a|s)\\
        &\hspace{1cm}= - \learningrate \tau\int_{\mathcal{S}}\mathbf{m}_{\pi_t}^{(m)}(\mathrm{d}s')\int_{\mathcal{A}}\pi^{(m)}_t(\mathrm{d}a'|s')
        \left(\log\frac{\pi^{(m)}_t}{\pi^{(m)}_*}(a'|s')-\DKL{\pi^{(m)}_t}{\pi^{(m)}_*}(s')\right)\\
        &\hspace{3cm}\times\left(\Theta^{(m)}((a,s),(a',s'))-\EE_{\pi^{(m)}_t}[\Theta^{(m)}((A,s),(a',s'))]\right)
        +o\left(\learningrate C(\theta_t)\right),
    \end{align*}
    where the constant $C(\theta_t)$ does not depend on $\learningrate$.
\end{lem}

\begin{proof}[Proof]
    The gradient of the policy reads as
    \begin{align}\label{eq: general case gradient policy pi}
        \nabla_\theta\pi^{(m)}_t(a|s)
        &=\frac{1}{\tau}\pi^{(m)}_t(a|s)\int_{\mathcal{A}}\left(\delta_{a,\mathrm{d}a'}-\pi^{(m)}_t(\mathrm{d}a'|s)\right)
        \nabla_{\theta}h^{(m)}_t(a,s).
    \end{align}
    Let $(a,s)\in\mathcal{A}\times\mathcal{S}$.
    Using \eqref{eq: update cascade learning} and a first order Taylor approximation, we write
    \begin{align}\label{eq: gen case log increment to compute}
        \log\pi_{t+1}^{(m)}(a|s)-\log\pi_t^{(m)}(a|s)
        &= (\theta_{t+1}^{(m)}-\theta_t^{(m)})\cdot\frac{\nabla_\theta \pi_t^{(m)}(a|s)}{\pi_t^{(m)}(a|s)}+o\left(\learningrate C(\theta_t)\right)\nonumber\\
        &=\frac{\learningrate}{\tau^2}\EE_{\pi_t}\bigg[C_m\int_{\mathcal{A}\times\mathcal{A}}\left(\delta_{a,\mathrm{d}a'}-\pi_t^{(m)}(\mathrm{d}a'|s)\right)\nonumber\\
        &\hspace{0.4cm}\times\left(\delta_{A_{n-m},\mathrm{d}a''}-\pi_t^{(m)}(\mathrm{d}a''|S_{n-m})\right)\Theta^{(m)}((a',s),(a'',S_{n-m}))\bigg]\nonumber\\
        &\hspace{7.5cm}+o\left(\learningrate C(\theta_t)\right).
    \end{align}
    We focus on the expectation.
    It is equal to
    \begin{align*}
        &\EE_{\pi_t}\bigg[C_m\bigg(\Theta^{(m)}((a,s),(A_{n-m},S_{n-m}))
        -\EE_A\left[\Theta^{(m)}((A,s),(A_{n-m},S_{n-m}))\right]\\
        &\hspace{3cm}-\EE_A\left[\Theta^{(m)}((a,s),(A',S_{n-m}))\right]
        +\EE_{A,A'}\left[\Theta^{(m)}((A,s),(A',S_{n-m}))\right]\bigg)\bigg],
    \end{align*}
    where $A,A'$ have respective laws $\pi_t^{(m)}(\cdot|s)$ and $\pi_t^{(m)}(\cdot|S_{n-m})$ and are mutually independent of all other variables (conditionally given $S_{n-m}$ for $A'$).
    Using the trick $\EE[X(Y-\EE[Y])]=\EE[(X-\EE[X])Y]$, we obtain
    \begin{align}\label{eq: gen case after expectation trick}
        &\EE_{\pi_t}\Big[\left(C_m-\EE\left[C_m|S_{n-m}\right]\right)\Big(\Theta^{(m)}((a,s),(A_{n-m},S_{n-m}))
        -\EE_A\left[\Theta^{(m)}((A,s),(A_{n-m},S_{n-m}))\right]\Big)\Big].
    \end{align}
    We write
    \begin{align*}
        \EE\left[C_m|S_{n-m}\right]
        &=\EE\left[\sum_{\ell=n-m}^{n}\left(R_{\ell}-\tau\log\frac{\pi_t^{(n-\ell)}}{\overline{\pi}}(A_{\ell}|S_{\ell})\right)\bigg|S_{n-m}\right]\\
        &=V_{\pi_t}^{(m)}(S_{n-m})\\
        &=V_*^{(m)}(S_{n-m})-\DKL{\pi_t^{(m)}}{\pi_*^{(m)}}(S_{n-m}),
    \end{align*}
    where we used Lemma \ref{lem: general case value function is optimal value function minus DKL} and the fact that $\pi_t^{(i)}(\cdot|s)=\pi_*^{(i)}(\cdot|s)$ for $\mathbf{m}_{\pi_t^{(i)}}$-almost every $s\in\mathcal{S}$, for all $i=1,\ldots,m-1$.
    Similarly and using the expression \eqref{eq: general case optimal policy} of the optimal policy, we have
    \begin{align*}
        \EE[C_m|S_{n-m},A_{n-m}]
        &=R_{n-m}-\tau\log\frac{\pi_t^{(m)}}{\overline{\pi}}(A_{n-m}|S_{n-m}) + \EE\left[V_{\pi_t}^{(m-1)}(S_{n-m})\Big|S_{n-m},A_{n-m}\right]\\
        &= \EE\Big[V_{\pi_t}^{(m-1)}(S_{n-m+1})-V_*^{(m-1)}(S_{n-m+1})\Big|S_{n-m},A_{n-m}\Big]\\
        &\hspace{4cm}-\tau\log\frac{\pi_t^{(m)}}{\pi_*^{(m)}}(A_{n-m}|S_{n-m})+V_*^{(m)}(S_{n-m})\\
        &= -\tau\log\frac{\pi_t^{(m)}}{\pi_*^{(m)}}(A_{n-m}|S_{n-m})+V_*^{(m)}(S_{n-m}).
    \end{align*}
    Hence, the expression in \eqref{eq: gen case after expectation trick} becomes
    \begin{align*}
        &\tau\EE_{\pi_t}\bigg[\bigg(\DKL{\pi_t^{(m)}}{\pi_*^{(m)}}(A_{n-m}|S_{n-m})
        -\log\frac{\pi_t^{(m)}}{\pi_*^{(m)}}(A_{n-m}|S_{n-m})\bigg)
        \Big(\Theta^{(m)}((a,s),(A_{n-m},S_{n-m}))\\
        &\hspace{8.7cm}-\EE_A\left[\Theta^{(m)}((A,s),(A_{n-m},S_{n-m}))\right]\Big)\bigg],
    \end{align*}
    which corresponds to the first order term in right-hand side of the equation in the Lemma.
    Coming back to \eqref{eq: gen case log increment to compute}, this concludes the proof.
\end{proof}

\begin{proof}[\textbf{Proof of Theorem \ref{thm: general case global optimality}}]
The idea of the proof is rather simple:
by Lemmas \ref{lem:converging_subsequence_of_policies} and \ref{lem:parameters_remain_bounded}, we know that any subsequence of $\pi_t$ has a converging subsequence, and that the limits have finite norm parameters.
Hence, we only need to show the unicity of the limit, namely, that $\theta\in\RR^P$ is a critical point if and only $\pi_\theta=\pi_*$.
We follow the intuition given after Theorem \ref{thm: general case global optimality}.

We reason by induction.
Let $m\leq n$, suppose that $\pi^{(i)}_t(\cdot|s)=\pi^{(i)}_*(\cdot|s)$ for $\mathbf{m}_{\pi_t}^{(i)}$-almost every $s\in\mathcal{S}$, for all $i=1,\ldots,m-1$, and that $\pi^{(i)}_t=\pi^{(i)}_\infty$ for all $i=m,\ldots,n$.
In particular, we are at a critical point $(\theta^{(1)}_t,\ldots,\theta^{(n)}_t)$ of $(\theta^{(1)},\ldots,\theta^{(n)})\mapsto J_n(\pi_\theta)$.
Let $a\in\mathcal{A},s\in\mathcal{S}$.
By Lemma \ref{lem: general case expected update log pi_t}, we have that 
\begin{align*}
    0&=
    \log\pi^{(m)}_{t+1}(a|s)-\log\pi^{(m)}_t(a|s)\\
    &= - \learningrate \tau\int_{\mathcal{A}\times\mathcal{S}}\mathbf{m}_{\pi_t}^{(m)}(\mathrm{d}s')\pi_t^{(m)}(\mathrm{d}a'|s')\left(\log\frac{\pi_t^{(m)}}{\pi_*^{(m)}}(a'|s')-\DKL{\pi_t^{(m)}}{\pi_*^{(m)}}(s')\right)\\
    &\hspace{3cm}\times\bigg(\Theta^{(m)}((a,s),(a',s'))
    -\EE_{\pi^{(m)}_t}[\Theta^{(m)}((A,s),(a',s'))]\bigg)
    +o\left(\learningrate C(\theta_t)\right),
\end{align*}
Since the above must be true for all $\learningrate>0$, we deduce that 
\begin{align}\label{eq: gen case lower bound log increment with d}
    &\int_{\mathcal{A}\times\mathcal{S}}\mathbf{m}_{\pi_t}^{(m)}(\mathrm{d}s')\pi_t^{(m)}(\mathrm{d}a'|s')\left(\log\frac{\pi_t^{(m)}}{\pi_*^{(m)}}(a'|s')-\DKL{\pi_t^{(m)}}{\pi_*^{(m)}}(s')\right)\nonumber\\
    &\hspace{3cm}\times\bigg(\Theta^{(m)}((a,s),(a',s'))
    -\EE_{\pi^{(m)}_t}[\Theta^{(m)}((A,s),(a',s'))]\bigg)=0.
\end{align}
Let $\widetilde{\Theta}^{(m)}$ be the positive-semidefinite kernel constructed from $\Theta^{(m)}$ and $\pi_t^{(m)}$ as in Lemma \ref{lem: softmax RKHS is larger}.
One can easily check that
\begin{align*}
    \log\frac{\pi_t^{(m)}}{\pi_*^{(m)}}(a|s)-\DKL{\pi_t^{(m)}}{\pi_*^{(m)}}
    &=h_t^{(m)}(a,s)-Q_*^{(m)}(a,s)-\EE_{\pi_t^{(m)}}\left[h_t^{(m)}(A,s)-Q_*^{(m)}(A,s)\right].
\end{align*}
In particular, using the trick $\EE[X(Y-\EE[Y])]=\EE[(X-\EE[X])Y]$, we can rewrite \eqref{eq: gen case lower bound log increment with d} as
\begin{align}\label{eq: h-Q in softmax RKHS}
    \int_{\mathcal{A}\times\mathcal{S}}\mathbf{m}_{\pi_t}^{(m)}(\mathrm{d}s')\pi_t^{(m)}(\mathrm{d}a'|s')\left(h_t^{(m)}(a',s')-Q_*^{(m)}(a',s')\right)\widetilde{\Theta}^{(m)}((a,s),(a',s'))=0.
\end{align}
Since the above is true for all $(a,s)\in\mathcal{A}\times\mathcal{S}$, we see by Lemma \ref{lem: if I_Kf=0 then in ortho} that $h_t^{(m)}-Q_*^{(m)}\in(\mathcal{H}_{\widetilde{\Theta}^{(m)}})^\perp$, that is the orthogonal complement of $\mathcal{H}_{\widetilde{\Theta}^{(m)}}$ in $L^2(\mathbf{m}_{\pi_t}^{(m)}(\mathrm{d}s')\pi_t^{(m)}(\mathrm{d}a'|s'))$.
By Assumption \ref{assumption: optimal policy in RKHS}, we get $h_t^{(m)}-Q_*^{(m)}\in\mathcal{H}_{\Theta^{(m)}}\cap(\mathcal{H}_{\widetilde{\Theta}^{(m)}})^\perp$, and Lemma \ref{lem: softmax RKHS is larger} entails that for $\mathbf{m}_{\pi_t}^{(m)}$-almost every $s\in\mathcal{S}$, the map $a\mapsto h_t^{(m)}(a,s)-Q_*(a,s)$ is constant.
This implies in turn that $\pi_t^{(m)}(\cdot|s)=\pi_*^{(m)}(\cdot|s)$ for $\mathbf{m}_{\pi_t}^{(m)}$-almost every $s\in\mathcal{S}$, which concludes the proof.
\end{proof}

\subsection{Global optimality of MPG: non-realizable case}
\label{Section: global optimality of MPG: non-realizable case}

In order to extend the global optimality from the case where $\pi_*$ belongs to the parametric space $\mathscr{P}_n$ to the case where $\pi_*$ is outside of $\mathscr{P}_n$, we use tools from information geometry and apply the strategy outlined in Section \ref{section main: global convergence beyond the realizability assumption}.

We use the following notation in the proof: the set of parametric $1$-step policies whose preference $h_\theta$ belongs to $\mathcal{H}_{\Theta^{(i)}}$ is denoted by $\mathscr{P}^{(i)}$.

\begin{proof}[\textbf{Proof of Theorem \ref{Thm: global cvg outside of the RKHS}}]
Let $\vartheta\in\RR^\nparam$ be a critical point of $\theta\mapsto J_n(\pi_\theta)$.
Consider a fixed $i\in\{1,\ldots, n\}$.
Recall that $Q_{\pi_\vartheta}^{(i)}(a,s)=r(a,s)+\int_{\mathcal{S}}p(s,a,\mathrm{d}s')V_{\pi_\vartheta}^{(i-1)}(s')$, which does not depend on $\pi_\vartheta^{(j)}$, $j\geq i$.
Let $\widehat{\pi}^{(i)}$ be the policy with preference $Q_{\pi_\vartheta}^{(i)}$.
Note that $Q_{\pi_\vartheta}^{(i)}$ does not necessarily belong to $\mathcal{H}_{\Theta^{(i)}}$, hence we do not make the dependence on $\vartheta$ (which is fixed) explicit in $\widehat{\pi}^{(i)}$.
This is the optimal policy given that the shorter $j$-step policies, $j<i$, are fixed.
Indeed, we always have that
\begin{align*}
    \widehat{J}^{(i)}(\pi^{(i)},\vartheta)
    :=&\int_{\mathcal{S}}\mathbf{m}_{\pi_\vartheta}^{(i)}(\mathrm{d}s)\left(\EE_{\pi^{(i)}}[Q^{(i)}_{\pi_{\vartheta}}(A,s)]-\tau\DKL{\pi^{(i)}}{\overline{\pi}}(s)\right)\\
    =&\ \tau\int_{\mathcal{S}}\mathbf{m}_{\pi_\vartheta}^{(i)}(\mathrm{d}s)\left(\log\left(\int_{\mathcal{A}}\overline{\pi}(\mathrm{d}a|s)e^{Q^{(i)}_{\pi_\vartheta}(a,s)/\tau}\right)-\DKL{\pi^{(i)}}{\widehat{\pi}^{(i)}}(s)\right).
\end{align*}
The first term of the right-hand side depends on $\pi_\vartheta^{(j)}$ through $Q_{\pi_\vartheta}^{(i)}$ for $j<i$, whereas it depends on $\pi_\vartheta^{(j)}$ through $\mathbf{m}_{\pi_\vartheta}^{(i)}$ for $j>i$, but it does not depend on $\pi_{\vartheta}^{(i)}$.
Therefore, we see that $\widehat{\pi}^{(i)}=\mathrm{argmax}_{\pi^{(i)}\in\mathcal{P}_1}\widehat{J}^{(i)}(\pi^{(i)},\vartheta)$.

Similar to what was done in appendix \ref{appendix: information geometry}, let $\mathscr{P}^{(i)}_{\star}$ be the quotient space of $\mathscr{P}^{(i)}$ and its subspace of policies whose preferences are constant in $a$ for all $s\in\mathrm{Supp}(\mathbf{m}_{\pi_{\vartheta}^{(i)}})$.
In this quotient space, policies that are equal to each other on the support of $\mathbf{m}_{\pi_{\vartheta}^{(i)}}$ are identified as the same policy, since states outside of this set are never visited with probability one.
Define
\begin{align*}
    \pi^{(i)}_{\theta_*}
    =\underset{\pi^{(i)}_\theta\in\mathscr{P}^{(i)}_{\star}}{\mathrm{argmin}}\ D^{(i)}(\widehat{\pi}^{(i)},\pi^{(i)}_\theta)
    :=\underset{\pi^{(i)}_\theta\in\mathscr{P}^{(i)}_\star}{\mathrm{argmin}}\int_{\mathcal{S}}\mathbf{m}^{(i)}_{\pi_\vartheta}(\mathrm{d}s)\DKL{\pi^{(i)}_\theta}{\widehat{\pi}^{(i)}}(s)
    .
\end{align*}
It turns out that the map $D^{(i)}$ defined above is a Bregman divergence on $\mathcal{P}_\star$ (denoting the space where policies that coincide on the support of $\mathbf{m}_{\pi_\vartheta}^{(i)}$ are identified together).
Using the fact that $\vartheta$ is a critical point combined with Lemma \ref{lem: same critical points for divergence}, we have that
\begin{align*}
    0&=
    \nabla_{\theta^{(i)}}J_n(\pi_\vartheta)
    =-\nabla_{\theta^{(i)}}D(\pi_{\vartheta}^{(i)},\pi^{(i)}_{\theta_*}).
\end{align*}
We stress once more that $\pi^{(i)}_{\theta_*}$ only depends on $\widehat{\pi}^{(i)}$, which in turn only depends on $\pi_\vartheta^{(1)},\ldots,\pi_\vartheta^{(i-1)}$ through $Q_{\pi_\vartheta}^{(i)}$ and on $\pi_\vartheta^{(i+1)},\ldots,\pi_\vartheta^{(n)}$ through $\mathbf{m}_{\pi_\vartheta}^{(i)}$.
Therefore, the equation above corresponds to the gradient of the objective of $1$-step MPG with optimal policy $\pi_{\theta_*}^{(i)}$.
This observation brings us back to the realizable case, for which Theorem \ref{thm: general case global optimality} applies.
This implies that necessarily, $\pi_\vartheta^{(i)}(\cdot|s)=\pi_{\theta_*}^{(i)}(\cdot|s)$ for $\mathbf{m}_{\pi_\vartheta}^{(i)}$-almost every $s\in\mathcal{S}$.
In particular, this shows the uniqueness of the $\mathrm{argmin}$ for reachable states.

The above argument proves that if $\vartheta\in\RR^\nparam$ is a critical point, then
\begin{align*}
    J_n(\pi_\vartheta)
    &=\max_{\theta^{(i)}\in\RR^{\nparam_i}}J_n(\pi^{(1)}_\vartheta,\ldots,\pi^{(i)}_{\theta^{(i)}},\ldots,\pi^{(n)}_{\vartheta^{(n)}}).
\end{align*}
Since this is true for every $i=1,\ldots,n$ and since maxima can be taken in any order, we have that 
\begin{align*}
    J_n(\pi_\vartheta)=\max_{\theta\in\RR^P}J_n(\pi_\theta)
\end{align*}
We have thus proved that any critical point is a global maximum of the objective.
As in the proof of Theorem \ref{thm: general case global optimality}, the argument using Lemmas \ref{lem:converging_subsequence_of_policies} and \ref{lem:parameters_remain_bounded} applies to show global convergence, 
concluding the proof.
\end{proof}

\begin{proof}[\textbf{Proof of Proposition \ref{prop: projectional consistency property}}]
    Suppose that $\theta_t=(\theta_t^{(1)},\ldots,\theta_t^{(n)})$ satisfies the projectional consistency property \eqref{eq: projectional consistency property}.
    We thus have that $h_t^{(1)}-Q_*^{(1)}\in(\mathcal{H}_{\widetilde{\Theta}^{(1)}})^\perp$, the orthogonal space of $\mathcal{H}_{\widetilde{\Theta}^{(1)}}$ in $L^2(\mathbf{m}^{(1)}(\mathrm{d}s)\pi_t^{(1)}(\mathrm{d}a))$.
    In particular, using Lemma \ref{lem: if I_Kf=0 then in ortho}, one can show that Equation \eqref{eq: h-Q in softmax RKHS} is satisfied, entailing that $\nabla_{\theta^{(1)}}J_n(\pi_\theta)=0$.
    The same reasoning applies for all steps $i=1,\ldots,n$, showing that $\theta_t$ is a critical point, and therefore, the unique global optimum by Theorem \ref{Thm: global cvg outside of the RKHS}.
    This concludes the proof.
\end{proof}


\section{Assumptions}
\label{appendix: assumptions}


We now list the assumptions and briefly mention their roles in this work:
\begin{itemize}
    \item In Proposition \ref{prop: extending horizon converges to standard optimal policy}, $\nu$ has full support in $\mathcal{S}$ and the MDP is ergodic: it is not restrictive, as its role is to ensure that the optimal policies for all horizons visit Lebesgue almost all states, thus avoiding considerations about reachable states. Ergodicity ensures the existence of a stationary state distribution. In particular, the optimal policy $\pi_*$ does not depend on $\nu$.
    \item Continuous closed $\mathcal{A},\mathcal{S}$: to apply Mercer's Theorem.
    \item Continuous and bounded kernels $\Theta^{(i)}$: to apply Mercer's Theorem.
    \item Measurable selection assumption and measurability of $p$ and $r$: ensures the measurability of the variables generated by the MDP, Lebesgue integrability and avoid pathological cases.
    \item For all Borel set $B\subset\mathcal{S}$, the map $(s,a)\mapsto p(s,a,B)$ is continuous: is used in the proof of Lemma \ref{lem:converging_subsequence_of_policies} to guarantee convergence of a subsequence of the state distributions and policies.
    \item Rewards are bounded: ensures that value functions are well defined. It is also used to prove the convergence of the objective and of the parameters to finite values.
    \item For all $s\in\mathcal{S}$ and all $\theta^{(i)}\in\RR^{P_i}$, the map $a\mapsto h_{\theta}^{(i)}(a,s)$ is constant if and only if $||\theta^{(i)}||=0$: this guarantees that $\pi_\theta=\overline{\pi}$ if and only if $||\theta||=0$ and avoid pathological cases, such as divergence of parameters.
\end{itemize}

\section{Numerical experiments}

\label{appendix: numerical experiments}
We apply MPG on a number of numerical experiments as detailed in section \ref{Section: numerical experiments}. The MPG is implemented as in algorithm \ref{alg:mpg}.

\begin{algorithm}
\caption{MPG implementation for $N$ horizon task}\label{alg:mpg}
\begin{algorithmic}[0]
\STATE \textbf{Input:} initial temperature $\tau_0$, initial learning rate $\eta_0$, final temperature $\tau_T$, final learning rate $\eta_T$

\STATE $\tau \leftarrow \tau_0$
\STATE $\eta \leftarrow \eta_0$
\FOR {t = 1, $...$, episodes}

\STATE generate trajectory from policies $\{\pi_t^{n}, \pi_t^{n-1}, ..., \pi_t^{1}  \}$: $\{(s_i,s_{i+1},a_i,r_i)\}_{i=0}^{n-1}$ 

\FOR {i = $1$, $\cdots$ , $n$}
\STATE $C_i = \sum_{\ell = n-i}^{n-1} \left( r_\ell - \tau\log\frac{\pi_t^{(n-\ell)}}{\bar{\pi}} (a_\ell | s_\ell) \right)$ 
\STATE $\theta^{(i)}_{t+1}  = \theta^{(i)}_t +\eta C_i \nabla \log \pi^{(i)}_t (a_{n-i}|s_{n-i})$
\ENDFOR 
    \STATE decay $\tau, \eta$ using $d_\tau = \left( \frac{\tau_T}{\tau_0}\right)^{1/\mbox{episodes}}$ and $d_\eta = \left( \frac{\eta_T}{\eta_0}\right)^{1/\mbox{episodes}}$
\ENDFOR
\end{algorithmic}
\end{algorithm}

\subsection{Analytical task}
\label{appendix:analytical_task}

\paragraph{Set-up:} We consider a state-space consisting of $\mathcal{S}=\{0,1,2,3,4\}$, an action space $\mathcal{A}=\{1,2\}$. At each state $s$, the agent performs action $a$, taking the agent to the next state $(s+a)\mod 5$.

We define an orthonormal basis (in $\ell^2(\mathcal{A}\times\mathcal{S})$) of the space of functions $\{f:\mathcal{A}\times\mathcal{S}\to\RR:f(1,s)+f(2,s)=0,\ \forall s\in\mathcal{S}\}$.
Note that one can always recenter any map $g$ on $\mathcal{A}\times\mathcal{S}$ so that $g(1,s)+g(2,s)=0$, without changing the policy obtained as the softmax of $g$, in particular, any policy can be written as the softmax of such a function. The basis is defined as
\begin{align*}
     e_1 =\sqrt{6}\begin{pmatrix}
    1 & -1 \\
    0 & 0 \\
    1 & -1 \\
    0 & 0 \\
    1 & -1 \\
    \end{pmatrix}, \quad     &e_2 =\sqrt{4}\begin{pmatrix}
    0 & 0 \\
    1 & -1 \\
    0 & 0 \\
    1 & -1 \\
    0 & 0 \\
    \end{pmatrix}, \quad e_3 =\sqrt{4}\begin{pmatrix}
    0 & 0 \\
    1 & -1 \\
    0 & 0 \\
    -1 & 1 \\
    0 & 0 \\
    \end{pmatrix},\\
    e_4 &=\sqrt{8}\begin{pmatrix}
    -2 & 2 \\
    0 & 0 \\
    1 & -1 \\
    0 & 0 \\
    1 & -1 \\
    \end{pmatrix}, \quad e_5 =\sqrt{4}\begin{pmatrix}
    0 & 0 \\
    0 & 0 \\
    1 & -1 \\
    0 & 0 \\
    -1 & 1 \\
    \end{pmatrix}.
\end{align*}

Recall that $Q_*^{(1)}(a,s)=r(a,s)$, which can be represented by
\[Q_*^{(1)}(a,s) = \sum_{j=1}^5 \theta^*_j e_j(a,s), \quad \theta^*_j \in \mathbb{R}.\] 

\paragraph{Experiments:}
\begin{enumerate}
    \item Obtaining the first two step policies with assumption \ref{assumption: optimal policy in RKHS}, namely, that the optimal policy's parameters can be represented by our parametric space, and when the assumption \ref{assumption: optimal policy in RKHS} does not hold.
    
\textbf{Setup:}
\begin{itemize}
    \item preference function $h(a,s)$ is expressed by a linear model
    \item $\theta_0$ randomly initialised with i.i.d. centered Gaussian with standard deviation $1$;
    \item $\theta^* = (0,0.1,-0.15,0.05,-0.1)$;
    \item Initial learning rate $\eta_0 = 0.001$, terminal learning rate $\eta_T = 0.001$ (no decay);
    \item Temperature $\tau = 1.0$ remains fixed during training (no decau);
    \item True gradient update
    \item Number of episodes: $14000$
\end{itemize}

\end{enumerate}

\subsection{Control problems}
\label{appendix:control_problems}

\textbf{Setup:}
\begin{itemize}
    \item Preference function $h(a,s)$ is expressed by a fully connected neural network with $3$ hidden layers, each with width $100$ and ReLU activation function. The output layer has a softmax activation;
    \item Parameters are initialised with He initialisation;
    \item Initial learning rate $\eta_0$ and initial temperature $\tau_0$ are hyper-parameters;
    \item Final learning rate $\eta_T$ and final temperature $\tau_T$ are fixed and problem dependent;
    \item Number of episodes is task dependent;
    \item Both the learning rate and temperature decay during training, using the following decay rates $d_\eta = \left(\frac{\eta_T}{\eta_0}\right)^{1/episodes}$ and $d_\tau = \left(\frac{\tau_T}{\tau_0}\right)^{1/episodes}$, respectively;
    \item Gradient update estimated using one trajectory as in \eqref{eq: update cascade learning}.
\end{itemize}


\paragraph{Frozen lake:} 

Aside from the details specified above, further details of the set-up for Frozen lake are:
\begin{itemize}
    \item Reward: it is well-known that reshaping the reward function can change the performance of the algorithm. The original reward function does not discriminate between falling into a hole, not moving and moving, so we used a reshaped reward function: falling ($-1$), moving against a wall ($-0.1$), moving successfully ($+0.01$) and reaching the treasure ($+10.0$);
    \item Final learning rate $\eta_T = 1\times 10^{-6}$;
    \item Final temperature $\tau_T=0.01$ (when applicable);
    \item Number of episodes: 1000.
\end{itemize}

We train sets of $3$ agents to explore the hyper-parameter space as denoted in \ref{table:frozenlake-hypers}. In table \ref{table:frozen-lake-stage-1}, we show the hyper-parameter exploration for the different considered algorithms.

\begin{table}
\begin{center}
 \caption{Hyper-parameter search using $3$ agents for Frozen lake task. The `**' symbol denotes that no runs were made for that set of hyper-parameters. \label{table:frozen-lake-stage-1}}
\begin{tabular}{| c | c | c | } 
 \hline
$\eta_0$ & softPG & MPG \\ [0.5ex] \hline
\hline \multicolumn{3}{|c|}{$\tau_0 = 0.15 $}\\ \hline
 $0.1$ & $-1.89$ & ** \\
 $0.05$ & $-1.66$ &  ** \\
 $0.01$ & $0.20$ & $2.13$ \\
 $0.005$ & $0.20$ & $2.06$ \\
 $0.001$ & $0.20$ & $3.04$ \\
 $0.0005$ & $0.20$ & $9.36$ \\
 $0.0001$ & $0.20$ & $-1.38$ \\
\hline \multicolumn{3}{|c|}{$\tau_0 = 0.20 $}\\ \hline
 $0.1$ & $-0.53$ & **\\
 $0.05$ & $-0.19$ & ** \\
 $0.01$ & $0.20$ & $6.53$ \\
 $0.005$ & $0.20$ & $6.69$ \\
 $0.001$ & $0.20$ & $10.05$ \\
 $0.0005$ & $0.20$ & $10.05$ \\
 $0.0001$ & $0.20$ & $1.66$ \\
\hline \multicolumn{3}{|c|}{$\tau_0 = 0.25 $}\\ \hline
 $0.1$ & $-0.93$ & ** \\
 $0.05$ & $-1.27$ & ** \\
 $0.01$ & $0.20$ & $10.05$ \\
 $0.005$ & $0.20$ & $10.05$ \\
 $0.001$ & $0.20$ & $10.05$ \\
 $0.0005$ & $0.20$ & $9.99$ \\
 $0.0001$ & $0.20$ & $1.39$ \\
\hline \multicolumn{3}{|c|}{$\tau_0 = 0.30 $}\\ \hline
 $0.1$ & $-1.89$ & ** \\
 $0.05$ & $-1.89$ & ** \\
 $0.01$ & $0.18$ & $6.33$ \\
 $0.005$ & $0.20$ & $10.05$ \\
 $0.001$ & $0.20$ & $10.05$ \\
 $0.0005$ & $0.20$ & $10.05$ \\
 $0.0001$ & $0.19$ & $-1.13$ \\
\hline \multicolumn{3}{|c|}{$\tau_0 = 0.35 $}\\ \hline
 $0.1$ & $-1.27$ & ** \\
 $0.05$ & $-1.05$ & ** \\
 $0.01$ & $0.20$ & $10.05$ \\
 $0.005$ & $0.20$ & $10.05$ \\
 $0.001$ & $0.20$ & $10.05$ \\
 $0.0005$ & $0.20$ & $6.04$ \\
 $0.0001$ & $0.20$ & $-1.65$ \\[1ex] \hline
\end{tabular}~
\begin{tabular}{| c | c | c | } 
 \hline
$\eta_0$ & softPG & MPG \\ [0.5ex] \hline \hline
\multicolumn{3}{|c|}{$\tau_0 = 0.40 $}\\ \hline
 $0.1$ & $-1.78$ & ** \\
 $0.05$ & $-1.66$ & ** \\
 $0.01$ & $0.19$ & $10.05$ \\
 $0.005$ & $0.19$ & $10.05$ \\
 $0.001$ & $0.20$ & $10.05$ \\
 $0.0005$ & $0.20$ & $6.72$ \\
 $0.0001$ & $0.20$ & $0.30$ \\
\hline \multicolumn{3}{|c|}{$\tau_0 = 0.45 $}\\ \hline
 $0.1$ & $-1.55$ & ** \\
 $0.05$ & $-1.26$ & ** \\
 $0.01$ & $-0.42$ & $10.05$ \\
 $0.005$ & $0.20$ & $6.40$ \\
 $0.001$ & $0.20$ & $4.09$ \\
 $0.0005$ & $0.20$ & $2.95$ \\
 $0.0001$ & $0.20$ & $0.25$ \\
\hline \multicolumn{3}{|c|}{$\tau_0 = 0.50 $}\\ \hline
 $0.1$ & $-2.00$ & ** \\
 $0.05$ & $0.20$ & ** \\
 $0.01$ & $-0.15$ & $10.05$ \\
 $0.005$ & $0.20$ & $6.61$ \\
 $0.001$ & $0.20$ & $9.75$ \\
 $0.0005$ & $0.20$ & $2.49$ \\
 $0.0001$ & $0.20$ & $-1.22$ \\
\hline \multicolumn{3}{|c|}{$\tau_0 = 0.55 $}\\ \hline
 $0.1$ & $-1.78$ & ** \\
 $0.05$ & $-1.27$ & ** \\
 $0.01$ & $3.48$ & $10.05$ \\
 $0.005$ & $0.19$ & $9.82$ \\
 $0.001$ & $0.20$ & $7.02$ \\
 $0.0005$ & $0.20$ & $2.59$ \\
 $0.0001$ & $0.18$ & $-1.29$ \\
\hline \multicolumn{3}{|c|}{$\tau_0 = 0.60 $}\\ \hline
 $0.1$ & $-1.55$ & ** \\
 $0.05$ & $-1.27$ & ** \\
 $0.01$ & $3.48$ & $10.05$ \\
 $0.005$ & $0.20$ & $10.05$ \\
 $0.001$ & $0.20$ & $6.71$ \\
 $0.0005$ & $0.20$ & $-0.92$ \\
 $0.0001$ & $0.20$ & $-1.09$ \\ \hline
\end{tabular}~
\begin{tabular}{| c | c | c | } 
 \hline
$\eta_0$ & softPG & MPG \\ [0.5ex] \hline
\hline \multicolumn{3}{|c|}{$\tau_0 = 0.65 $}\\ \hline
 $0.1$ & $-2.00$ & ** \\
 $0.05$ & $-0.53$ & ** \\
 $0.01$ & $3.14$ & ** \\
 $0.005$ & $0.20$ & ** \\
 $0.001$ & $0.20$ & ** \\
 $0.0005$ & $-0.53$ & ** \\
 $0.0001$ & $-0.44$ & ** \\
\hline \multicolumn{3}{|c|}{$\tau_0 = 0.70 $}\\ \hline
 $0.1$ & $-1.89$ & ** \\
 $0.05$ & $-1.27$ & ** \\
 $0.01$ & $-0.55$ & ** \\
 $0.005$ & $-0.54$ & ** \\
 $0.001$ & $0.20$ & ** \\
 $0.0005$ & $-0.34$ & ** \\
 $0.0001$ & $-0.56$ & ** \\ \hline  \multicolumn{3}{c}{~} \\[1cm]

\hline 
$\eta_0$ & PG & nsPG \\ [0.5ex] \hline\hline
 $0.5$ & ** & $-1.89$ \\
 $0.1$ & ** & $-1.76$ \\
 $0.05$ & $2.24$ & $6.03$ \\
 $0.01$ & $10.05$ & $2.91$ \\
 $0.005$ & $7.69$ & $2.88$ \\
 $0.001$ & $-0.83$ & $-1.00$ \\
 $0.0005$ & $-0.88$ & $-1.06$ \\
 $0.0001$ & $-1.12$ & $-1.15$ \\ [1ex] \hline
\end{tabular}
\end{center}
\end{table}

\subsection{Cart Pole}
Aside from the details specified above, further details of the set-up for Frozen lake are:
\begin{itemize}
    \item Reward: The original reward function gives a $+1$ reward for each time that the pole stays upright, and the task finishes if the cart leaves the domain or if the pole is far enough from being upright. We reshape the reward function to yield a penalty $-10\exp(-0.05(i-1))$ when the task concludes, where $i$ is the length of time the pole stayed upright. Namely, if the pole falls early in the task, the penalisation is larger.
    \item Final learning rate $\eta_T = 1\times 10^{-8}$;
    \item Final temperature $\tau_T=0.01$ (when applicable);
    \item Number of episodes: 1000.
\end{itemize}

We train sets of $3$ agents to explore the hyper-parameter space as denoted in \ref{table:cartpole-hypers}. In table \ref{table:cartpole-stage-1}, we show the hyper-parameter exploration for the different considered algorithms.

\begin{table}
\begin{center}
 \caption{Hyper-parameter search using $3$ agents for balancing the cart pole task. The `**' symbol denotes that no runs were made for that set of hyper-parameters.  \label{table:cartpole-stage-1}}
\begin{tabular}{| c | c | c | } 
 \hline
$\eta_0$ & softPG & MPG \\ [0.5ex] \hline
\hline \multicolumn{3}{|c|}{$\tau_0 = 0.05 $}\\ \hline
 $0.001$ & ** & $1.87$ \\
 $0.0005$ & ** & $1.80$ \\
 $0.0001$ & ** & $1.70$ \\
 $5\times 10^{-5}$ & ** & $25.21$ \\
 $1\times 10^{-5}$ & ** & $73.22$ \\
 $1\times 10^{-6}$ & ** & $41.37$ \\
 $5\times 10^{-6}$ & ** & $58.85$ \\
\hline \multicolumn{3}{|c|}{$\tau_0 = 0.10 $}\\ \hline
 $0.001$ & ** & $1.82$ \\
 $0.0005$ & $1.70$ & $7.50$ \\
 $0.0001$ & $50.95$ & $48.21$ \\
 $5\times 10^{-5}$ & $51.72$ & $41.21$ \\
 $1\times 10^{-5}$ & $81.99$ & $95.32$ \\
 $1\times 10^{-6}$ & $55.40$ & $33.45$ \\
 $5\times 10^{-6}$ & $85.41$ & $70.08$ \\
\hline \multicolumn{3}{|c|}{$\tau_0 = 0.15 $}\\ \hline
 $0.001$ & ** & $1.69$ \\
 $0.0005$ & $14.86$ & $1.83$ \\
 $0.0001$ & $59.52$ & $68.65$ \\
 $5\times 10^{-5}$ & $79.77$ & $86.28$ \\
 $1\times 10^{-5}$ & $75.65$ & $88.94$ \\
 $1\times 10^{-6}$ & $64.15$ & $36.35$ \\
 $5\times 10^{-6}$ & $83.88$ & $81.07$ \\
\hline \multicolumn{3}{|c|}{$\tau_0 = 0.20 $}\\ \hline
 $0.001$ & ** & $29.14$ \\
 $0.0005$ & $1.84$ & $1.79$ \\
 $0.0001$ & $51.66$ & $77.32$ \\
 $5\times 10^{-5}$ & $79.13$ & $88.69$ \\
 $1\times 10^{-5}$ & $79.07$ & $77.21$ \\
 $1\times 10^{-6}$ & $37.01$ & $19.39$ \\
 $5\times 10^{-6}$ & $66.27$ & $74.18$ \\
 \hline
 \multicolumn{3}{|c|}{$\tau_0 = 0.25 $}\\ \hline
 $0.001$ & ** & $6.01$ \\
 $0.0005$ & $66.09$ & $20.18$ \\
 $0.0001$ & $62.15$ & $71.22$ \\
 $5\times 10^{-5}$ & $91.59$ & $91.37$ \\
 $1\times 10^{-5}$ & $64.18$ & $77.80$ \\
 $1\times 10^{-6}$ & $26.29$ & $18.45$ \\
 $5\times 10^{-6}$ & $57.06$ & $73.08$ \\
 \hline
\end{tabular}~ \begin{tabular}{| c | c | c | } 
 \hline
$\eta_0$ & softPG & MPG \\ [0.5ex] \hline
\hline \multicolumn{3}{|c|}{$\tau_0 = 0.30 $}\\ \hline
 $0.001$ & ** & $13.95$ \\
 $0.0005$ & $61.24$ & $32.84$ \\
 $0.0001$ & $67.95$ & $97.10$ \\
 $5\times 10^{-5}$ & $93.45$ & $99.37$ \\
 $1\times 10^{-5}$ & $67.59$ & $74.31$ \\
 $1\times 10^{-6}$ & $61.09$ & $28.72$ \\
 $5\times 10^{-6}$ & $65.20$ & $80.48$ \\
\hline \multicolumn{3}{|c|}{$\tau_0 = 0.35 $}\\ \hline
 $0.001$ & ** & $37.08$ \\
 $0.0005$ & ** & $54.98$ \\
 $0.0001$ & $84.55$ & $75.19$ \\
 $5\times 10^{-5}$ & $72.55$ & $90.18$ \\
 $1\times 10^{-5}$ & $86.74$ & $76.82$ \\
 $1\times 10^{-6}$ & $33.44$ & $37.75$ \\
 $5\times 10^{-6}$ & $59.14$ & $65.92$ \\
\hline \multicolumn{3}{|c|}{$\tau_0 = 0.40 $}\\ \hline
 $0.001$ & ** & $17.66$ \\
 $0.0005$ & ** & $46.27$ \\
 $0.0001$ & $75.47$ & $98.04$ \\
 $5\times 10^{-5}$ & $65.54$ & $95.11$ \\
 $1\times 10^{-5}$ & $64.47$ & $66.66$ \\
 $1\times 10^{-6}$ & $27.15$ & $20.33$ \\
 $5\times 10^{-6}$ & $75.12$ & $63.17$ \\
\hline \multicolumn{3}{|c|}{$\tau_0 = 0.45 $}\\ \hline
 $0.001$ & ** & ** \\
 $0.0005$ & ** & ** \\
 $0.0001$ & $88.47$ & ** \\
 $5\times 10^{-5}$ & $76.74$ & ** \\
 $1\times 10^{-5}$ & $64.13$ & ** \\
 $1\times 10^{-6}$ & $3.08$ & ** \\
 $5\times 10^{-6}$ & $52.08$ & ** \\ \hline
 \end{tabular}~ \begin{tabular}{| c | c | c | }
\hline 
$\eta_0$ & PG & nsPG \\ [0.5ex] \hline\hline
$0.005$ & $2.18$ & $38.68$ \\
 $0.001$ & $80.75$ & $98.94$ \\
 $0.0005$ & $31.99$ & $34.39$ \\
 $0.0001$ & $18.62$ & $19.18$ \\
 $5\times 10^{-5}$ & $17.81$ & $18.83$ \\
 $1\times 10^{-5}$ & $16.66$ & $16.96$ \\
 $5\times 10^{-6}$ & $16.60$ & $18.25$ \\
 $1\times 10^{-6}$ & $17.62$ & $16.28$ \\\hline
\end{tabular}
\end{center}
\end{table}

\clearpage

\vskip 0.2in
\bibliographystyle{abbrvnat}
\bibliography{biblio}

\end{document}